\documentclass[final, 12pt]{colt2019} 

\usepackage{times}

\coltauthor{\Name{Jeongyeol Kwon}\thanks{These two authors equally contributed to this work.} \Email{kwonchungli@utexas.edu} \\
    \addr The University of Texas at Austin \\
  \Name{Wei Qian}\footnotemark[1] \Email{wq34@cornell.edu} \\
    \addr Cornell University \\
  \Name{Constantine Caramanis} \Email{constantine@utexas.edu} \\
  \addr The University of Texas at Austin \\
  \Name{Yudong Chen} \Email{yudong.chen@cornell.edu} \\
  \addr Cornell University \\
  \Name{Damek Davis} \Email{dsd95@cornell.edu} \\
  \addr Cornell University
}

\usepackage[utf8]{inputenc} 
\usepackage[T1]{fontenc}    
\usepackage{url}            
\usepackage{booktabs}       
\usepackage{amsfonts}       
\usepackage{nicefrac}       
\usepackage{microtype}      

\usepackage[english]{babel}

\usepackage{amsmath}
\usepackage{amssymb}
\usepackage{graphicx}
\usepackage{bm}
\usepackage[colorinlistoftodos]{todonotes}

\usepackage[title]{appendix}
\usepackage{comment}
\usepackage{amsfonts}

\newtheorem{lem}{Lemma}

\usepackage{thm-restate}
\newtheorem*{remark*}{Remark}

\usepackage{dsfont,subcaption,float}
\newif\ifdraft
\draftfalse 

\newcommand{\red}[1]{\textcolor{red}{#1}}

\newcommand{\magenta}[1]{\textcolor{magenta}{#1}}
\newcommand{\orange}[1]{\textcolor{orange}{#1}}
\newcommand{\wqcomment}[1]{\ifdraft {\bf{{\magenta{{Wei --- #1}}}}}\else\fi}

\newcommand{\yccomment}[1]{\ifdraft {\bf{{\red{{Yudong --- #1}}}}}\else\fi}
\newcommand{\jycomment}[1]{\ifdraft {\bf{{\orange{{Jeongyeol --- #1}}}}}\else\fi}

\newcommand{\bbeta}{\bm{\beta}}		
\newcommand{\betastar}{\bbeta^*}	

\title{Global Convergence of the EM Algorithm for Mixtures of Two Component Linear Regression}

\begin{document}
\maketitle 
\begin{abstract}
The Expectation-Maximization algorithm is perhaps the most broadly used algorithm for inference of latent variable problems. A theoretical understanding of its performance, however, largely remains lacking. Recent results established that EM enjoys global convergence for Gaussian Mixture Models. For Mixed Linear Regression, however, only local convergence results have been established, and those only for the high SNR regime. We show here that EM converges for mixed linear regression with two components (it is known that it may fail to converge for three or more), and moreover that this convergence holds for random initialization. Our analysis reveals that EM exhibits very different behavior in Mixed Linear Regression from its counterpart in Gaussian Mixture Models, and hence our proofs require the development of several new ideas.%
\footnote{This paper results from a merger of work from two groups who work on the problem at the same time.}
\end{abstract}

\section{Introduction}
The expectation-maximization (EM) algorithm is a general-purpose technique for computing the maximum likelihood solution for problems with missing data, often modeled as latent variables ~\citep{dempster1977maximum, wu1983convergence}. In general, maximizing the likelihood in the presence of missing data is an intractable problem due to the non-convexity of the log-likelihood function. EM is an iterative procedure that computes successively tighter lower bounds of the log-likelihood function. 
Despite its simplicity and its widespread use in practice, relatively little is understood about the theoretical properties of EM. Recent results have demonstrated that in the high SNR regime (and under additional regularity assumptions), EM converges locally (e.g., \citealt{yi2015regularized, balakrishnan_statistical_2017, klusowski2019estimating, yi2014alternating, yi2016solving}). For the special case of Gaussian Mixture Models (GMM) with two components, very recent work \citep{daskalakis2017ten} has shown that a two-phase version of EM converges from random initialization. As far as we know, no comparable global convergence result is known for Mixed Linear Regression (MLR), despite the empirical success of EM in this problem \citep{jordan1994hierarchical, de1989mixtures}.

The lack of global convergence guarantees for EM under MLR is not simply an oversight. Rather, as we show later, MLR exhibits very different behavior from GMM, even on the population (infinite sample) level. Existing techniques used to analyze EM under GMM---often based on $\ell_2$ distance contraction---are fundamentally insufficient for establishing global convergence of EM for MLR.

In this work, we show for the first time that EM for MLR with two components converges globally without the need for any special initialization. Moreover, our proof reveals (a bound on) the rate of convergence of EM as a function of how far it is from the true parameter. Locally, we improve upon past results, as these not only required an initialization step and a high SNR assumption, but also failed to provide a tight final error bound that correctly captures the improvement achieved by EM over the initial solution. We explain connections to prior art in more details in Section~\ref{ssec:relatedwork}.

\subsection{Basic Setup and the EM Algorithm}
\label{ssec:setup}
Mixed linear regression (MLR) models the setting where different subsets of the response variables are generated by different regressors. In the case of two components, which we consider here, the data $(\bm{x}_i,y_i) \in \mathbb{R}^d \times \mathbb{R}$ are generated by a mixture of two linear models with unknown regressors $\pm \betastar \in \mathbb{R}^d$:
\begin{equation}
\label{eq:mlr_setup}
	y_i = z_i \bm{\beta}^* \bm{x}_i + e_i,	\qquad i = 1,...,n,
\end{equation}
where $z_i \in \{\pm 1\}$ are the hidden/latent variables, which play the role of labels denoting whether a data point $(\bm{x}_i,y_i)$ is generated by  $+\bm{\beta}^*$ or $-\bm{\beta}^*$. Finding the true parameter $\bm{\beta}^*$ is known to be NP-hard in general \citep{yi2014alternating} even without noise. Accordingly, a common assumption in the literature stipulates that the covariates and noise terms, $\bm{x}_i$ and $e_i$, are sampled independently from Gaussian distributions, that is, $\bm{x}_i \sim \mathcal{N}(0, I_d)$ and $e_i \sim \mathcal{N}(0, \sigma^2)$, where $\sigma$ is known. We assume, moreover, that the hidden variables $z_i$ take values $\pm1$ with equal probability and are independent of everything else.


At each iteration, the EM algorithm performs two steps: the E-step that computes the expectation of the log likelihood function conditioned on the current estimate of $\betastar$, and the M-step that maximizes this expectation.
For MLR, when we plug in the likelihood of the assumed Gaussian distribution and replace the expectation with an empirical average over observed data $\{\bm{x}_i,y_i\}$, the $M$-step becomes the familiar (weighted) least squared loss minimization problem. In this case, the sample-based EM update with the current estimator $\bm{\beta}$ has the following closed form expression (for a derivation see \cite{balakrishnan_statistical_2017, klusowski2019estimating}):
\begin{equation}
\label{eq:EM-CF}
\mbox{(EM)} \qquad	\tilde{\bm{\beta}}' = \Big( \frac{1}{n} \sum_{i=1}^n \bm{x}_i \bm{x}_i^{\top} \Big)^{-1} \bigg( \frac{1}{n} \sum_{i=1}^n \tanh \Big( \frac{\langle \bm{\beta}, \bm{x}_i \rangle}{\sigma^2} y_i \Big) y_i\bm{x}_i \bigg).
\end{equation}
In the setting where the covariates $\{\bm{x}_i\}$ have identity covariance (or have been normalized to so), it is also interesting to consider the following simplified version of EM (we call it  ``Easy-EM'') that replaces the matrix $\frac{1}{n} \sum_{i=1}^n \bm{x}_i \bm{x}_i^\top$ with its expectation:
\begin{equation}
\label{eq:EEM-CF}
\mbox{(Easy-EM)} \qquad	\tilde{\bm{\beta}}'' =  \frac{1}{n} \sum_{i=1}^n \tanh \Big( \frac{\langle \bm{\beta}, \bm{x}_i \rangle}{\sigma^2} y_i \Big) y_i\bm{x}_i.
\end{equation}

The contribution of this work is to analyze these two iterations in the finite-sample setting, and thereby to provide guarantees for their convergence from a random initialization. 

\subsection{Related Work and Main Contributions}
\label{ssec:relatedwork}

As mentioned above, our knowledge of when EM converges to a true solution is still limited. In general, it is known that the EM algorithm may settle in a bad local optimum unless it starts from a well initialized point \citep{wu1983convergence}. Recent progress on the theoretical understanding of EM has been made in \cite{balakrishnan_statistical_2017}, which proposed a novel framework to analyze the EM algorithm. Motivated by this work, there have been some positive results for two related problems: GMMs and MLRs \citep{daskalakis2017ten, xu2016global, yi2015regularized, balakrishnan_statistical_2017, klusowski2019estimating, yi2014alternating, yi2016solving}.  

For GMM with two components, \cite{daskalakis2017ten, xu2016global} provide a global analysis of EM for the mixture of two Gaussians and deliver results that guarantee convergence of EM for this specific problem from a random initialization. For GMM with more components, however, it is known that EM does not converge globally \citep{jin2016local}.

For MLR with two components, only the local convergence of EM has been recently established: it is known that the EM algorithm does converge to the global optimum if we start from a point sufficiently close to the true parameter \citep{yi2015regularized, balakrishnan_statistical_2017, yi2014alternating, yi2016solving}. A better local contraction region was suggested in \cite{klusowski2019estimating}, where the convergence is guaranteed inside a region where the angle formed by the initialization with the true parameter is small. Still, all known results remain inherently local for MLR, and in particular, are not satisfied by a random initialization, even when a norm bound on the true parameter is known.

MLR is an interesting problem by itself, for which many algorithms have been proposed. 
The work in \cite{chen2014convex} developed a lifted convex formulation approach that achieves tight minimax error rates. A good initialization strategy for EM based on Stein's second-order lemma was proposed in \cite{yi2014alternating}, though this seems to rely on the noiseless setting which they study. The above two papers have focused on the mixture of two components case. Recent work has extended the focus to more components. Work in \cite{li2018learning, zhong2016mixed} develops gradient descent based algorithms. In parallel, the work in \cite{yi2016solving, chaganty2013spectral, sedghi2016provable} utilizes tensor decomposition of third order moments.

The question of whether EM converges from a random initialization for MLR with two components still remains open. Our main contribution is to resolve this question affirmatively. \\

{\bf Main Contributions}. We prove the global convergence of the EM algorithm, i.e., it converges with probability one from a random initialization. We first establish this
result for the infinite sample limit, i.e., the population EM, by analyzing its trajector along the landscape of the likelihood function. We then couple the finite-sample EM with the population EM, thereby providing a finite sample analysis. This coupling idea is inspired by \cite{balakrishnan_statistical_2017}, but our strategy and most of the technical details differ. In particular, we control not only the $\ell_2$ distance between the population and finite-sample EM, but also, crucially, the \emph{angle} between them. As we comment on in greater details below, this coupling is strong enough to guarantee convergence even when the current EM iterate is far from the desired solution; moreover, it yields near-optimal sample complexity bounds that improve upon the results in \cite{balakrishnan_statistical_2017}, particularly in terms of the dependence on the signal-to-noise ratio.

\subsection{A Roadmap and Proof Outline}
\label{ssec:roadmap}
We provide a brief outline of the main steps of the paper.

\paragraph{Analysis of the population EM:} 
 
\begin{itemize}
    \item {\bf Landscape}. As mentioned, previous work on analyzing the EM algorithm for MLR
    relies on demonstrating that the $\ell_2$ distance between the current iterate and the true solution $\bm{\beta}^*$, contracts at every iteration provided that the initial distance is already small. 
    Such a contraction, however, cannot hold globally, as the EM update initialized randomly may in fact result in a {\em larger distance} from $\bm{\beta}^*$. This phenomenon was pointed out in \cite{klusowski2019estimating}. We provide a geometric explanation in this paper by showing the existence of saddle points of the log-likelihood function in the direction orthogonal to $\betastar$. These saddle points prevent a global convergence in $\ell_2$ distance of EM (which is equivalent to gradient ascent). On the other hand, we show that $\pm\betastar$ are the only local maxima, hence suggesting that global convergence can be proved by other means.
    \item {\bf Decreasing Angle}. Instead of proving a global convergence via the $\ell_2$ distance, we show that the angle between the iterate and $\bm{\beta}^*$ is always decreasing (unless we start from an exactly orthogonal vector---a measure zero event). 
    Consequently, EM quickly enters a local region where the current iterate is well aligned with the direction of  $\bm{\beta}^*$. In this local region, we show that a contraction in distance indeed holds. 

    \item {\bf Escaping Nearly Orthogonal Region}.
    Random initialization in a $d$-dimensional space typically yields a vector whose correlation with $\bm{\beta}^*$ is $O(1/\sqrt{d})$. In this region, the contracting behavior of the angle can be very subtle. 
    We provide a fine-grained analysis in this region, showing that first the cosine of the angle increases geometrically, and then the sine of the angle decreases geometrically afterwards. Consequently, in a logarithmic number of steps EM escapes this nearly orthogonal region and attains a constant correlation with $\betastar$.
    \item {\bf Low SNR}. Besides being local in nature, previous results are dependent on the high SNR assumption \citep{yi2014alternating, balakrishnan_statistical_2017, yi2015regularized,wu2016convergence}; that is, the standard deviation, $\sigma$, of the additive noise, is sufficiently smaller than the norm of the true parameter. 
    Our analysis is applicable in both low and high SNR regimes, and reveals an explicit convergence rate as a function of the noise level.

\end{itemize}

\paragraph{Analysis of the Finite Sample EM:}

\begin{itemize}
    \item {\bf Coupling in Angle}. We analyze the finite-sample EM update by coupling it with the population EM. \cite{balakrishnan_statistical_2017} provided a bound between these two updates in $l_2$ distance. Since our argument is based on contraction of angle, we need to establish additional concentration inequalities in order to bound the cosine and sine of the angle. We then conclude that, starting from a random initial guess in $d$-dimensional space, with $n = \tilde{O}(\max(1,\eta^{-2}) d/\epsilon^2)$ fresh samples in each iteration, the finite-sample EM yields an estimate with an $l_2$ error bounded by $O(\epsilon)$ after $T = O(\max(1, \eta^{-2}) \max(\log d, \log (1/\epsilon)))$ iterations. $\eta$ here is the notation for signal-to-noise ratio (SNR).

    
    
    \item {\bf Statistical Error}. In the high SNR regime, we further refine the finite sample analysis and show that the EM algorithm in fact achieves an error of $\tilde{O}\left(\sigma \sqrt{d/n}\right)$. Note that the error rate is independent of the signal strength $\|\bm{\beta}^*\|$. This is in a stark contrast to all the previous analysis of EM which proved an error of $\tilde{O}\left(\sqrt{\sigma^2 + \|\bm{\beta}^*\|^2} \sqrt{d/n}\right)$ for MLR \citep{balakrishnan_statistical_2017, klusowski2019estimating} ---such an error bound is no better than the bound achieved by a simple spectral initialization and in particular cannot guarantee exact recovery in the noiseless setting.
    
    \item {\bf Analysis of Easy-EM}.  For the early iterations of EM where the cosine between the estimate (or a random initialized point) and $\betastar$ can be as small as $1/\sqrt{d}$, we can instead run Easy-EM, which does not need the computation of the inverse of the sample covariance. Easy-EM also provides a guarantee for reaching an angle larger than $O(1)$, while in our analysis, standard EM requires an additional condition that the statistical fluctuation, due to the size finite samples, should be less than $O(1/\sqrt{d})$. Therefore our results indicate that one can run Easy-EM until the cosine of the angle between the current estimate and the true parameter is large, and subsequently run EM.

\end{itemize}

\paragraph{Paper Organization}
In Section \ref{sec:population_EM_update}, we derive a closed form equation of the population EM and prove some of its structural properties. Section \ref{sec:large_main_results} is devoted to summarize our results on the global convergence of the population EM. 
The analysis on the finite-sample EM is provided in Section~\ref{sec:finite_sample}. All technical proofs that are not given in the main paper are deferred to the Appendix.

\section{The Population EM Update}
\label{sec:population_EM_update}
In this section, we consider the infinite-sample limit of the EM update (i.e., the population EM) and discuss its basic properties. This discussion highlights the main challenges in the MLR problem and the reasons why they can be resolved. It also serves as a starting point of our subsequent proof for global convergence. 

\subsection{Basic Notation}


We use $\angle(\bm{u},\bm{v})$ to denote the angle between two vectors $\bm{u}$ and $\bm{v}$. The norm operator $\|\cdot\|$ without subscript is taken as the $l_2$ norm for a vector or the operator norm for a matrix. $\langle \cdot,\cdot\rangle$ denotes the usual inner product: $\langle \bm{u},\bm{v}\rangle = \bm{u}^{\top}\bm{v}$ for  $\bm{u},\bm{v}\in \mathbb{R}^d$. 

We use $(X,Y)$ as a generic random variable representing the covariate and response variables of MLR, and use  $\left\{(\bm{x}_i, y_i)\right\}$ as independent copies of $(X,Y)$. Due to a symmetry between the regressors $\pm \betastar$, we focus on the convergence to one of them, say $\bm{\beta}^{\ast}$. Accordingly, at the $t^{th}$ iteration of the algorithm, $\bm{\beta}_t$ is the current estimate of $\bm{\beta}^*$. 
When we are interested in understanding a single iteration, we drop the subscript $t$ and use $\bm{\beta}$ in place of $\bm{\beta}_{t}$, and $\bm{\beta}'$ in place of $\bm{\beta}_{t+1}$. We use $\theta_t:=\angle(\bbeta_t, \betastar)$ to denote the angle formed by $\bm{\beta}_t$ and $\bm{\beta}^*$ , and similarly $\theta_{t+1} := \angle(\bbeta_{t+1}, \betastar)$. For a single iteration, we use $\theta$ for $\theta_t$ and and $\theta'$ for $\theta_{t+1}$. We assume without loss of generality that the initial angle $\theta_0$ is in $[0, \pi/2)$, where $\pi/2$ is excluded as it has measure zero. An initialization falling in the remainder of the circle has precisely the same behavior, but with a convergence to $-\bm{\beta}^*$ instead of $\bm{\beta}^{\ast}$. 

$\sigma$ is the standard deviation of the noise $e$ and assumed to be known. We define the signal-to-noise ratio (SNR) of the problem as $\eta := \frac{\|\bm{\beta}^*\|}{\sigma}$.

\subsection{An Explicit Expression for the Population EM Update}

As in \cite{balakrishnan_statistical_2017}, we consider the following population EM update
\begin{equation}
\label{eq:em_update_full}
	\bm{\beta}_{t+1} = \mathbb{E}_{X \sim \mathcal{N}(0, I)} \left [ \left( \mathbb{E}_{Y|X \sim \mathcal{N}(\langle X, \bm{\beta}^* \rangle, \sigma^2)} \left[ \tanh \left( \frac{\langle X, \bm{\beta}_t \rangle}{\sigma^2} Y \right)Y \right] \right)X \right].
\end{equation}
The above expression follows from taking the limit $n\to \infty$ in the EM update formula~\eqref{eq:EM-CF} and simplifying the result using the symmetry of the distribution of $Y$ given $X$. We refer to \cite{balakrishnan_statistical_2017} for the details of this standard derivation.



We focus on one iteration of the population EM which yields the next iterate $\bbeta'$. It is convenient to change the basis by choosing $\bm{v}_1 = \bm{\beta}/\|\bm{\beta}\|$ in the direction of the current iterate and $\bm{v}_2$ to be the orthogonal complement of $\bm{v}_1$ in ${\rm span}\{\bm{\beta},\bm{\beta}^{\ast}\}$. We expand them to an orthonormal basis $\{\bm{v}_1, ..., \bm{v}_d\}$ in  $\mathbb{R}^d$. Introduce the shorthand $b_1 := \langle \bm{\beta}, \bm{v}_1 \rangle = \|\bm{\beta}\|$, $b_1^* := \langle \bm{\beta}^*, \bm{v}_1 \rangle$ and $b_2^* := \langle \bm{\beta}^*, \bm{v}_2 \rangle$. Using the spherical symmetry of the distribution of $X$,  we may write the next iterate $\bbeta'$ as
\begin{equation}
	\bm{\beta}' = \mathbb{E}_{\alpha_i} \left[ \mathbb{E}_{y|\alpha_i} \left[ \tanh \left(\frac{b_1 \alpha_1}{\sigma^2} y \right)y \right] \sum_i \alpha_i \bm{v}_i \right],
\end{equation}
where the expectation is taken over $\alpha_i \sim \mathcal{N}(0,1)$ and $y|\alpha_i \sim \mathcal{N}(\alpha_1 b_1^* + \alpha_2 b_2^*, \sigma^2)$. Without loss of generality, we assume $b_1, b_1^*, b_2^* \ge 0$. The lemma below plays a key role in our later development. It provides an explicit expression of $\bbeta'$ in terms of the above basis system, which, among other things, implies that $\bm{\beta}'\in$ span $(\bm{\beta},\bm{\beta}^*)$ (and hence $\bm{\beta}_t\in$ span $(\bm{\beta}_0, \bm{\beta}^*)$ for all $t$). 
\newcommand{\tanhcustom}{\tanh{ \left(\frac{\alpha_1 b_1}{\sigma^2}(y+\alpha_1 b_1^*) \right)}}
\newcommand{\tanhpcustom}{\tanh'{ \left(\frac{\alpha_1 b_1}{\sigma^2}(y+\alpha_1 b_1^*) \right)}}

\begin{restatable} {lem}{simplifiedupdate}
\label{lem:em_update_summary}
Define $\sigma_2^2 := \sigma^2 + {b_2^*}^2$. We can write $\bm{\beta}' = b_1' \bm{v}_1 + b_2' \bm{v}_2$, where $b_1'$ and $b_2'$ satisfy
\begin{equation}
	b_1' = b_1^* S + R, \quad\text{and}\quad
    b_2' = b_2^* S, \label{eq:em_update}
\end{equation}
where $S \geq 0$ and $R > 0$ are given explicitly in (\ref{eq:s_r_expression}) in Appendix \ref{appendix:lemma_1}. Moreover, $S=0$ iff $b_1^*=0$.
\end{restatable}

\subsection{Structural Properties of the Population EM}
\label{sec:structural_properties}

Note that the quantities $b_1'$ and $b_2'$ in Lemma~\ref{lem:em_update_summary} represent the projections of $\bbeta'$ in and orthogonal to the direction of $\bbeta$. From the expression of $b_2'$,  we immediately deduce the following structural property of the population EM update:
\begin{enumerate}
    \item[1.] \textbf{Decreasing angle}: $\bm{\beta}'$ forms a smaller angle with $\bm{\beta}^*$ compared to $\bm{\beta}$. 
    To see this, note that $0\le \tan\angle(\bbeta', \bbeta) = \frac{b_2'}{b_1'} \leq  \frac{b_2^*}{b_1^*} = \tan\angle(\betastar, \bbeta)$.
    When $\frac{b_2'}{b_1'}>0$, the angle strictly decreases; when $\frac{b_2'}{b_1'}=0$, the angle remains the same. In particular, $\frac{b_2'}{b_1'}=0$ holds iff $b_2'=0$, that is, either $b_2^*=0$ (i.e., $\bm{\beta}\in $ span$(\bm{\beta}^*)$) or $S=0$ (i.e., $\bm{\beta} \perp \bm{\beta}^*$).
\end{enumerate}
From the expression of  $b_1'$,  we deduce the following (cf.\ Lemma~\ref{lem:dynamic_along}):
\begin{enumerate}
    \item[2.] \textbf{Contraction along $\bm{\beta}$}: In the direction of $\bm{v}_1$ (equivalently, $\bbeta$), $\bm{\beta}'$ moves towards a unique fixed point $E(\bm{v}_1)$; i.e., $|b_1' - E(\bm{v}_1) | \le | b_1 - E(\bm{v}_1)| $ with equality holds iff $b_1 = E(\bm{v}_1)$.
\end{enumerate}
It is also easy to see that the iterates remain bounded: $\|\bbeta'\| \le 3\sqrt{\sigma^2+\|\bm{\beta}^*\|^2}$ (cf.\ Lemma~\ref{lemma:bprime_bounded}). 


Interestingly, it can be shown that the population EM update is equivalent to applying \emph{gradient ascent} with a fixed step size 
to the population log likelihood function of MLR. Building on the above structural properties, we obtain the following complete characterization of the fixed points of the population EM as well as the stationary points of the population log likelihood.

\begin{restatable}[Population EM and Log-likelihood]{theorem}{stationarypoints}
\label{thm:stationary_points}
For each nonzero $\bbeta$ not parallel to $\betastar$, in span$(\bbeta, \betastar)$, the set of fixed points of the population EM is equal to the set of stationary points of the log-likelihood. This set contains exactly five elements: (i) $\betastar$ and $-\betastar$, which are global maxima; (ii) $\bm{0}$, which is a local minimum; (iii) $ E(\bm{v})\bm{v}$ and $- E(\bm{v})\bm{v}$, where $\bm{v}\perp\bm{\beta}_*$. Moreover, a stationary point in the orthogonal space is a saddle point whose Hessian has a strictly positive eigenvalue.
\end{restatable}

As $\pm \betastar$ are the only local maxima, it becomes less surprising that the population EM (equivalent to gradient ascent) converges to them from a random initialization. On the other hand, with the existence of saddle points, it is easy to see that the $\ell_2$ distance to $\betastar$ \emph{cannot} contract globally; that is, $\| \bbeta' -\betastar \| > \| \bbeta -\betastar \| $ for some $\bbeta$. Note that GMM does not have such saddle points, and the $\ell_2$ distance does decrease globally as is established in a previous work~\citep{daskalakis2017ten}.


\section{Main Results on the Population EM}
\label{sec:large_main_results}

In this section, we provide our main results on the global convergence of the population EM. We adopt a new strategy for the convergence analysis to get around the aforementioned challenge based on the contraction of the $\ell_2$ distance. We first prove a rapid decrease in angle and then show a geometric decrease in distance. The convergence result in three phases is summarized below:

\begin{enumerate}
\item {\bf Increasing Cosine:} Starting from a randomly initialized vector in $\mathbb{R}^d$, after $O(\max(1, \eta^{-2}) \log d)$ iterations, EM outputs a vector whose angle with $\betastar$ is less than $\pi / 3$.

\item {\bf Decreasing Sine:} Starting from a vector whose angle with $\betastar$ is less than $\pi / 3$, after $O(\max(1, \eta^{-2}))$ iterations, EM outputs a vector whose angle with $\bm{\beta}^*$ is less than $\pi / 8$.

\item {\bf Convergence in $\ell_2$:} Starting from a vector whose angle with $\betastar$ is less than $\pi / 8$, after $O(\max(1, \eta^{-2}) \log(1/\epsilon))$ iterations, EM outputs an estimate of $\betastar$ whose $\ell_2$ error is $O(\epsilon)$.
\end{enumerate}

All the above results hold for an arbitrary SNR, thus improving on previous results that are only established in the high SNR regime.

\subsection{Convergence of Cosine} \label{sec:conv_cosine}

Recall that $\theta_0, \theta$ and $\theta'$ denote the angles that $\betastar$ forms with $\bbeta_0$ (initial iterate), $\bbeta$ (current iterate), and $\bbeta'$ (next iterate), respectively. By symmetry we may assume w.l.o.g.\ that $\cos\theta_0 $ is positive. Note that  $\cos\theta_0 = \Theta(1/\sqrt{d})$ with high probability. For the early stage of iterations, we focus on the cosine of the angle and show that it increases geometrically. Therefore, starting from $\cos \theta_0 = \Theta(1/\sqrt{d})$, a logarithmic number of iterations of EM is sufficient to guarantee $\cos \theta_t = O(1)$.
\begin{restatable}[Cosine Convergence]{theorem}{theoremcosine}
\label{theorem:cosine}
	As long as $\frac{\pi}{2} > \theta \ge \frac{\pi}{3}$, each population EM iteration satisfies 
\begin{align}
\label{ineq:cosine_increase}
	\cos \theta' \ge \kappa \cos \theta, 
\end{align}
where $\kappa = \sqrt{1 + \frac{\eta^2}{\frac{2}{3} + \eta^2}} > 1$. 
Consequently, if $\cos \theta_0 = \Theta(1/\sqrt{d})$, after $T = O(\log(d) \max(1, \eta^{-2}))$ iterations, we get $\theta_T < \pi / 3$ or $\cos \theta_T \ge \frac{1}{2}$.
\end{restatable}

The proof is in Appendix \ref{appendix:theorem_cosine}. From the proof, it shows that $\cos(\theta^\prime)\geq \cos(\theta)\sqrt{1 + \frac{\sin^2 \theta}{\cos^2 \theta + \frac{1}{2} (1+\eta^{-2})} }$. However, the ratio between $\cos \theta'$ and $\cos \theta$ approaches 1 as $\theta$ goes to 0. In other words, cosine angles are not informative for establishing a constant convergence factor bounded away from 1 when $\theta$ is small. In the following subsection, we state a similar result for sine of the angle to complement this result.

\subsection{Convergence of Sine}  \label{sec:main_sine}

We next show that the sine of the angle converges geometrically  to 0. This is reminiscent of the proof for Theorem 3 in \cite{xu2016global}, where they used a similar logic to show \emph{asymptotic} convergence. Here we provide an explicit rate of convergence by quantifying the amount of increase in sine, which is critical in order to port the population-level results to the finite sample setting. 

\begin{restatable}[Sine Convergence] {theorem}{theoremsine}
\label{theorem:sine}
	As long as $0 \le \theta < \frac{\pi}{2}$, each population EM iteration satisfies
\begin{equation}
\label{ineq:sine_increase}
\sin \theta' \le \kappa \sin \theta,
\end{equation}
where $\kappa = \left(\sqrt{1 + \frac{2\eta^2}{1+\eta^2} \cos^2\theta}\right)^{-1} < 1$.
\end{restatable}

It is proved in Appendix \ref{appendix:sine}. Note that the speed of convergence increases as the angle decreases. This result is most useful when the angle is bounded away from $\pi/2$---complementary to the case covered by Theorem~\ref{theorem:cosine}. We also remark that in a high SNR regime ($\eta \gg 1$), $\kappa$ can be much smaller than 1 (depending on the initial angle); in a low SNR ($\eta \ll 1$) regime, however, the convergence rate cannot be faster than $1 - O(\eta^2)$, regardless of the initial angle.



\subsection{Convergence of Distance}
\label{sec:conv_distance}


Combining the above results on cosine and sine, we can conclude that eventually EM pushes any random initialization into a region with a small angle around $\betastar$. 
At this point, EM safely transits to the region of contraction in distance, which is the content of our next result.

\newcommand {\newbeta}{\frac{\sigma_2^2}{\sigma^2} b_1}
\begin{restatable}[$\ell_2$ Convergence] {theorem}{theoremdist}
\label{theorem:conv_distance}
	Assume that $\theta < \pi/8$, and define $\sigma_2^2 = \sigma^2 + {b_2^*}^2$. If $b_2^* < \sigma$ or $\newbeta < b_1^*$, then we have
	\begin{subequations}
	\begin{equation}
		\label{eq:b1p_conv_large_sig}
		\|\bm{\beta}' - \bm{\beta}^*\| \le \kappa \|\bm{\beta} - \bm{\beta}^*\| + \kappa (16\sin^3 \theta) \|\bm{\beta}^*\| \frac{\eta^2}{1+\eta^2},
	\end{equation}
    {where} $\kappa = \left({\sqrt{1 + {\min(\newbeta, b_1^*)^2}/{\sigma_2^2}}} \right)^{-1}$. 
    {Otherwise, we have}
    \begin{equation}
    \label{eq:b1p_conv_small_sig}
    	\|\bm{\beta}' - \bm{\beta}^*\| \le 0.6 \|\bm{\beta} - \bm{\beta}^*\|.
    \end{equation}
	\end{subequations}
\end{restatable}

 In order to give a geometrically decaying error bound, we have an additional term in (\ref{eq:b1p_conv_large_sig}) that depends on angle and SNR. When $b_1$ is close to $b_1^*$ and $\sigma$ is small, we get a better contraction (\ref{eq:b1p_conv_small_sig}). The detailed proof is in Appendix \ref{appendix:theorem_distance}.

With the above per-iteration contraction result, we can bound the $\ell_2$ error after $t$ iterations of population EM and conclude that it convergence to $\betastar$.

\begin{restatable}{corol}{corollarydist}
\label{corollary:distance}
Assume we start from $\theta_0 < \pi/8$. After $T$ iterations of the population EM, there exists some constant $\kappa < 1$ such that
\begin{equation}
    \label{eq:large_sig_converge}
    	\|\bm{\beta}_T - \bm{\beta}^*\| < \kappa^{T} \|\bm{\beta}_0 - \bm{\beta}^*\| + T \kappa^{T} \|\bm{\beta}^*\| \frac{\eta^2}{1 +  \eta^2}.
\end{equation}
In particular, the result is satisfied if we take $\kappa$ to be the maximum among 
\begin{align}
\label{eq:kappa_candidates}
0.6, \; 
\sqrt{\left( {1+\frac{\|\bm{\beta}_0\|^2}{\sigma^2}} \right)^{-1}}, \;
\sqrt{1 - \frac{0.8\eta^2}{1+\eta^2}}.
\end{align}
\end{restatable}

The convergence rate depends on the SNR $\eta$ and the norm of an initial guess. For different $\eta$'s, the rate is either a constant or $1 - O(\eta^2)$, as was in the case of sine. Therefore,  $T=O(\max(1, \eta^{-2}) \log(1/\epsilon))$ iterations is sufficient to achieve an $\epsilon$-optimal solution. 
In the Appendix, we show that the convergence rate only becomes faster as the algorithm proceeds.

\section{Finite Sample Analysis}
\label{sec:finite_sample}
We now turn to prove the convergence of the finite-sample EM given in Eq.~(\ref{eq:EM-CF}).  
Along the way, we also prove the convergence of the Easy-EM algorithm given in Eq.~\eqref{eq:EEM-CF}. As we discuss in length below, Easy-EM is not only interesting on its own, but also useful in the setting where the ``statistical fluctuation'' $\epsilon_f \propto \sqrt{d/n}$ between the population and the finite-sample EM updates---which is determined by the sample size $n$---is $O(1)$ rather than $O(1/\sqrt{d})$. 


In this section, we use $\bm{\beta}$ to denote our current iterate, $\bm{\beta}'$ for the output from one step of the {\em population} EM, and $\tilde{\bm{\beta}}'$ for the output from one step of the {\em finite-sample} EM. Accordingly,  $\tilde{\theta}'$ denotes the angle between $\tilde{\bm{\beta}}'$ and $\bm{\beta}^*$. When we consider the sequence of iterates generated by the finite-sample EM, we use $\tilde{\bbeta}_t$ for the $t^{th}$ iterate and $\tilde{\theta}_t$ for its angle with $\betastar$. 
Similarly to the population EM discussed in previous section, we will show that the finite-sample EM converges in several phases:
\begin{enumerate}
\item {\bf Possible initialization with Easy-EM:} Starting from a randomly initialized vector with large enough norm in $d$-dimensional space, compare the statistical fluctuation $\epsilon_f$ to $1/\sqrt{d}$. If it is smaller than $1/\sqrt{d}$, then go to step 2. Otherwise, run Easy-EM for $O(\log(\epsilon_f\sqrt{d}) \max(1, \eta^{-2}))$ iterations to get $\cos \tilde{\theta}_t \ge \epsilon_f$. 

\item {\bf Increasing Cosine:} Starting from the vector obtained from the last step, run the finite-sample EM for $O(\min(\log d, \log (1/\epsilon_f)) \max(1, \eta^{-2}))$ iterations to get $\cos \tilde{\theta}_t \ge 1/2$.

\item {\bf Decreasing Sine:} Starting from a vector with cosine of its angle $\cos \tilde{\theta}_0$ at least $1/2$, run the finite-sample EM for $O(\max(1, \eta^{-2}))$ iterations to get $\sin \tilde{\theta}_t \le \sin (\pi/70)$. 

\item {\bf Convergence in $\ell_2$:} Starting from $\tilde{\theta}_0 \le \pi/70$, run the finite-sample EM for $O(\max(1, \eta^{-2}) \log(1/\epsilon))$ iterations to get $\| \tilde{\bm{\beta}}_t - \bm{\beta}^*\| \le O(\epsilon)$. 
\end{enumerate}

Collecting all the steps, we obtain the following overall guarantee:
\begin{theorem}[The Finite Sample EM]
\label{theorem:finite_main}
    Suppose we start from an initial vector in $ \mathbb{R}^d$ whose correlation with $\bm{\beta}^*$ is at least $\Omega \big( 1/\sqrt{d} \big)$, with $\ell_2$ norm at least $\|\bm{\beta}^*\| / 10$. We run the sample-splitting finite-sample EM with $O (\max(1, poly(\eta^{-1}))$ $(d/\epsilon^2) \log(T/\delta) )$ fresh samples in each iteration. After $T = O(\max(1,\eta^{-2})$ $\max(\log d, \log (1/\epsilon)) )$ iterations, we get 
    $$
        \mathbb{P}(\| \tilde{\bm{\beta}}_T - \bm{\beta}^*\| \le \epsilon) \ge 1 - \delta.
    $$
\end{theorem}

\subsection{Analysis of the Finite-Sample EM}

We now provide the details for the four-phase convergence outlined above. 
We consider sample-splitting as an analysis technique, as it renders subsequent iterations of the EM algorithm independent. As with many other papers that have used this analysis technique, we believe it is an artifact of the analysis,  but we are unable to find a way to remove it. 


As discussed in the introduction, our approach is to couple the finite sample EM to the population EM. Work in \cite{balakrishnan_statistical_2017} establishes a bound on the $\ell_2$ distance between the population and the finite-sample EM in the form of $\|\tilde{\bm{\beta}}' - \bm{\beta}'\| = O\Big(\sqrt{d/n}\Big). $ This type of bound implies local contraction in distance; however it is not sufficient for us, as we need to control the angle outside of the local contraction region. Here we prove a more fine-grained result of the form $|\langle \tilde{\bm{\beta}}' - \bm{\beta}', \bm{\beta}^* \rangle|= O \left(1/\sqrt{n} +  d/n \right)$ (cf.\ Theorem \ref{thm:finite_cos_update} in Appendix \ref{appendix:aux_concentration_one_dir}). This allows us to show that the finite-sample EM decreases the angle up to a statistical fluctuation per iteration.

\begin{restatable}{theorem}{finitecosineupdate}
\label{theorem:finite_cosine_update}
    Suppose that $\|\bm{\beta}\|\geq \|\bm{\beta}^*\| / 10$. Then, with $n= \tilde{O}(\max(1,\eta^{-2})d/\epsilon_f^2)$ samples for one finite-sample based EM iteration, we have
    \begin{subequations}
    \begin{align}
	    &\cos \tilde{\theta}' \ge \kappa(1 - 10\epsilon_f) \cos \theta - O \bigg( \max \Big( \frac{\epsilon_f}{\sqrt{d}}, \epsilon_f^2 \Big) \bigg), \label{ineq:cosine_sample} \\
	    &\sin^2 \tilde{\theta}' \le \kappa' \sin^2 \theta + O( \epsilon_f ), \label{ineq:sine_sample}
    \end{align}
    \end{subequations}
    with $\kappa = \sqrt{1 + \frac{\sin^2 \theta}{\cos^2 \theta +  \frac{1}{2}(1+\eta^{-2}) }} \ge 1$, and $\kappa' = \left(1 + \frac{2\eta^2}{1+\eta^2} \cos^2\theta \right)^{-1} < 1$.
\end{restatable}

The theorem implies that the cosine and sine of the angle improves, up to a quantity that depends on $\epsilon_f \propto \sqrt{d/n}$ (and hence on the sample size). We note the extra factor $\epsilon_f^2$ in the bound. Technically, this arises from controlling the random fluctuation of the inverse sample covariance matrix $(\frac{1}{n} \sum_{i=1}^n \bm{x}_i \bm{x}_i^{\top})^{-1}$.
We provide two sufficient conditions under which this term is negligible: (i) $\epsilon_f$ is small enough, namely, $<1/\sqrt{d}$, or (ii)  $ \langle \bm{\beta}_0, \bm{\beta}^{\ast}\rangle > \epsilon_f $, in other words, the initialization is good (cf.\ proofs of Theorem \ref{thm:finite_cos_update} and Corollary \ref{theorem:finite_cosine_update} in the Appendix).
In Section \ref{ssec:easyem} we show that Easy-EM exhibits a very similar convergence behavior, without the appearance of the $\epsilon_f^2$ term. Therefore, if $\epsilon_f>1/\sqrt{d}$, one can simply run Easy-EM until the estimate has enough correlation with $\bbeta^*$ (i.e., $ \langle \bm{\beta}_t, \bm{\beta}^{\ast}\rangle > \epsilon_f $), and then switch to EM. 

For now, we assume that one of the conditions described above holds, and thus we can assume that the $\epsilon_f^2$ term can be safely ignored. 

With the per-iteration bounds in (\ref{ineq:cosine_sample}) and (\ref{ineq:sine_sample}), we can bound the angle after $T$ steps of the finite-sample EM and thereby guarantee achieving a final error of $\epsilon$. 
It will become clear that 
\begin{align}
\epsilon = \epsilon_f \max(1,\eta^{-2})
\end{align}
(cf.\ Proof for Lemma \ref{lemma:finite_cosine}, Lemma \ref{lemma:finite_sine} and Lemma \ref{thm:final_stats_error}) since the final error has an accumulation of statistical fluctuations (quantified by $\epsilon_f$) from $T = \tilde{O}(\max(1, \eta^{-2}))$ iterations. 
\yccomment{We could perhaps interpret $\epsilon$ as the ``statistical fluctuation'' (between the population and sample EM iterations), and interpret $\epsilon_1$ and $\sigma O(\epsilon)$ etc.\ as the ''statistical error'' (of the final estimate obtained after a sequence of EM iterations). No abuse of notation then?} \jycomment{does this look fine?}

\begin{restatable} [Finite-Sample Cosine Convergence] {lem}{finitesamplecosine}
\label{lemma:finite_cosine}
Assume $\|\tilde{\bbeta}_0\| \ge \|\bm{\beta}^*\| / 10$. Take $\epsilon_f > 0$ small enough to ensure $\kappa = (1 - 10\epsilon_f) \sqrt{1+\frac{\eta^2}{\frac{2}{3} + \eta^2}} > 1$. We run the sample-splitting finite-sample EM, each step with $n/T = \tilde{O} \big(\max(1, \eta^{-2}) d/\epsilon_f^2 \big)$ samples and $T = O\big(\max(1, \eta^{-2}) \log d \big)$ iterations.  As long as $\tilde{\theta}_t > \pi/3$ for all $t \le T$, we have with high probability
\begin{equation}
	\cos \tilde{\theta}_T \ge \kappa^{T} \cos \tilde{\theta}_0 - \frac{\kappa^{T} - 1}{\kappa - 1} O \Big( \frac{\epsilon_f}{\sqrt{d}} \Big).
\end{equation}
In particular, when $\cos \tilde{\theta}_0 = \Theta \big( 1/\sqrt{d} \big)$, we get $\cos \tilde{\theta}_T \ge \frac{1}{2} - O(\epsilon)$. \yccomment{For this lemma, we are assuming that the above two ``sufficient conditions'' hold?} \jycomment{For $\epsilon_f < 1/\sqrt{d}$, we claimed it in the previous paragraph. And what is the other one?} \wqcomment{The other condition is when $\epsilon_f<\langle \bbeta,\betastar\rangle$, mentioned in the previous paragraph too.}
\end{restatable}

\begin{restatable} [Finite-Sample Sine Convergence] {lem}{finitesamplesine}
\label{lemma:finite_sine}
Suppose we get a $\tilde{\bbeta}_0$ whose angle formed with $\bm{\beta}^*$ is less than $\pi/3$ from the previous phase. We run the sample-splitting sample-based EM, each step with $n/T = \tilde{O}(\max(1, \eta^{-2}) d/\epsilon_f^2)$ samples. Then with high probability and a constant $\kappa = \Big( \sqrt{1+\frac{0.5 \eta^2}{1+\eta^2}} \Big)^{-1} < 1$, we have
\begin{equation}
	\sin^2 \tilde{\theta}_T \le \kappa^{2T} \sin^2 \tilde{\theta}_0 + \frac{1}{1-\kappa^2} O(\epsilon_f). 
\end{equation}
After $T = O(\max(1, \eta^{-2}))$ iterations, we get $\sin^2 \tilde{\theta}_T \le \sin^2 \frac{\pi}{70} + O(\epsilon)$.
\end{restatable}

\begin{remark*}
We should take small $\epsilon_f$ such that $\tilde{\theta}'$ remains less than $\tilde{\theta}$ in each iteration with high probability. 
A sufficient condition is that $\epsilon_f < O(\min(1, \eta^2))$, which ensures that
    \begin{equation*}
    	\left( 1+\frac{0.5 \eta^2}{1+\eta^2} \right)^{-1} \sin^2 \tilde{\theta} + O(\epsilon_f) \le \sin^2 \tilde{\theta},
    \end{equation*}

\end{remark*}

Finally, suppose we have reached the angle below $\pi / 70$. The following theorem, proved in Appendix~\ref{sec:proof_of_statserror}, provides a convergence guarantee in $\ell_2$ distance for sample based EM.


\begin{restatable}[Finite-Sample Distance Convergence]{theorem}{finitestatserror}
\label{thm:final_stats_error}
	Suppose we get a $\tilde{\bbeta}_0$ whose angle with $\bm{\beta}^*$ is less than $\frac{\pi}{70}$ from the previous phase. There exist a constant $C>1$ for which the following holds.
	\begin{itemize}
	    \item If $ \eta < C$, sample-splitting finite-sample EM with $n/T = \tilde{O}(\eta^{-2}d/\epsilon_f^2)$ samples per iteration satisfies
	\begin{equation}
    	\label{eq:sample_distance}
    	\|\tilde{\bbeta}_T - {\bbeta}^*\| \le \kappa^{T} \|\tilde{\bbeta}_0 - \bm{\beta}^*\| + T \kappa^{T} \|\bm{\beta}^*\| \frac{\eta^2}{1+\eta^2} + O(\epsilon) \|\bm{\beta}^*\|,
    \end{equation}
    where $\kappa$ is the maximum among (\ref{eq:kappa_candidates}) as in Corollary \ref{corollary:distance}.
    After $T = O(\eta^{-2} \log(1/\epsilon))$ iterations, we estimate  $\bbeta^*$ with an $\ell_2$ error bounded by $O(\epsilon)$.
    
    \item If $ \eta \ge C$, sample-splitting finite-sample EM with $n/T=\tilde{O}(d/\epsilon_f^2)$ samples per iteration satisfies
     \begin{equation}
      \label{eq:bound_highsnr}
          \|\tilde{\bbeta}_T-{\bbeta}^*\|\leq \kappa^{T}\|\tilde{\bbeta }_0-\bm{\beta}^*\|+O(\epsilon)\sigma,  
    \end{equation}
        where $\kappa = 0.95+\epsilon_f<1$. After $T=O(\log (1/(\sigma\epsilon) ))$ iterations, we estimate $\bbeta^*$ with an $\ell_2$ error bounded by $\sigma O(\epsilon)$.
	\end{itemize}

\end{restatable}


Note that the results for the low and high SNR cases are different and they actually require different proof techniques. For the low SNR regime, the bound \eqref{eq:sample_distance} is obtained by coupling the population and the finite-sample EM updates as mentioned before. Since the statistical fluctuation between these two updates scales with $ \|\bm{\beta}^*\| + \sigma$, the final estimation error depends on $\|\bm{\beta}^*\|$. 
For the high SNR regime, we in fact take a different approach and directly control $\tilde{\bm{\beta}}'-\bm{\beta}^*$ by using the sample covariance matrix $\frac{1}{n}\sum_{i=1}^{n}\bm{x}_i\bm{x}_i^{\top}$ to our advantage. The resulting bound \eqref{eq:bound_highsnr} scales with $\sigma$ only and guarantees exact recovery when $\sigma =0$. 

\subsection{Analysis of Easy-EM}
\label{ssec:easyem}

In the results above we have assumed that the effect of the term $\epsilon_f^2$ is negligible in equation (\ref{ineq:cosine_sample}) . We believe this term is simply an artifact of our analysis. This motivates us to consider the Easy-EM update in~\eqref{eq:EEM-CF}, for which we can eliminate this $\epsilon_f^2$ factor. We do so by proving a better concentration bound $|\langle \tilde{\bm{\beta}}'' - \bm{\beta}', \bm{\beta}^* \rangle|= O \big(1/\sqrt{n} \big)$ for Easy-EM using the fact that Easy-EM does not have the inverse sample covariance matrix, where we recall that $ \tilde{\bm{\beta}}''$ denotes its next iterate.  This bound allows us to establish the following theorem, which is a counterpart of Theorem~\ref{theorem:finite_cosine_update} for EM.



\begin{restatable}{theorem}{easyemupdate}
\label{thm:easyem_update}
    Suppose that the norm of the current estimator $\|\bm{\beta}\|$ is larger than $\|\bm{\beta}^*\| / 10$. Then, with $n= \tilde{O}(\max(1,\eta^{-2})d/\epsilon_f^2)$ samples for one Easy-EM iteration, we have
    \begin{subequations}
    \begin{align}
    \label{eq:easyem_cosine}
	    &\cos \tilde{\theta}'' \ge \kappa(1 - 10\epsilon_f) \cos \theta - O \left(\frac{\epsilon_f}{\sqrt{d}} \right), \\ 
	    &\sin^2 \tilde{\theta}'' \le \kappa' \sin^2 \theta + O(\epsilon_f),
    \end{align}
    \end{subequations}
    with $\kappa = \sqrt{1 + \frac{\sin^2 \theta}{\cos^2 \theta +  \frac{1}{2}(1+\eta^{-2}) }} \ge 1$, and $\kappa' = \left(1 + \frac{2\eta^2}{1+\eta^2} \cos^2\theta \right)^{-1} < 1$.
\end{restatable}

The only difference between Theorems~\ref{theorem:finite_cosine_update} and \ref{thm:easyem_update} is that equation (\ref{eq:easyem_cosine}) does not has an extra factor $\epsilon_f^2$. Thus, Easy-EM improves the angle in each step even without the assumption $\epsilon_f \le 1/\sqrt{d}$, and therefore the multi-step bounds in Lemmas \ref{lemma:finite_cosine}, \ref{lemma:finite_sine} and (\ref{eq:sample_distance}) can be identically applied to Easy-EM. 

\subsection{Discussions}

The overall sample complexity to achieve $\epsilon$ error is $n = \tilde{O} \big(\max(1, poly(\eta^{-1})) (d/\epsilon^2)\big)$. In the high SNR regime, this is $\tilde{O} (d/\epsilon^2 )$, which is the minimax sample complexity up to log factors. Moreover, the final statistical error is $\tilde{O}\big(\sigma \sqrt{d/n}\big)$ which guarantees exact recovery as $\sigma \rightarrow 0$. In the low SNR regime, the sample complexity becomes $O(\eta^{-6} d/\epsilon^2)$. This high dependency on SNR arises because the convergence rates of sine and distance are $1-O(\eta^2)$ in the low SNR regime, and  the statistical fluctuation has to be smaller than $\eta^2$ in order to guarantee that every iteration improves the angle or distance.
It seems to be the nature of EM algorithm as we have seen similarly high dependence on SNR in GMM settings \citep{daskalakis2017ten}. Nevertheless, once enough number of samples are given to offset a low SNR, we achieve an statistical error of $\tilde{O}\big( \|\betastar\| \sqrt{d/n}\big)$. 
\yccomment{Optimal in what and why is it optimal (as compared to the minimax lower bound)?} \jycomment{Would this description fine?}

\section{Conclusion}
We studied the EM algorithm for a mixture of two linear regression models. In the large sample limit, we showed that EM converges to true parameters globally without any specialized initialization.  
In finite sample case, we showed that EM enjoys the same convergences behavior, though it may need the aid of Easy-EM in the first few steps.
It would be interesting to explore whether we can remove the dependency on Easy-EM. Extensions of this work could be analyzing the performance of EM when the weight of each component is not equal or there are more than two components.

\subsection*{Acknowledgement}
J.\ Kwon and C.\ Caramanis are partially supported by NSF EECS-1609279, CCF-1302435, and CNS-1704778.
W.\ Qian and Y.\ Chen are partially supported by NSF grants 1657420 and 1704828.

\newpage

\bibliographystyle{unsrt}
\bibliography{main}

\begin{thebibliography}{20}
\providecommand{\natexlab}[1]{#1}
\providecommand{\url}[1]{\texttt{#1}}
\expandafter\ifx\csname urlstyle\endcsname\relax
  \providecommand{\doi}[1]{doi: #1}\else
  \providecommand{\doi}{doi: \begingroup \urlstyle{rm}\Url}\fi

\bibitem[Balakrishnan et~al.(2017)Balakrishnan, Wainwright, and
  Yu]{balakrishnan_statistical_2017}
Sivaraman Balakrishnan, Martin~J. Wainwright, and Bin Yu.
\newblock Statistical guarantees for the {EM} algorithm: {From} population to
  sample-based analysis.
\newblock \emph{The Annals of Statistics}, 45\penalty0 (1):\penalty0 77--120,
  February 2017.
\newblock ISSN 0090-5364.
\newblock \doi{10.1214/16-AOS1435}.
\newblock URL \url{http://projecteuclid.org/euclid.aos/1487667618}.

\bibitem[Chaganty and Liang(2013)]{chaganty2013spectral}
Arun~Tejasvi Chaganty and Percy Liang.
\newblock Spectral experts for estimating mixtures of linear regressions.
\newblock In \emph{International Conference on Machine Learning}, pages
  1040--1048, 2013.

\bibitem[Chen et~al.(2014)Chen, Yi, and Caramanis]{chen2014convex}
Yudong Chen, Xinyang Yi, and Constantine Caramanis.
\newblock A convex formulation for mixed regression with two components:
  Minimax optimal rates.
\newblock In \emph{Conference on Learning Theory}, pages 560--604, 2014.

\bibitem[Daskalakis et~al.(2017)Daskalakis, Tzamos, and
  Zampetakis]{daskalakis2017ten}
Constantinos Daskalakis, Christos Tzamos, and Manolis Zampetakis.
\newblock Ten steps of em suffice for mixtures of two gaussians.
\newblock In \emph{Conference on Learning Theory}, pages 704--710, 2017.

\bibitem[De~Veaux(1989)]{de1989mixtures}
Richard~D. De~Veaux.
\newblock Mixtures of linear regressions.
\newblock \emph{Computational Statistics \& Data Analysis}, 8\penalty0
  (3):\penalty0 227--245, 1989.

\bibitem[Dempster et~al.(1977)Dempster, Laird, and Rubin]{dempster1977maximum}
Arthur~P. Dempster, Nan~M. Laird, and Donald~B. Rubin.
\newblock Maximum likelihood from incomplete data via the em algorithm.
\newblock \emph{Journal of the royal statistical society. Series B
  (methodological)}, pages 1--38, 1977.

\bibitem[Jin et~al.(2016)Jin, Zhang, Balakrishnan, Wainwright, and
  Jordan]{jin2016local}
Chi Jin, Yuchen Zhang, Sivaraman Balakrishnan, Martin~J Wainwright, and
  Michael~I Jordan.
\newblock Local maxima in the likelihood of gaussian mixture models: Structural
  results and algorithmic consequences.
\newblock In \emph{Advances in Neural Information Processing Systems}, pages
  4116--4124, 2016.

\bibitem[Jordan and Jacobs(1994)]{jordan1994hierarchical}
Michael~I Jordan and Robert~A Jacobs.
\newblock Hierarchical mixtures of experts and the em algorithm.
\newblock \emph{Neural computation}, 6\penalty0 (2):\penalty0 181--214, 1994.

\bibitem[Klusowski et~al.(2019)Klusowski, Yang, and
  Brinda]{klusowski2019estimating}
Jason~M Klusowski, Dana Yang, and WD~Brinda.
\newblock Estimating the coefficients of a mixture of two linear regressions by
  expectation maximization.
\newblock \emph{IEEE Transactions on Information Theory}, 2019.

\bibitem[Li and Liang(2018)]{li2018learning}
Yuanzhi Li and Yingyu Liang.
\newblock Learning mixtures of linear regressions with nearly optimal
  complexity.
\newblock In \emph{Conference On Learning Theory}, pages 1125--1144, 2018.

\bibitem[Sedghi et~al.(2016)Sedghi, Janzamin, and
  Anandkumar]{sedghi2016provable}
Hanie Sedghi, Majid Janzamin, and Anima Anandkumar.
\newblock Provable tensor methods for learning mixtures of generalized linear
  models.
\newblock In \emph{Artificial Intelligence and Statistics}, pages 1223--1231,
  2016.

\bibitem[Vershynin(2010)]{vershynin2010introduction}
Roman Vershynin.
\newblock Introduction to the non-asymptotic analysis of random matrices.
\newblock \emph{arXiv preprint arXiv:1011.3027}, 2010.

\bibitem[Wainwright(2015)]{wainwright2015high}
Martin~J. Wainwright.
\newblock High-dimensional statistics: A non-asymptotic viewpoint.
\newblock \emph{preparation. University of California, Berkeley}, 2015.

\bibitem[Wu(1983)]{wu1983convergence}
CF~Jeff Wu.
\newblock On the convergence properties of the em algorithm.
\newblock \emph{The Annals of statistics}, pages 95--103, 1983.

\bibitem[Wu et~al.(2016)Wu, Yang, Zhao, and Zhu]{wu2016convergence}
Chong Wu, Can Yang, Hongyu Zhao, and Ji~Zhu.
\newblock On the convergence of the em algorithm: A data-adaptive analysis.
\newblock \emph{arXiv preprint arXiv:1611.00519}, 2016.

\bibitem[Xu et~al.(2016)Xu, Hsu, and Maleki]{xu2016global}
Ji~Xu, Daniel~J. Hsu, and Arian Maleki.
\newblock Global analysis of expectation maximization for mixtures of two
  gaussians.
\newblock In \emph{Advances in Neural Information Processing Systems}, pages
  2676--2684, 2016.

\bibitem[Yi and Caramanis(2015)]{yi2015regularized}
Xinyang Yi and Constantine Caramanis.
\newblock Regularized em algorithms: A unified framework and statistical
  guarantees.
\newblock In \emph{Proc. Advances in Neural Information Processing Systems},
  pages 1567--1575, 2015.

\bibitem[Yi et~al.(2014)Yi, Caramanis, and Sanghavi]{yi2014alternating}
Xinyang Yi, Constantine Caramanis, and Sujay Sanghavi.
\newblock Alternating minimization for mixed linear regression.
\newblock In \emph{International Conference on Machine Learning}, pages
  613--621, 2014.

\bibitem[Yi et~al.(2016)Yi, Caramanis, and Sanghavi]{yi2016solving}
Xinyang Yi, Constantine Caramanis, and Sujay Sanghavi.
\newblock Solving a mixture of many random linear equations by tensor
  decomposition and alternating minimization.
\newblock \emph{arXiv preprint arXiv:1608.05749}, 2016.

\bibitem[Zhong et~al.(2016)Zhong, Jain, and Dhillon]{zhong2016mixed}
Kai Zhong, Prateek Jain, and Inderjit~S. Dhillon.
\newblock Mixed linear regression with multiple components.
\newblock In \emph{Advances in neural information processing systems}, pages
  2190--2198, 2016.

\end{thebibliography}

\newpage

\begin{appendices}
\section{Proofs for Population EM Update}

\newcommand{\fracbsig}{\frac{{b_2^*}}{\sigma}}

\subsection{Proof of Lemma \ref{lem:em_update_summary}}
\label{appendix:lemma_1}

\simplifiedupdate*

\begin{align}
\label{eq:s_r_expression}
	S := & \mathbb{E}_{\substack{\alpha_1 \sim \mathcal{N}(0,1), \\ y \sim \mathcal{N}(0, \sigma_2^2)}}
    \left[ \tanhcustom + \frac{\alpha_1 b_1}{\sigma^2} (y+\alpha_1 b_1^*) \tanhpcustom \right], \nonumber \\
    R := & (\sigma^2 + \|\bm{\beta}^*\|^2) \mathbb{E}_{\substack{\alpha_1 \sim \mathcal{N}(0,1), \\ y \sim \mathcal{N}(0, \sigma_2^2)}} 
    \left[ \frac{\alpha_1^2 b_1}{\sigma^2} \tanhpcustom \right].
\end{align}

For completeness of the proof, we repeat some arguments in the main text. Recall the EM update:
\[
	\bm{\beta}' = \mathbb{E}_{X \sim \mathcal{N}(0, I)} \left [ \left( \mathbb{E}_{y|X \sim \mathcal{N}(\langle X, \bm{\beta}^*\rangle, \sigma^2)} \left[ \tanh \left( \frac{\langle X, \bm{\beta}\rangle}{\sigma^2} y \right)y \right] \right)X \right].
\]

We first change the basis by choosing $\bm{v}_1 = \bm{\beta}/\|\bm{\beta}\|$, the unit vector in the direction of the current estimator, and $\bm{v}_2$ to be the orthogonal complement of $\bm{v}_1$ in ${\rm span}\{\bm{\beta},\bm{\beta}^{\ast}\}$. We let $\bm{v}_3, ..., \bm{v}_d$ be a completion to an orthonormal basis for the full parameter space, $\mathbb{R}^d$, along with $\bm{v}_1$ and $\bm{v}_2$. By the spherical symmetry of the distribution of $\bm{x}$,  we have
\begin{equation}
	\bm{\beta}' = \mathbb{E}_{\alpha_i} \left[ \mathbb{E}_{y|\alpha_i} \left[ \tanh \left(\frac{b_1 \alpha_1}{\sigma^2} y \right)y \right] \sum_i \alpha_i \bm{v}_i \right],
\end{equation}
where the expectation is taken over $\alpha_i \sim \mathcal{N}(0,1)$, and $y|\alpha_i \sim \mathcal{N}(\alpha_1 b_1^* + \alpha_2 b_2^*, \sigma^2)$, and we defined $b_1 = \langle \bm{\beta}, \bm{v}_1 \rangle = \|\bm{\beta}\|$, $b_1^* = \langle \bm{\beta}^*, \bm{v}_1 \rangle$, and $b_2^* = \langle \bm{\beta}^*, \bm{v}_2 \rangle$. Without loss of generality, we assume $b_1, b_1^*, b_2^* \ge 0$.
The inner expectation over $y$ does not have any dependence on $\alpha_i$ for $i \ge 3$. Therefore, taking expectation over $\alpha_i$ for $i\ge 3$ yields 0, which implies $\bm{\beta}'$ is also on the plane spanned by $\bm{v}_1, \bm{v}_2$. It enables us to rewrite it as $\bm{\beta}' = b_1' \bm{v}_1 + b_2' \bm{v}_2$ where 
\begin{subequations}
\begin{align}
	b_1' = \mathbb{E}_{\alpha_1, \alpha_2} \left[ \mathbb{E}_{y|\alpha_1, \alpha_2} \left[ \tanh \left(\frac{b_1 \alpha_1}{\sigma^2} y \right)y \right] \alpha_1 \right], \label{eq:b_1_prime} \\
    b_2' = \mathbb{E}_{\alpha_1, \alpha_2} \left[ \mathbb{E}_{y|\alpha_1, \alpha_2} \left[ \tanh \left(\frac{b_1 \alpha_1}{\sigma^2} y \right)y \right] \alpha_2 \right], \label{eq:b_2_prime}
\end{align}
\end{subequations}

where the expectation is similarly taken over $\alpha_i \sim \mathcal{N}(0,1)$, and $y|\alpha_i \sim \mathcal{N}(\alpha_1 b_1^* + \alpha_2 b_2^*, \sigma^2)$. In the following, we prove that $b_1'$ and $b_2'$ have a simplified representation as in Lemma \ref{lem:em_update_summary}.

\begin{proof}
    We start with second coordinate $b_2'$. We will occasionally omit variables for expectation when it is clear over which variable the expectation is taken. We can rewrite the equation (\ref{eq:b_2_prime}) as
    \begin{align*}
        b_2' = \mathbb{E}[g(\alpha_1, \alpha_2) \alpha_2],
    \end{align*}
    where $g(\alpha_1, \alpha_2) = \mathbb{E}_{y \sim \mathcal{N}(0,\sigma^2))}[ \tanh(\frac{b_1\alpha_1}{\sigma^2}(y+\alpha_1b_1^* + \alpha_2b_2^*)) (y+\alpha_1b_1^* + \alpha_2b_2^*)]$. Apply Stein's lemma with respect to $\alpha_2$ yields
    \begin{align*}
        b_2' = & \mathbb{E}[g(\alpha_1, \alpha_2) \alpha_2] = \mathbb{E} \left[ \frac{\partial}{\partial \alpha_2} g(\alpha_1, \alpha_2) \right], \\
        \frac{\partial}{\partial \alpha_2} g(\alpha_1, \alpha_2) = & b_2^* \mathbb{E}_{y \sim \mathcal{N}(0,\sigma^2))} \bigg[ \tanh \left( \frac{\alpha_1 b_1}{\sigma^2} (y+\alpha_1b_1^* + \alpha_2 b_2^*) \right)  \\ 
        &  \hspace{1em}+ \frac{\alpha_1 b_1}{\sigma^2}(y+\alpha_1 b_1^* + \alpha_2 b_2^*)  \tanh' \left(\frac{\alpha_1 b_1}{\sigma^2} (y+\alpha_1b_1^* + \alpha_2 b_2^*) \right) \bigg] \\
         \stackrel{(a)}{=} & b_2^* \mathbb{E}_{y \sim \mathcal{N}(0,\sigma_2^2))} \left[\tanhcustom + \frac{\alpha_1 b_1}{\sigma^2}(y+\alpha_1 b_1^*)  \tanhpcustom \right]. \\
        \therefore b_2' = & b_2^* S,
    \end{align*}
    where in (a), we replaced $y + \alpha_2 b_2^*$ with a new Gaussian variable as they are the sum of two Gaussian variables.
    
    For the first coordinate $b_1'$, we take the similar strategy but we arrange it in a different way. First, we rewrite equation (\ref{eq:b_1_prime}) as
    \begin{align}
        b_1' = \mathbb{E}_{\alpha_1 \sim \mathcal{N}(0, 1)} \left[\mathbb{E}_{y \sim \mathcal{N}(0,\sigma_2^2))} \left[ \tanh \left( \frac{b_1\alpha_1}{\sigma^2}(y+\alpha_1b_1^*) \right) (y+\alpha_1b_1^*) \right] \alpha_1 \right], \label{single_use:b_1_prime}
    \end{align}
    where we again replaced $y+\alpha_2b_2^*$ with one Gaussian variable. Then observe that another application of Stein's lemma yields
    \begin{align}
        &\mathbb{E} \left[ \tanhcustom \alpha_1^2 \right] \nonumber \\
        = &\mathbb{E} \left[ \tanhcustom + \left( \frac{2 b_1^* b_1}{\sigma^2} \alpha_1 + \frac{b_1}{\sigma^2}y \right) \alpha_1 \tanhpcustom \right] \nonumber \\
        = &\mathbb{E} \left[ \tanhcustom + \frac{\alpha_1 b_1}{\sigma^2}(y+\alpha_1 b_1^*)  \tanhpcustom \right] \nonumber \\
        & \qquad + b_1^* \mathbb{E} \left[ \frac{\alpha_1^2 b_1}{\sigma^2} \tanhpcustom \right]. \label{single_use:S_lower}
    \end{align}
    
    On the other hand,
    \begin{align*}
        \mathbb{E}_{ \begin{smallmatrix} \alpha_1 \sim \mathcal{N}(0,1) \\ y \sim \mathcal{N}(0,\sigma_2^2)) \end{smallmatrix} } \left[ \tanhcustom \alpha_1 y \right] &= 
        \sigma_2^2 \mathbb{E} \left[ \frac{\alpha_1^2 b_1}{\sigma^2} \tanhpcustom \right],
    \end{align*}
    where we applied Stein's lemma for $y$ this time. Plugging the above two equations into (\ref{single_use:b_1_prime}), we get  
    \begin{align*}
        b_1' = b_1^*S + R.
    \end{align*}
Finally, $R > 0$ since it is the expectation of positive values almost surely over the real line. Lemma \ref{lemma:S_bounds} in the appendix shows that $S \ge 0$. $S=0$ iff $b_1=0$ or $b_1^*=0$. Since we only consider the case where $b_1\neq 0$, the proof is complete. 
\end{proof}





\subsection{Proof of Theorem \ref{thm:stationary_points}}

\label{subsec:stationary_points}
\stationarypoints*

\begin{proof}
From the correspondence between the EM update and gradient step as in Lemma \ref{lem:em_gd_mlr}, it is easy to see that $\bm{\beta}'=\bm{\beta}$ iff the current gradient is $\bm{0}$. Therefore, the set of fixed points of population EM is equal to the set of stationary points of the log-likelihood.\\
To characterize the fixed points, there are two key components as follows. First of all, we show in Lemma \ref{lem:em_fixedpoints} that every 2 dimensional space that includes $\bm{\beta}^*$ contains exactly $5$ fixed points, $\bm{0}$, $\bm{\beta}^*$, $-\bm{\beta}^*$ and two other symmetric points in the orthogonal direction to $\bm{\beta}^*$. Secondly, we need to some understanding of the property about the Hessian of those fixed points. For $\bm{\beta}^*$ and $-\bm{\beta}^*$, they are the global maxima as they are the optimal parameters. For $\bm{0}$, a simple calculation shows that the Hessian is positive definite, thus, it is a local minima. For any fixed point $\bm{v}$ in the orthogonal direction to $\betastar$, we use Stein's lemma to show that in the direction of $\bm{\beta}^*$, $\langle \bm{\beta}^*, \mathcal{H}\bm{\beta}^*\rangle$ is strictly positive, where $\mathcal{H}$ is the Hessian matrix (Proposition \ref{prop:pos_eig}). On the other hand, we show that any fixed point $\bm{v}$ is a local maxima in span$(\bm{v})$. These two facts allow us to conclude that the fixed points in the orthogonal space are indeed saddle points. To illustrate the second point, we utilize two observations: (1) in the first part of the proof, we have demonstrated span $(\bm{v})$ is an invariant subspace for the EM operator and $\bm{v}$ is a contracting point; (2) the monotonicity property of the EM algorithm says that the log-likelihood of the EM iterate does not decrease. Therefore, any fixed point in the direction orthogonal to $\bm{\beta}^*$ is a local maxima in span$(\bm{v})$.  
\end{proof} 

\begin{lem}
\label{lem:em_fixedpoints}
For each unit vector $\bm{v}$ satisfying $\bm{v}\perp\bm{\beta}^*$,
the population EM starting at $\bm{\beta} \in \text{span}(\bm{\beta}^*,\bm{v})$ has exactly five fixed points: $0$, $\bm{\beta}^*$,
$-\bm{\beta}^*$, $E(\bm{v})\bm{v}$, and $-E(\bm{v})\bm{v}$ for some $E(\bm{v})>0$. 
\end{lem}

\begin{proof}
Let $\bm{\beta}'$ be the EM update as in the standard notation.
When $\bm{\beta}=0$, we have $\bm{\beta}'=0$ and thus $\bm{\beta}=0$
is a fixed point. For the other cases, we will use a few facts
established in Lemma \ref{lem:em_update_summary}:

\begin{itemize}
\item if $\langle \bm{\beta},\bm{\beta}^*\rangle=0$ (i.e, $b_{1}^{*}=0$), it follows that $S=0$ and $b_{2}'=0$. In other words, if the current iterate $\bm{\beta}$ is orthogonal to the
ground truth $\bm{\beta}_{*}$, the population EM update remains orthogonal
to $\bm{\beta}^*$. 
\item if $\langle\bm{\beta}^{\perp},\bm{\beta}^*\rangle=0$ (i.e, $b_{2}^{*}=0$), it follows that
$b_{2}'=0$. In other words, if the current iterate $\bm{\beta}$ is in the direction of $\bm{\beta}^*$
(or $-\beta_{*}$), the population EM iterate remains in that direction.
\item if $\langle\bm{\beta},\bm{\beta}^*\rangle>0$ (or $\langle\bm{\beta},-\bm{\beta}^*\rangle>0$),
it follows that $b_{2}^{*}>0, S>0$, and thus $b_{2}'>0$. In other words, if
the current iterate has an acute angle with $\bm{\beta}^*$ (or $-\bm{\beta}^*$),
$\angle(\bm{\beta}',\bm{\beta}_*)$ (or $\angle(\bm{\beta}',\bm{\beta}_*$)
will strictly decrease and no fixed point can exist in this region. 
\end{itemize}
Therefore, we deduce that the fixed points of the population EM lies
either in span($\bm{\beta}^*$) or in the subspace orthogonal to $\bm{\beta}^*$.
They are the invariant subspaces of the population EM operator. In Lemma \ref{lem:dynamic_along}, it is shown that for each unit direction of $\bm{\beta}$, there exists
a unique contraction point. We thus conclude that if $\bm{\beta}$ is in the direction of $\bm{\beta}^*(-\bm{\beta}^*)$, the fixed point is $\bm{\beta}^*(-\bm{\beta}^*)$; if $\bm{\beta}$ is in the direction of $\bm{v}(-\bm{v})$, the fixed point is $E(\bm{v})(-E(\bm{v}))$ for some $E(\bm{v})>0$.
\end{proof}

\begin{lem}
\label{lem:dynamic_along}
Suppose $\langle \bm{\beta},\bm{\beta}^*\rangle\geq 0$. Let $\bm{v}_1$ be the unit vector of $\bm{\beta}$ and $b_1'$ be the notation used in Lemma \ref{lem:em_update_summary}, denoting the the projection of the EM update onto span($\bm{v}_1$). There exists a unique non-zero number $E(\bm{v}_1) $ satisfying
\begin{align*}
\begin{cases}
\|\bm{\beta} \| < b_1' < E(\bm{v}_1) & \text{if} \; \|\bm{\beta}\| < E(\bm{v}_1),\\
E(\bm{v}_1) < b_1'<\|\bm{\beta} \| & \text{if} \; \|\bm{\beta}\| > E(\bm{v}_1),\\
b_1' = E(\bm{v}_1) & \text{if} \; \|\bm{\beta}\| = E(\bm{v}_1).
\end{cases}
\end{align*}
\end{lem}
\begin{proof}
When $\bm{v}_1$ is fixed, $b_1'$ only depends on $\|\bm{\beta}\|$. We thus use $f(\|\bm{\beta}\|)$ for $b_1'$ in the following to emphasize it is a function of $\|\bm{\beta}\|$. 
\begin{align*}
f(\|\bm{\beta}\|) := &\mathbb{E}_{X,y} \left[y\langle X, \bm{v}_1\rangle \tanh \left(\frac{\|\bm{\beta}\| y\langle X, \bm{v}_1 \rangle}{\sigma^2}\right) \right] \\
= & \mathbb{E}_{X\sim \mathcal{N}(0,I),y\sim \mathcal{N}(\langle \bm{\beta}^*, X\rangle,\sigma^2)} \left[y\langle X, \bm{v}_1\rangle \tanh \left(\frac{\|\bm{\beta}\| y\langle X, \bm{v}_1 \rangle}{\sigma^2} \right) \right].
\end{align*}
Let us check a few properties of $f$:
\begin{enumerate}
\item $f$ is smooth (obvious).
\item $f$ is strictly increasing and concave. Note that its derivative with respect to $\|\bm{\beta}\|$ 
\begin{align*}
f'(\|\bm{\beta}\|) = \mathbb{E}_{X\sim \mathcal{N}(0,I),y\sim \mathcal{N}(\langle \bm{\beta}^*, X\rangle,\sigma^2)} \left[\frac{(y\langle X, \bm{v}_1\rangle)^2}{\sigma^2} \tanh' \left(\frac{\|\bm{\beta}\| y\langle X, \bm{v}_1 \rangle}{\sigma^2} \right) \right]
\end{align*}
is positive and is strictly decreasing with respect to $\|\bm{\beta}\|$.
\item $f(0)=0$ and $f'(0)>1$ since
\begin{align*}
f'(0)  
= &  \mathbb{E}_{X\sim \mathcal{N}(0,I),y\sim \mathcal{N}(\langle \bm{\beta}^*, X\rangle,\sigma^2)} \left[ \frac{(y\langle X, \bm{v}_1 \rangle)^2}{\sigma^2} \right]\\
= & \frac{\sigma^2 + \|\bm{\beta}^*\|^2(3\cos^2(\angle(\bm{\beta},\bm{\beta}_*))+\sin^2(\angle(\bm{\beta},\bm{\beta}_*)))}{\sigma^2}.
\end{align*}
\item $f$ is bounded (cf. Lemma \ref{lemma:bprime_bounded})
\end{enumerate}
Let $g(\|\bm{\beta}\|) : =f(\|\bm{\beta}\|)-\|\bm{\beta}\|$, it is a strictly concave and smooth function from Property 2. Moreover, $g(0)=0$, $g'(0)>0$ from Property 3 and $\lim_{\|\bm{\beta}\| \to \infty} g(\|\bm{\beta}\|)=  -\infty$ from Property 4. Lemma \ref{lem:contractfunc} shows that there exists a unique $E(\bm{v}_1)>0$ for $g$ such that $g(E(\bm{v}_1))=0$. Moreover when $\|\bm{\beta}\|<E(\bm{v}_1)$, $g(\|\bm{\beta}\|)>0$  and when $\|\bm{\beta}\|>E(\bm{v}_1)$, $g(\|\bm{\beta}\|)<0$. Equivalently, it means that 
\[
\begin{cases}
\|\bm{\beta}\|< f(\|\bm{\beta}\|) < E(\bm{v}_1) & \text{if } 0< \|\bm{\beta}\| < E(\bm{v}_1), \\
\|\bm{\beta}\|> f(\|\bm{\beta}\|) > E(\bm{v}_1) & \text{if } \|\bm{\beta}\| > E(\bm{v}_1), \\
f(\|\bm{\beta}\|) = E(\bm{v}_1) & \text{if } \|\bm{\beta}\| = E(\bm{v}_1).
\end{cases}
\]
\end{proof}

\begin{lem}
	\label{lem:contractfunc}
	Let $f:\mathbb{R}^{+} \to \mathbb{R}$ be a smooth and concave function, with strictly decreasing derivative, satisfying $f(0)=0, f'(0)>0$, and $\lim_{x\to \infty} f(x)=-\infty$. Then there exists a unique $t>0$ such that $f(t)=0$ and $f'(t)<0$. Moreover, $f(x)>0$ if $x\in (0,t)$ and $f(x)<0$ if $x\in (t,\infty)$.  
\end{lem}
\begin{proof}
	Since $f$ has a continuous gradient at $0$ with $f'(0)>0$, there exists $t_1>0$ such that $f'(x)>0$ for all $x\leq t_1$. We thus conclude that
	\begin{align*}
	f(x) >0 \; \forall x\in (0, t_1] 
	\end{align*}
	by the Fundamental theorem of Calculus. By the continuity of $f$ and the condition that $\lim_{x\to \infty} f(x) = -\infty$, there exists $t_2>0$ such that $f(t_2)<0$. Rolle's theorem tells us that there exists $t\in (t_1,t_2)$ such that $f(t)=0$. Since $f(0)=0$, the mean value theorem tells us that there exists $t_3 \in (0,t)$ such that $f'(t_3)=0$. By assumption $f$ is strictly decreasing derivative, we have $f'(x) \leq 0$ for all $x\geq t_3$ and $f'(x)>0$ for all $x\in (0,t_3)$. It follows that $f(x)$ is increasing on $(0,t_3)$ and it is decreasing on $(t_3,\infty)$.  The statement follows.
\end{proof}

\begin{lem}[Correspondence between EM and GD]
\label{lem:em_gd_mlr}
In the basic set up of the 2MLR problem, let $\mathcal{L}$ denote
the log-likelihood for the population 2MLR 
as follows:
\begin{align*}
\mathcal{L} & =\mathbb{E}_{X,y}\log(\sum_{z\in \left\{-1,1\right\}}f(X,y,z;\bm{\beta})).
\end{align*}

There is a correspondence between the gradient (with respect to $\bm{\beta}$)
of the log-likelihood and the EM operator:
\begin{align*}
\bm{\beta}' & =\bm{\beta}+\sigma^2\nabla_{\bm{\beta}}\mathcal{L}.
\end{align*}
Therefore, the set of fixed points of the population EM iterate $\bm{\beta}'$ is
equal to the set of stationary points of the population log-likelihood
of $\mathcal{L}$
\end{lem}

\begin{proof}
The log-likelihood function $\mathcal{L}$ of the population  MLR with the optimal parameter $ \bm{\beta}^* $  is given by
\begin{align}
\mathcal{L}(\bm{\beta}) 
= & \mathbb{E}_{X,y}\left[\log \left(\frac{1}{2} \cdot \Phi \left(y; \langle X, \bm{\beta}\rangle,\sigma^2 \right) + \frac{1}{2} \cdot \Phi \left(y; -\langle X, \bm{\beta}\rangle,\sigma^2\right)\right) \right] \nonumber \\
=& \mathbb{E}_{X,y} \left[\log \left (\frac{1}{2\sqrt{2\pi\sigma^2}}\exp \left( -\frac{(y-\langle X,\bm{\beta}\rangle)^2}{2\sigma^2} \right) + \frac{1}{2\sqrt{2\pi\sigma^2}} \exp \left(-\frac{(y+\langle X,\bm{\beta}\rangle)^2}{2\sigma^2} \right) \right) \right],  \label{eq:log_pop}
\end{align} 
where $\Phi$ denotes the pdf for the Gaussian distribution. The gradient of the population log-likelihood functions with respect to $\bm{\beta}$ has the following expression:
\begin{align}
& \nabla_{\bm{\beta}} \mathcal{L} = \frac{1}{\sigma^2}\left[-\bm{\beta} + \mathbb{E}_{X,y} \left[ yX\tanh \left(\frac{y\langle X,\bm{\beta}\rangle}{\sigma^2}\right) \right]\right]= \frac{1}{\sigma^2}(-\bm{\beta}+\bm{\beta}'). \label{eq:grad_pop}
\end{align}
where the last line follows from \eqref{eq:em_update}.
\end{proof}

\begin{proposition}[Positive eigenvalue along $\bm{\beta}^*$]
\label{prop:pos_eig}
In the basic set of the MLR problem, the Hessian matrix of the population log-likelihood is
\[
\mathcal{H}=\frac{1}{\sigma^2}\left(-I +  \mathbb{E}_{X,y}\left[\frac{1}{\sigma^2}y^2XX^{T}\tanh' \left(\frac{y\langle X,\bm{\beta}\rangle}{\sigma^2}\right) \right] \right).
\]
Moreover, let $\hat{\bm{\beta}}_*$ be the unit vector in the direction of ${\bm{\beta}}_*$. The following holds for every fixed point $\bm{\beta}$ that is orthogonal to $\bm{\beta}^*$:
	\begin{align*}
	\langle \hat{\bm{\beta}}_*, \mathcal{H} \hat{\bm{\beta}}_* \rangle \geq \frac{1}{\sigma^2}\frac{\|\bm{\beta}^*\|^2}{\sigma^2+\|\bm{\beta}^*\|^2}.
	\end{align*}
\end{proposition}

\begin{proof}
Using the correspondence between the population EM update and gradient of the log-likelihood function of MLR,$\mathcal{L}$ (cf. Lemma \ref{lem:em_gd_mlr}):
\[
\nabla_{\bm{\beta}}\mathcal{L} = \frac{1}{\sigma^2}(\bm{\beta}'- \bm{\beta}).
\]
Hence, the Hessian matrix is:
\[
\mathcal{H} = \frac{1}{\sigma^2}(-I + \nabla_{\bm{\beta}} \bm{\beta}').
\]
Recall the EM update:
\begin{align}
\bm{\beta}'= & \mathbb{E}_{X\sim \mathcal{N}(0,I)}\left[\mathbb{E}_{y|X\sim \mathcal{N}(\langle X,\bm{\beta}^*\rangle,\sigma^2)} \left[yX\tanh \left(\frac{y\langle X,\bm{\beta}\rangle}{\sigma^2}\right)\right] \right]\nonumber \\
= & \sigma \mathbb{E}_{X\sim \mathcal{N}(0,I)}\left[\mathbb{E}_{y|X\sim \mathcal{N}(\langle X,\frac{\bm{\beta}^*}{\sigma}\rangle,1)} \left[yX\tanh \left(y\langle X,\frac{\bm{\beta}}{\sigma}\rangle \right)\right]\right] \; (\text{rescaling}) \label{eq:sigma1_update}.
\end{align}
The gradient with respect to $\bm{\beta}$ is:
\begin{align}
\nabla_{\bm{\beta}'} \bm{\beta}=&\mathbb{E}_{X,y}\left[\frac{1}{\sigma^2}y^2XX^{T}\tanh' \left(\frac{y\langle X,\bm{\beta}\rangle}{\sigma^2}\right) \right] \nonumber \\ 
=&\mathbb{E}_{X\sim \mathcal{N}(0,I)}\left[\mathbb{E}_{y|X\sim \mathcal{N}(\langle X,\frac{\bm{\beta}^*}{\sigma}\rangle,1)} \left[y^2XX^{\top}\tanh'\left(y\langle X,\frac{\bm{\beta}}{\sigma}\rangle\right)\right]\right] \; (\text{rescaling}) \label{eq:sigma1_hessian}.
\end{align}
The first part of the claim is proved. For the second part of the claim, it suffices to prove the case for $\sigma =1$ due to the equivalent representation by rescaling in \eqref{eq:sigma1_update} and \eqref{eq:sigma1_hessian}. If we can show the following relation 
\begin{equation}
\label{eq:sigma1_case}
\langle \hat{\bm{\beta}}_*,\mathcal{H}\hat{\bm{\beta}}_*\rangle \geq  \frac{\|\bm{\beta}_*\|^2}{\|\bm{\beta}_*\|^2+1}
\end{equation}
holds when $\sigma=1$, we can easily conclude that for general $\sigma$,
\[
\langle \hat{\bm{\beta}}_*,\mathcal{H}\hat{\bm{\beta}}_*\rangle \geq  \frac{1}{\sigma^2}\frac{\|\bm{\beta}_*\|^2}{\|\bm{\beta}_*\|^2+\sigma^2}.
\]
In the following, our effort is devote to proving \eqref{eq:sigma1_case} assuming $\sigma=1$. The EM update is now simplified to the following:
\[
\bm{\beta}'= \mathbb{E}_{X\sim \mathcal{N}(0,I)} \left[\mathbb{E}_{y|X\sim \mathcal{N}(\langle X,\bm{\beta}^*\rangle,1)} \left[yX\tanh \left(y\langle X,\bm{\beta}\rangle \right)\right]\right].
\]
As before, we use the following orthonormal basis with $\bm{v}_1=\hat{\bm{\beta}}$ and $\bm{v}_2=\hat{\bm{\beta}^*}$, where $\hat{\bm{\beta}}$ is the unit vector of $\bm{\beta}$ and $\hat{\bm{\beta}^*}$ is the unit vector of $\bm{\beta}^*$. Note that $\bm{v}_1\perp \bm{v}_2$ because we assume $\bm{\beta} \perp \bm{\beta}_*$.  Since $\bm{\beta}$ is a fixed point, it follows that $\bm{\beta} = \mathbb{E}_{X,y} yX\tanh(y\langle X,\bm{\beta}\rangle)$. A necessary condition is:
\begin{equation}
\label{eq:fixed_point_condition}
b_1= \|\bm{\beta}\| = \mathbb{E}_{\alpha_2,\alpha_1,\epsilon}\left(\|\bm{\beta}^*\|\alpha_2+\epsilon\right)\alpha_1\tanh\left(\|\bm{\beta}\|(\|\bm{\beta}^*\|\alpha_2+\epsilon)\alpha_1\right)=b_1',
\end{equation}
where we integrate over $\alpha_1\sim \mathcal{N}(0,1)$, $\alpha_2\sim \mathcal{N}(0,1)$, $\epsilon \sim \mathcal{N}(0,1)$. Using Stein's Lemma for $b_1'$ with respect to $\alpha_1$, we have 
\begin{align}
 & \mathbb{E}_{\alpha_2,\alpha_1,\epsilon}\left[\|\bm{\beta}^*\|\alpha_2+\epsilon)\alpha_1\tanh\left(\|\bm{\beta}\|(\|\bm{\beta}^*\|\alpha_2+\epsilon)\alpha_1\right)\right] \nonumber\\
 = & \mathbb{E}_{\alpha_2,\alpha_1,\epsilon}\nabla_{\alpha_1}\left[(\|\bm{\beta}^*\|\alpha_2+\epsilon)\tanh(\|\bm{\beta}\|(\|\bm{\beta}^*\|\alpha_2+\epsilon)\alpha_1)\right] \nonumber \\
 = & \|\bm{\beta}\|\mathbb{E}_{\alpha_2,\alpha_1,\epsilon}\left[(\|\bm{\beta}^*\|\alpha_2+\epsilon)^2\tanh'(\|\bm{\beta}\|(\|\bm{\beta}^*\|\alpha_2+\epsilon)\alpha_1)\right] \label{eq:first_relation_sub1}.
\end{align}
We obtain a first relation by substituting \eqref{eq:first_relation_sub1} back to \eqref{eq:fixed_point_condition}:
\begin{align}
\label{eq:relation1}
    1 = \mathbb{E}_{\alpha_2,\alpha_1,\epsilon}[(\|\bm{\beta}^*\|\alpha_2+\epsilon)^2\tanh'(\|\bm{\beta}\|(\|\bm{\beta}^*\|\alpha_2+\epsilon)\alpha_1)].
\end{align}
Note that we can write $\|\bm{\beta}^*\|\alpha_2+\epsilon=\sqrt{1+\|\bm{\beta}^*\|^2}Z$ for $Z\sim \mathcal{N}(0,1)$, where the equality holds in the distribution sense. We can apply Stein's Lemma for $b_1'$ again with respect to $Z$:
\begin{align}
& \mathbb{E}_{\alpha_2,\alpha_1,\epsilon}(\|\bm{\beta}^*\|\alpha_2+\epsilon)\alpha_1\tanh(\|\bm{\beta}\|(\|\bm{\beta}^*\|\alpha_2+\epsilon)\alpha_1) \nonumber\\
= & \mathbb{E}_{\alpha_2,Z}\sqrt{1+\|\bm{\beta}^*\|^2}Z \alpha_1\tanh(\|\bm{\beta}\|\sqrt{1+\|\bm{\beta}^*\|^2}Z \alpha_1) \nonumber\\
= & \sqrt{1+\|\bm{\beta}^*\|^2} \mathbb{E}_{\alpha_2,Z} \nabla_{Z}[\alpha_1\tanh(\|\bm{\beta}\|\sqrt{1+\|\bm{\beta}^*\|^2}Z \alpha_1)] \nonumber\\
= & \|\bm{\beta}\|(1+\|\bm{\beta}^*\|^2)\mathbb{E}_{\alpha_2,Z}[\alpha_1^2\tanh'(\|\bm{\beta}\|\sqrt{1+\|\bm{\beta}^*\|^2}Z \alpha_1)] \nonumber \\
= & \|\bm{\beta}\|(1+\|\bm{\beta}^*\|^2)\mathbb{E}_{\alpha_2,\alpha_1,\epsilon}[\alpha_1^2\tanh'(\|\bm{\beta}\|(\|\bm{\beta}^*\|\alpha_2+\epsilon) \alpha_1)] \label{eq:relation2_sub1}.
\end{align}
We thus obtain a second relation by substituting \eqref{eq:relation2_sub1} back to \eqref{eq:fixed_point_condition}:
\begin{align}
\label{eq:relation2}
    1 = (1+\|\bm{\beta}^*\|^2)\mathbb{E}_{\alpha_2,\alpha_1,\epsilon}[\alpha_1^2\tanh'(\|\bm{\beta}\|(\|\bm{\beta}^*\|\alpha_2+\epsilon) \alpha_1)].
\end{align}
The quantity of interest is the following:
\begin{equation}
\label{eq:quantity_of_interest}
 \langle \bm{v}_2,\mathcal{H}\bm{v}_2\rangle
=  -1 + \mathbb{E}_{\alpha_2,\alpha_1,\epsilon} \left[\alpha_2^2(\|\bm{\beta}^*\|\alpha_2+\epsilon)^2 \tanh' \left(\|\bm{\beta}\|\alpha_1(\|\bm{\beta}^*\|\alpha_2+\epsilon) \right) \right]. 
\end{equation}
Let us apply Stein's Lemma with respect to $\alpha_2$ to simplify the expression in \eqref{eq:quantity_of_interest}:
\begin{align}
  & -1 + \mathbb{E}_{\alpha_2,\alpha_1,\epsilon} \left[\alpha_2^2(\|\bm{\beta}^*\|\alpha_2+\epsilon)^2 \tanh' \left(\|\bm{\beta}\|\alpha_1(\|\bm{\beta}^*\|\alpha_2+\epsilon) \right) \right] \nonumber \\
	=& -1 + \mathbb{E}_{\alpha_2,\alpha_1,\epsilon} \left[\nabla_{\alpha_2}\left[\alpha_2(\|\bm{\beta}^*\|\alpha_2+\epsilon)^2\tanh' \left(\|\bm{\beta}\|\alpha_1(\|\bm{\beta}^*\|\alpha_2+\epsilon)\right) \right] \right] \nonumber \\
	= & -1 + \mathbb{E}_{\alpha_2,\alpha_1,\epsilon} \left[(\|\bm{\beta}^*\|\alpha_2+\epsilon)^2\tanh' \left(\|\bm{\beta}\|\alpha_1(\|\bm{\beta}^*\|\alpha_2+\epsilon) \right) \right]  \nonumber \\
	& + 2\|\bm{\beta}^*\| \mathbb{E}_{\alpha_2,\alpha_1,\epsilon} \left[\alpha_2(\|\bm{\beta}^*\|\alpha_2+\epsilon)\tanh' \left(\|\bm{\beta}\|\alpha_1(\|\bm{\beta}^*\|\alpha_2+\epsilon)\right) \right] \nonumber \\
	&+\|\bm{\beta}^*\|\|\bm{\beta}\| \mathbb{E}_{\alpha_2,\alpha_1,\epsilon} \left[\alpha_2\alpha_1(\|\bm{\beta}^*\|\alpha_2+\epsilon)^2 \tanh'' \left(\|\bm{\beta}\|\alpha_1(\|\bm{\beta}^*\|\alpha_2+\epsilon)\right)\right] \nonumber \\
	= & 2\|\bm{\beta}^*\| \mathbb{E}_{\alpha_2,\alpha_1,\epsilon} \left[\alpha_2(\|\bm{\beta}^*\|\alpha_2+\epsilon)\tanh' \left(\|\bm{\beta}\|\alpha_1(\|\bm{\beta}^*\|\alpha_2+\epsilon)\right)\right] \nonumber \\
	&+\|\bm{\beta}^*\|\|\bm{\beta}\| \mathbb{E}_{\alpha_2,\alpha_1,\epsilon} \left[\alpha_2\alpha_1(\|\bm{\beta}^*\|\alpha_2+\epsilon)^2 \tanh'' \left(\|\bm{\beta}\|\alpha_1(\|\bm{\beta}^*\|\alpha_2+\epsilon)\right)\right], \label{eq:ss3}
\end{align}
where \eqref{eq:ss3} follows from relation \eqref{eq:relation1}. In addition, if we use Stein's Lemma again for the following expression with respect to $\alpha_1$, we obtain: 
\begin{align}
&\mathbb{E}_{\alpha_2,\alpha_1,\epsilon} \left[\alpha_2\alpha_1^2(\|\bm{\beta}^*\|\alpha_2+\epsilon) \tanh' \left(\|\bm{\beta}\|\alpha_1(\|\bm{\beta}^*\|\alpha_2+\epsilon)\right)\right] \nonumber \\
= & \mathbb{E}_{\alpha_2,\alpha_1,\epsilon} \nabla_{\alpha_1}\left[\alpha_2\alpha_1(\|\bm{\beta}^*\|\alpha_2+\epsilon) \tanh' \left(\|\bm{\beta}\|\alpha_1(\|\bm{\beta}^*\|\alpha_2+\epsilon)\right)\right] \nonumber\\
= & \mathbb{E}_{\alpha_2,\alpha_1,\epsilon} \left[\alpha_2(\|\bm{\beta}^*\|\alpha_2+\epsilon) \tanh' \left(\|\bm{\beta}\|\alpha_1(\|\bm{\beta}^*\|\alpha_2+\epsilon)\right)\right] \nonumber\\
&+\|\bm{\beta}\|\mathbb{E}_{\alpha_2,\alpha_1,\epsilon}\left[\alpha_2\alpha_1(\|\bm{\beta}^*\|\alpha_2+\epsilon)^2 \tanh'' \left(\|\bm{\beta}\|\alpha_1(\|\bm{\beta}^*\|\alpha_2+\epsilon)\right)\right] \label{eq:relation3}.
\end{align}
Substitute this relation \eqref{eq:relation3} back to \eqref{eq:ss3}, we have
\begin{align}
    & \langle \bm{v}_2,\mathcal{H}\bm{v}_2 \rangle \nonumber \\
    = & \underbrace{\|\bm{\beta}^*\|\mathbb{E}_{\alpha_2,\alpha_1,\epsilon} \left[\alpha_2\alpha_1^2(\|\bm{\beta}^*\|\alpha_2+\epsilon) \tanh' \left(\|\bm{\beta}\|\alpha_1(\|\bm{\beta}^*\|\alpha_2+\epsilon)\right)\right]}_{A} \nonumber \\
    &+\underbrace{\|\bm{\beta}^*\| \mathbb{E}_{\alpha_2,\alpha_1,\epsilon} \left[\alpha_2(\|\bm{\beta}^*\|\alpha_2+\epsilon)\tanh' \left(\|\bm{\beta}\|\alpha_1(\|\bm{\beta}^*\|\alpha_2+\epsilon)\right)\right]}_{B}.
\end{align}
We bound $A$ and $B$ separately. In Lemma \ref{lem:positerm}, we show $A$ is non-negative. In Lemma \ref{lem:lowerbound}, we show $B$ is at least $\frac{\|\bm{\beta}^*\|^2}{1+\|\bm{\beta}^*\|^2}$, thus completing the proof.
\end{proof}

\begin{lem}[Bounding $A$]
	\label{lem:positerm}
	We have
	\begin{align*}
	& \|\bm{\beta}^*\|\mathbb{E}_{\alpha_2,\alpha_1,\epsilon} \left[\alpha_2\alpha_1^2(\|\bm{\beta}^*\|\alpha_2+\epsilon) \tanh' \left(\|\bm{\beta}\|\alpha_1(\|\bm{\beta}^*\|\alpha_2+\epsilon)\right) \right] \\
	= & \frac{\|\bm{\beta}^*\|^2}{1+\|\bm{\beta}^*\|^2}\mathbb{E}_{\alpha_2,\alpha_1,\epsilon} \left[\alpha_1^2(\|\bm{\beta}^*\|\alpha_2+\epsilon)^2\tanh' \left(\|\bm{\beta}\|\alpha_1(\|\bm{\beta}^*\|\alpha_2+\epsilon)\right)\right].
	\end{align*}
\end{lem}

\begin{proof}
Apply Stein's lemma with respect to $\alpha_2$ to the left hand side of the equation:
	\begin{align}
	&\|\bm{\beta}^*\|\mathbb{E}_{\alpha_2,\alpha_1,\epsilon} \left[\alpha_2\alpha_1^2(\|\bm{\beta}^*\|\alpha_2+\epsilon) \tanh' \left(\|\bm{\beta}\|\alpha_1(\|\bm{\beta}^*\|\alpha_2+\epsilon)\right)\right] \nonumber \\
	= & \|\bm{\beta}^*\|\mathbb{E}_{\alpha_2,\alpha_1,\epsilon} \left[\nabla_{\alpha_2} \left[\alpha_1^2(\|\bm{\beta}^*\|\alpha_2+\epsilon) \tanh' \left(\|\bm{\beta}\|\alpha_1(\|\bm{\beta}^*\|\alpha_2+\epsilon)\right)\right] \right]\nonumber \\
	= & \|\bm{\beta}^*\|^2 \mathbb{E}_{\alpha_2,\alpha_1,\epsilon} \left[\alpha_1^2 \tanh' \left(\|\bm{\beta}\|\alpha_1(\|\bm{\beta}^*\|\alpha_2+\epsilon)\right) \right] \nonumber \\
	& +\|\bm{\beta}\|\|\bm{\beta}^*\|^2 \mathbb{E}_{\alpha_2,\alpha_1,\epsilon}\left[\alpha_1^3(\|\bm{\beta}^*\|\alpha_2+\epsilon)\tanh'' \left(\|\bm{\beta}\|\alpha_1(\|\bm{\beta}^*\|\alpha_2+\epsilon)\right)\right] \nonumber\\
	= & \frac{\|\bm{\beta}^*\|^2}{1+\|\bm{\beta}^*\|^2}+ \|\bm{\beta}\|\|\bm{\beta}^*\|^2 \mathbb{E}_{\alpha_2,\alpha_1,\epsilon} \left[\alpha_1^3(\|\bm{\beta}^*\|\alpha_2+\epsilon)\tanh'' \left(\|\bm{\beta}\|\alpha_1(\|\bm{\beta}^*\|\alpha_2+\epsilon)\right)\right] \label{eq:step1}.
	\end{align}
	In \eqref{eq:step1}, we use relation \eqref{eq:relation2} for the first summand. Let us rewrite 
	\begin{align*}
	\|\bm{\beta}^*\|\alpha_2+\epsilon = \sqrt{\|\bm{\beta}^*\|^2+1}z_1,
	\end{align*}
	where $z_1\sim N(0,1)$ and $z_1$ is independent of $\alpha_1$. The following relation holds by applying Stein's lemma:
	\begin{align}
	&\mathbb{E}_{\alpha_2,\alpha_1,\epsilon} \left[\alpha_1^2(\|\bm{\beta}^*\|\alpha_2+\epsilon)^2\tanh' \left(\|\bm{\beta}\|\alpha_1(\|\bm{\beta}^*\|\alpha_2+\epsilon) \right) \right] \nonumber \\
	= & (\|\bm{\beta}^*\|^2+1) \mathbb{E}_{z_1,\alpha_1} \left[\alpha_1^2z_1^2 \tanh' \left(\|\bm{\beta}\|\sqrt{\|\bm{\beta}^*\|^2+1}\alpha_1z_1\right)\right]  \nonumber \\
	= & (\|\bm{\beta}^*\|^2+1) \mathbb{E}_{z_1,\alpha_1} \left[\nabla_{z_1}\left[\alpha_1^2z_1 \tanh' \left(\|\bm{\beta}\|\sqrt{\|\bm{\beta}^*\|^2+1}\alpha_1z_1\right) \right]\right]  \nonumber \\
	=& (\|\bm{\beta}^*\|^2+1) \mathbb{E}_{z_1,\alpha_1} \left[\alpha_1^2\tanh' \left(\|\bm{\beta}\|\sqrt{\|\bm{\beta}^*\|^2+1}\alpha_1z_1\right)\right] \nonumber \\
	& + \|\bm{\beta}\|(\|\bm{\beta}^*\|^2+1)^{1.5} \mathbb{E}_{z_1,\alpha_1} \left[\alpha_1^3z_1\tanh'' \left(\|\bm{\beta}\|\sqrt{\|\bm{\beta}^*\|^2+1}\alpha_1z_1\right)\right] \nonumber \\
	= & 1+\|\bm{\beta}\|(\|\bm{\beta}^*\|^2+1)\mathbb{E}_{\alpha_2,\alpha_1} \left[\alpha_1^3(\|\bm{\beta}^*\|\alpha_2+\epsilon)\tanh'' \left(\|\bm{\beta}\|\alpha_1(\|\bm{\beta}^*\|\alpha_2+\epsilon)\right) \right].\label{eq:step2}
	\end{align}
	Substituting equation \eqref{eq:step2} in \eqref{eq:step1}, we have:
	\begin{align*}
	& \|\bm{\beta}^*\|\mathbb{E}_{\alpha_2,\alpha_1,\epsilon} \left[\alpha_2\alpha_1^2(\|\bm{\beta}^*\|\alpha_2+\epsilon) \tanh' \left(\|\bm{\beta}\|\alpha_1(\|\bm{\beta}^*\|\alpha_2+\epsilon)\right)\right] \\
	= &\frac{\|\bm{\beta}^*\|^2}{1+\|\bm{\beta}^*\|^2}+\frac{\|\bm{\beta}^*\|^2}{1+\|\bm{\beta}^*\|^2}\left[\mathbb{E}_{\alpha_2,\alpha_1,\epsilon} \left[\alpha_1^2(\|\bm{\beta}^*\|\alpha_2+\epsilon)^2\tanh' \left(\|\bm{\beta}\|\alpha_1(\|\bm{\beta}^*\|\alpha_2+\epsilon) \right)\right]-1\right] \\
	= & \frac{\|\bm{\beta}^*\|^2}{1+\|\bm{\beta}^*\|^2}\mathbb{E}_{\alpha_2,\alpha_1,\epsilon} \left[\alpha_1^2(\|\bm{\beta}^*\|\alpha_2+\epsilon)^2\tanh' \left(\|\bm{\beta}\|\alpha_1(\|\bm{\beta}^*\|\alpha_2+\epsilon)\right)\right].
	\end{align*}
\end{proof}

\begin{lem}[Bounding $B$]
	\label{lem:lowerbound}
	We have 
	\begin{align*}
	\|\bm{\beta}^*\|\mathbb{E} \left[ \alpha_2(\|\bm{\beta}^*\|\alpha_2+\epsilon)\tanh' \left(\|\bm{\beta}\|\alpha_1(\|\bm{\beta}^*\|\alpha_2+\epsilon)\right) \right]= \frac{\|\bm{\beta}^*\|^2}{1+\|\bm{\beta}^*\|^2}.
	\end{align*}
\end{lem}
\begin{proof}
	On the one hand, we can use Stein's lemma with respect to $\alpha_1$ for the following quantity:
	\begin{align}
	& \mathbb{E}_{\alpha_2,\alpha_1,\epsilon}\left[\alpha_2\alpha_1\tanh \left(\|\bm\beta\|\alpha_1(\|\bm{\beta}^*\|\alpha_2+\epsilon) \right)\right] \nonumber \\
	= & \mathbb{E}_{\alpha_2,\alpha_1,\epsilon} \left[\nabla_{\alpha_1}\left[ \alpha_2\tanh\left(\|\bm\beta\|\alpha_1(\|\bm{\beta}^*\|\alpha_2+\epsilon)\right)\right] \right] \nonumber\\
	= & \|\bm{\beta}\| \mathbb{E}_{\alpha_2,\alpha_1,\epsilon} \left[\alpha_2(\|\bm{\beta}^*\|\alpha_2+\epsilon)\tanh' \left(\|\bm\beta\|\alpha_1(\|\bm{\beta}^*\|\alpha_2+\epsilon)\right)\right] 
	\label{eq:1st_rep}.
	\end{align}
	On the other hand, we can use Stein's lemma with respect to $\alpha_2$:
	\begin{align}
	& \mathbb{E}_{\alpha_2,\alpha_1,\epsilon} \left[\alpha_2\alpha_1\tanh \left(\|\bm\beta\|\alpha_1(\|\bm{\beta}^*\|\alpha_2+\epsilon)\right) \right] \nonumber\\
	= & \mathbb{E}_{\alpha_2,\alpha_1,\epsilon} \left[ \nabla_{\alpha_2}\left[ \alpha_1\tanh \left(\|\bm\beta\|\alpha_1(\|\bm{\beta}^*\|\alpha_2+\epsilon)\right)\right] \right] \nonumber \\
	= & \|\bm{\beta}\|\|\bm{\beta}^*\| \mathbb{E}_{\alpha_2,\alpha_1,\epsilon} \left[\alpha_1^2\tanh' \left(\|\bm\beta\|\alpha_1(\|\bm{\beta}^*\|\alpha_2+\epsilon)\right) \right] \nonumber \\
	= & \frac{\|\bm{\beta}\|\|\bm{\beta}^*\|}{1+\|\bm{\beta}^*\|^2}\label{eq:2nd_rep}.
	\end{align}
By setting \eqref{eq:1st_rep}$=$\eqref{eq:2nd_rep}, we are done. 	
\end{proof}

\section{Proofs for Main Results on Population EM}

\subsection{Proof of Theorem \ref{theorem:sine}}
\label{appendix:sine}

\theoremsine*

\begin{proof}
From equation (\ref{eq:em_update}), we can compute cosine and sine at the next iteration,
\begin{equation}
	\cos \theta' = \frac{\langle \bm{\beta}^*, \bm{\beta}' \rangle}{\|\bm{\beta}^*\| \|\bm{\beta}'\|} = \frac{S\|\bm{\beta}^*\|^2 + R b_1^*}{\|\bm{\beta}^*\| \sqrt{R^2 + S^2 \|\bm{\beta}^*\|^2 + 2SRb_1^*}},
\end{equation}

\begin{align}
\label{ineq:S_over_R_bound}
	\sin \theta' &= \frac{R b_2^*}{\|\bm{\beta}^*\| \sqrt{R^2 + S^2 \|\bm{\beta}^*\|^2 + 2SRb_1^*}} \nonumber \\
    &= \sin \theta \frac{1}{\sqrt{1 + (S/R)^2 \|\bm{\beta}^*\|^2 + 2(S/R)b_1^*}} \nonumber \\ 
    &\le \sin \theta \frac{1}{\sqrt{1 + 2(S/R)b_1^*}}.
\end{align}

Now we are left with proving $\frac{S}{R}b_1^* \ge \frac{{b_1^*}^2}{\sigma^2 + \|\bm{\beta}^*\|^2}$, which gives us the claimed result by plugging it into (\ref{ineq:S_over_R_bound}). To see that, we first observe
\begin{align*}
    S = \underbrace{\mathbb{E} \left[ \tanhcustom + \frac{\alpha_1b_1}{\sigma^2}y \tanhpcustom \right]}_{A} + \underbrace{b_1^*\mathbb{E} \left[ \frac{\alpha_1^2b_1}{\sigma^2} \tanhpcustom \right]}_{(\frac{b_1^*}{\sigma^2 + \|\bm{\beta}^*\|^2}) R}.
\end{align*}

Since $R \ge 0$ as it is the expectation of positive function, if A is greater than 0, then we get the desired result. Another application of Stein's lemma yields
\begin{align*}
    \mathbb{E} \left[ \tanhcustom y^2 \right] &= \sigma_2^2 \mathbb{E} \left[ \tanhcustom + \frac{\alpha_1b_1}{\sigma^2}y \tanhpcustom \right] \\
    &= \sigma_2^2 A.
\end{align*}

We can rewrite the left side as
\begin{align*}
    \mathbb{E} \left[ \tanhcustom y^2 \right] &= \frac{1}{2} \mathbb{E} \left[ \tanhcustom y^2 \right] + \frac{1}{2} \mathbb{E} \left[ \tanh \left( \frac{\alpha_1 b_1}{\sigma^2} (-y + \alpha_1 b_1^*) \right) y^2 \right] \\
    &= \frac{1}{2} \mathbb{E} \left[ \left( \tanh \left(\frac{\alpha_1 b_1}{\sigma^2} (-y + \alpha_1 b_1^*) \right) + \tanhcustom \right) y^2 \right] \\
    &\ge 0,
\end{align*}
where in the last inequality, we used the fact that $\tanh(c+x) + \tanh(-c+x) \ge 0$ when $x \ge 0$ for any real value $c$. Consequently, $A \ge 0$ and we complete the proof.
\end{proof}

\subsection{Proof of Theorem \ref{theorem:cosine}}
\label{appendix:theorem_cosine}

\theoremcosine*
\begin{proof}
Recall that from the proof in Theorem \ref{theorem:sine}, we have
\begin{align*}
	\cos \theta' = \frac{S\|\bm{\beta}^*\|^2 + Rb_1^*}{\|\bm{\beta}^*\| \sqrt{R^2 + 2SRb_1^* + S^2\|\bm{\beta}^*\|^2}}, \qquad \text{and} \qquad \frac{S}{R} \ge \frac{b_1^*}{\sigma^2 + \|\bm{\beta}^*\|^2}.
\end{align*}

Starting from these two equations, we can get a lower bound of $\cos \theta'$ in terms of $\cos \theta$ and $\sigma$. First observe that
\begin{align*}
    \cos \theta' &= \frac{(S/R)\|\bm{\beta}^*\|^2 + b_1^*}{\|\bm{\beta}^*\| \sqrt{1+2(S/R)b_1^* + (S/R)^2 \|\bm{\beta}^*\|^2}} \nonumber \\
    &\stackrel{(a)}{\ge} \frac{b_1^*(1+\frac{\|\bm{\beta}^*\|^2}{\|\bm{\beta}^*\|^2+\sigma^2})} {\|\bm{\beta}^*\| \sqrt{1 + {b_1^*}^2 \frac{1}{\|\bm{\beta}^*\|^2+\sigma^2} (2+\frac{\|\bm{\beta}^*\|^2}{\|\bm{\beta}^*\|^2+\sigma^2}) }} \nonumber  \\
    &\stackrel{(b)}{\ge} \cos \theta \sqrt{ 1 + \frac{{b_2^*}^2}{k(\sigma^2)^{-1} + {b_1^*}^2} }, 
\end{align*}
where $k(\sigma^2) = \frac{1}{\|\bm{\beta}^*\|^2+\sigma^2} (2+\frac{\|\bm{\beta}^*\|^2}{\|\bm{\beta}^*\|^2+\sigma^2})$. (a) comes from the following:
\begin{align*}
    \frac{z\|\bm{\beta}^*\|^2 + b_1^*}{\|\bm{\beta}^*\| \sqrt{1+2zb_1^* + z^2 \|\bm{\beta}^*\|^2}} &= \sqrt{\frac{z^2\|\bm{\beta}^*\|^2 + 2zb_1^* + {b_1^*}^2/\|\bm{\beta}^*\|^2} {1+2zb_1^* + z^2 \|\bm{\beta}^*\|^2}} \\
    &= \sqrt{1 - \frac{{b_2^*}^2/\|\bm{\beta}^*\|^2}{1+2zb_1^* + z^2 \|\bm{\beta}^*\|^2}},
\end{align*}
where $z \equiv (S/R)$. It shows us that $\cos \theta'$ is an increasing in $(S/R)$ and therefore lower bounded by the lowest possible value of $(S/R)$.

From (b), we can infer that the amount of increase gets smaller as the angle gets smaller. Thus, we can further bound it with straight-forward algebra by
\begin{align}
    \cos \theta \sqrt{ 1 + \frac{{b_2^*}^2}{k(\sigma^2)^{-1} + {b_1^*}^2} } &\ge \cos \theta \sqrt{1 + \frac{\sin^2 \theta}{\cos^2 \theta + \frac{1}{2} (1+\eta^{-2})} } \label{ineq:population_cosine_kappa} \\
    &\ge \cos \theta \sqrt{1 + \frac{\eta^2}{\frac{2}{3} + \eta^2}},
\end{align}
where the last inequality is established since we assumed $\theta \ge \pi/3$.
\end{proof}

\subsection{Proof of Theorem \ref{theorem:conv_distance}}
\label{appendix:theorem_distance}
Before we prove Theorem \ref{theorem:conv_distance}, we state two lemmas that are essential in our proof. Let all the symbols be as defined in Section \ref{sec:population_EM_update}. Recall that
\begin{gather*}
	S = \mathbb{E}_{ \begin{smallmatrix} \alpha_1 \sim \mathcal{N}(0,1) \\ y \sim \mathcal{N}(0,\sigma_2^2) \end{smallmatrix} } 
    \left[ \tanhcustom + \frac{\alpha_1 b_1}{\sigma^2} (y+\alpha_1 b_1^*) \tanhpcustom \right] \\
    R = (\sigma^2 + \|\bm{\beta}^*\|^2) \mathbb{E}_{ \begin{smallmatrix} \alpha_1 \sim \mathcal{N}(0,1) \\ y \sim \mathcal{N}(0,\sigma_2^2) \end{smallmatrix} }
    \left[ \frac{\alpha_1^2 b_1}{\sigma^2} \tanhpcustom \right].
\end{gather*}

\begin{lem} 
\label{lemma:S_bounds}
$1 - \left( \sqrt{1 + \frac{\min(\frac{\sigma_2^2}{\sigma^2}b_1, b_1^*) b_1^*}{\sigma_2^2}} \right)^{-1} \le S \le 1$.
\end{lem}

\begin{proof}
From equation (\ref{single_use:S_lower}) in proof of lemma \ref{lem:em_update_summary}, we get 
\begin{align*}
	S &= \mathbb{E}\left[\alpha_1^2 \tanhcustom - \frac{b_1 b_1^*}{\sigma^2} \alpha_1^2 \tanhpcustom \right] \\
    & \le \mathbb{E} \left[\alpha_1^2 \tanhcustom \right]
    \le \mathbb{E}[\alpha_1^2] = 1,
\end{align*}
where we used $\tanh'(x) \ge 0$ and $\tanh(x) \le 1$ for any $x$.

For the lower bound of S, we can apply the lemmas 1, 2 from \cite{daskalakis2017ten}.
\paragraph{Lemma 1 in \cite{daskalakis2017ten}}
\label{lem:lemma1tensteps}
{\it
    Let $\alpha, \beta \ge 0$ and $X \sim \mathcal{N}(\alpha, \sigma^2)$, then $\mathbb{E}[\tanh'(\beta X/\sigma^2) X] \ge 0$.
}

\paragraph{Lemma 2 in \cite{daskalakis2017ten}}
\label{lem:lemma2tensteps}
{\it
    Let $\alpha, \beta \ge 0$ and $X \sim \mathcal{N}(\alpha, \sigma^2)$, then $\mathbb{E}[\tanh(\beta X/\sigma^2)] \ge 1 - \exp[-\frac{\min(\alpha,\beta), \alpha}{2\sigma^2}]$.
}

We can apply these two lemmas by setting $\alpha = \alpha_1 b_1^*$, $\beta = \alpha_1 \frac{{b_2^*}^2}{\sigma^2}b_1$ (when $\alpha_1 < 0$, we can get the same result due to the symmetry of the expression in sign). It yields
\begin{align*}
	S & \ge \mathbb{E}_{\alpha_1}\left[1 - \exp{ \left[-\frac{\alpha_1^2 b_1^* \min(b_1^*, \frac{\sigma_2^2}{\sigma^2} b_1)} {2\sigma_2^2} \right]} \right] \\
    & \quad = 1 - \frac{1}{\sqrt{1 + \frac{\min(\frac{\sigma_2^2}{\sigma^2} b_1, b_1^*)b_1^*}{\sigma_2^2} }}.
\end{align*}
\end{proof}

\begin{lem}
\label{lemma:b_1_prime}
$b_1'$ is increasing in $b_1$. Furthermore, in the limit $b_1 \rightarrow \infty$, 
\begin{equation}
    \lim_{b_1 \rightarrow \infty} b_1' = \frac{2}{\pi} (b_1^* \tan^{-1} \left( \frac{b_1^*}{\sigma_2} \right) + \sigma_2).
\end{equation}
\end{lem}

\begin{proof}
    First, we show that $b_1'$ is increasing in $b_1$. From (\ref{eq:b_1_prime}), differentiate it with respect to $b_1$ yields
    \begin{align}
    \label{eq:b_1_prime_increasing}
        \frac{d b_1'}{d b_1} = \mathbb{E} \left[ \tanh'(\frac{b_1\alpha_1}{\sigma^2}y) y^2 \alpha_1^2 \right] \ge 0.
    \end{align}
    
    Next, we show the limit value of $b_1'$. Recall that $b_1' = b_1^*S + R$. Again from Stein's lemma, $R$ can be rewritten as
    \begin{gather*}
        R = \frac{\sigma^2 + \|\bm{\beta}^*\|^2}{\sigma_2^2} \mathbb{E}_{\alpha_1, y} \left[ \tanhcustom y\alpha_1 \right].
    \end{gather*}
    In the limit $b_1 \rightarrow \infty$, $\tanh$ function becomes sign function. Therefore,
    
    \begin{align*}
        \mathbb{E}_{\alpha_1, y} [\text{sign}(\alpha_1(y+\alpha_1 b_1^*)) y\alpha_1 ]
    	&= \frac{1}{\pi} \int_{0}^{\infty} 2\frac{\alpha_1}{\sigma_2} e^{-\frac{\alpha_1^2}{2}} \left( \int_{\alpha_1\bm{\beta}_1^*}^{\infty} y e^{-\frac{y^2}{2\sigma_2^2}}dy \right)  d\alpha_1 
        \\
        &= \frac{2}{\pi} \int_{0}^{\infty} \alpha_1\sigma_2 e^{-\frac{\alpha_1^2(b_1^*)^2}{2\sigma_2^2}} e^{-\frac{\alpha_1^2}{2}} d\alpha_1 \\
        &= \frac{2}{\pi} \sigma_2 / (1 +  (b_1^*/\sigma_2)^2),
    \end{align*}
    
    \begin{gather*}
    	\therefore \lim_{b_1 \rightarrow \infty} R = \frac{2}{\pi} \sigma_2.
    \end{gather*}

    Now we find a limit value of S. In the limit, $\lim_{c \rightarrow \infty} cx \tanh'(cx) = 0$ for all $x$. Therefore, 
    \begin{align*}
        \lim_{b_1 \rightarrow \infty} S &= \mathbb{E} [\text{sign} (\alpha_1(y+\alpha_1b_1^*))]
        = \frac{1}{\pi} \int_{0}^{\infty} \int_{-\alpha_1 b_1^*}^{\alpha_1 b_1^*} e^{-\frac{y^2}{2\sigma_2^2}} e^{-\frac{\alpha_1^2}{2}} \\ 
        &= \frac{2}{\pi} \int_{0}^{\infty} \int_{0}^{\alpha_1 b_1^*/\sigma_2} e^{-\frac{y^2}{2}} e^{-\frac{\alpha_1^2}{2}} 
        = \frac{2}{\pi} \tan^{-1}(b_1^* / \sigma_2).
    \end{align*}
    
    Combining the results, we get the desired lemma.
\end{proof}

Now we are ready to prove Theorem \ref{theorem:conv_distance}.
\theoremdist*

\begin{proof}[Proof of Theorem \ref{theorem:conv_distance}]
First, difference in second coordinate is easily bounded.
\begin{equation}
    (b_2^* - b_2') = (1-S) b_2^* \le \left( \sqrt{1 + \frac{\min(\frac{\sigma_2^2}{\sigma^2}b_1, b_1^*) b_1^*}{\sigma_2^2}} \right)^{-1} b_2^*.
\end{equation}
We therefore focus on giving a bound for $|b_1' - b_1^*|$.

We start from the following observation. Suppose $b_1 = \frac{\sigma^2}{\sigma_2^2} b_1^*$. From equation (\ref{single_use:b_1_prime}), we have
\begin{align}
\label{eq:b_1_prime_fixed}
b_1' = \mathbb{E}_{\alpha_1}[\mathbb{E}_{y \sim \mathcal{N}(\alpha_1 b_1^*,\sigma_2^2)}[\tanh( \frac{\alpha_1 b_1^*}{\sigma_2^2}y )y] \alpha_1] = \mathbb{E}_{\alpha_1} [\alpha_1^2 b_1^*] = b_1^*.
\end{align}
Also from Lemma \ref{lemma:b_1_prime}, $b_1'$ is increasing in $b_1$. We will separate the cases based on this point.

\paragraph{\it Case I. $b_1 \le \frac{\sigma^2}{\sigma_2^2} b_1^*$:} 
    \begin{align*}
        b_1' - \newbeta &= \mathbb{E}_{\alpha_1} \left[ \alpha_1 \left( \mathbb{E}_{y \sim \mathcal{N}(\alpha_1 b_1^*,\sigma_2^2)} \left[\tanh \left( \frac{\alpha_1 (\newbeta)}{\sigma_2^2}y \right)y \right] - 
        \mathbb{E}_{y \sim \mathcal{N}(\alpha_1(\newbeta), \sigma_2^2)} \left[\tanh \left( \frac{\alpha_1 (\newbeta) }{\sigma_2^2}y \right)y \right] \right) \right] \\
        &\stackrel{(a)}{\ge} \left(b_1^* - \newbeta \right) \mathbb{E} \left[\alpha_1^2 \min_{\mu \in (\newbeta, b_1^*)} \frac{\partial}{\partial \mu} \left(\mathbb{E} \left[\tanh \left( \frac{\alpha_1 (\newbeta) }{\sigma_2^2}(y+\mu) \right) (y+\mu) \right] \right) \right] \\
        &\stackrel{(b)}{\ge} \left(b_1^* - \newbeta \right) \mathbb{E} \left[\alpha_1^2 \left(1 - \exp{ \left( -\frac{\alpha_1^2 {\min(\newbeta, b_1^*)}^2}{2\sigma_2^2} \right) } \right) \right],
    \end{align*}
where in (a) we used mean-value theorem, and in (b) we applied lemma 1, 2 in \cite{daskalakis2017ten}. In turn, we have
\begin{align}
\label{ineq:newbeta_smaller}
    b_1^* - b_1' \le \kappa^3 \left(b_1^* - \newbeta \right) \le \kappa^3 (b_1^* - b_1),
\end{align}
where we have $\kappa = \left({\sqrt{1 + \frac{\min(\newbeta, b_1^*)^2}{\sigma_2^2}}} \right)^{-1}$ and plugging the relation $b_1 \le \newbeta \le b_1^*$ into the above.

Finally, we have $\left( \sqrt{1 + \frac{\min(\frac{\sigma_2^2}{\sigma^2}b_1, b_1^*) b_1^*}{\sigma_2^2}} \right)^{-1} \le \kappa$. Combining them altogether, we have
\begin{align*}
    \|\bm{\beta}^* - \bm{\beta}'\| \le \kappa \|\bm{\beta}^* - \beta\|.
\end{align*}

\paragraph{\it Case II. $b_1 > \frac{\sigma^2}{\sigma_2^2}b_1^*$, $\sigma > b_2^*$:}
Following the exactly same procedure above, we have
\begin{align}
\label{single_use:b_1_diff_bound}
    b_1' - b_1^* \le \kappa^3 (\newbeta - b_1^*) = \kappa^3 (b_1 - b_1^*) + \kappa^3 \frac{{b_2^*}^2}{\sigma^2}b_1.
\end{align}
By the condition in this case, $\kappa = \left({\sqrt{1 + \frac{{b_1^*}^2}{\sigma_2^2}}} \right)^{-1} = \sqrt{\frac{\sigma^2 + {b_2^*}^2}{\sigma^2 + \|\bm{\beta}^*\|^2}}$. We divided cases into two parts. \\

(i) Suppose $b_1 > 2b_1^*$, or $b_1 < 2(b_1 - b_1^*)$. Then,
\begin{align*}
    b_1' - b_1^* &\le \kappa^3 (b_1 - b_1^*) (1 + 2 \frac{{b_2^*}^2}{\sigma^2}) \\
    &= \kappa (b_1 - b_1^*) (\frac{\sigma^2 + {b_2^*}^2}{\sigma^2 + \|\bm{\beta}^*\|^2}) (1 + \frac{2{b_2^*}^2}{\sigma^2}) \\
    &= \kappa \underbrace{ \left(\frac{\sigma^2 + {b_2^*}^2}{\sigma^2 + {b_1^*}^2 + {b_2^*}^2} \frac{\sigma^2 + 2{b_2^*}^2}{\sigma^2} \right)}_{A} (\bm{\beta}_1 - \bm{\beta}_1^*).
\end{align*}

Check if $A$ is less than 1. To see that, 
\begin{align*}
    \sigma^2(\sigma^2 + {b_1^*}^2 + {b_2^*}^2) - (\sigma^2 + (b_2^*)^2)(\sigma^2 + 2(b_2^*)^2) \\
    = \sigma^2 ({b_1^*}^2 - 2{b_2^*}^2) - 2{b_2^*}^4 \stackrel{(a)}{\ge} \sigma^2 ({b_1^*}^2 - 4{b_2^*}^2) \stackrel{(b)}{\ge} 0,
\end{align*}
where (a) comes from $b_2^* < \sigma$ and (b) comes from $\tan \frac{\pi}{8} < 1/2$.

\begin{equation*}
    \therefore b_1' - b_1^* \le \kappa (b_1 - b_1^*)
\end{equation*}

(ii) $b_1 < 2b_1^*$. We will assume $b_1 \frac{{b_2^*}^2}{\sigma^2} \ge (\frac{1}{\kappa^2} - 1) (b_1 - b_1^*)$. Otherwise, we can easily get $b_1' - b_1^* \le \kappa (b_1 - b_1^*)$ similarly by plugging it into equation (\ref{single_use:b_1_diff_bound}).

\begin{align*}
	(b_1' - b_1^*)^2 &\le \kappa^6 (b_1 - b_1^*)^2 + \kappa^6 \left( 2(\fracbsig)^2 b_1(b_1 - b_1^*) + (\fracbsig)^4 b_1^2 \right) \\
    &\le \kappa^6(b_1 - b_1^*)^2 + \kappa^6 (\fracbsig)^4 b_1^2 \left(2(\frac{\kappa^2}{1 - \kappa^2}) + 1 \right) \\
    &= \kappa^6(b_1 - b_1^*)^2 + \underbrace{\kappa^6 (\fracbsig)^4 b_1^2 \left( \frac{2\sigma^2 + 2{b_2^*}^2 + {b_1^*}^2 }{{b_1^*}^2} \right)}_{B}.
\end{align*}

We bound $B$. We rearrange terms as below:
\begin{align*}
	B &= \kappa^6 (\fracbsig)^4 b_1^2 \left(\frac{2\sigma^2 + 2{b_2^*}^2 + {b_1^*}^2 }{{b_1^*}^2} \right) \\
    &= \kappa^2 (\fracbsig)^4 b_1^2 \left(\frac{2\sigma^2 + 2{b_2^*}^2 + {b_1^*}^2 }{{b_1^*}^2} \right) \left( \frac{\sigma^2 + {b_2^*}^2}{\sigma^2 + \|\bm{\beta}^*\|^2} \right)^2 \\
    &= \kappa^2 {b_2^*}^4 (\frac{b_1^2}{{b_1^*}^2}) \left(\frac{2\sigma^2 + 2{b_2^*}^2 + {b_1^*}^2}{\sigma^2 + {b_2^*}^2 + {b_1^*}^2} \right) \left(\frac{(\sigma^2 + {b_2^*}^2)^2}{\sigma^4} \right) \frac{1}{\sigma^2 + \|\bm{\beta}^*\|^2}
    \\ &\le \kappa^2 {b_2^*}^4 * 4 * 2 * 4 * \left( \frac{1}{\sigma^2+\|\bm{\beta}^*\|^2} \right) 
    \\ &= \kappa^2 \frac{32 {b_2^*}^2}{\sigma^2 + \|\bm{\beta}^*\|^2} {b_2^*}^2
\end{align*}

Therefore, we get $(b_1' - b_1^*)^2 \le \kappa^2 (b_1 - b_1^*)^2 + \kappa^2 \frac{32 {b_2^*}^2}{\sigma^2 + \|\bm{\beta}^*\|^2} {b_2^*}^2$. Combining it with $(b_2' - b_2^*)^2 \le \kappa^2 (b_2 - b_2^*)^2$ yields
\begin{align*}
    \|\bm{\beta}' - \bm{\beta}^*\|^2 \le \kappa^2 \|\bm{\beta} - \bm{\beta}^*\|^2 + \kappa^2 \frac{32 {b_2^*}^2}{\sigma^2 + \|\bm{\beta}^*\|^2} {b_2^*}^2
\end{align*}

Now using $\sqrt{a^2 + b^2} \le a + \frac{b^2}{2a}$,
\begin{align*}
    \|\bm{\beta}' - \bm{\beta}^*\| &\le \kappa \|\bm{\beta} - \bm{\beta}^*\| + \kappa \frac{16 {b_2^*}^2}{\sigma^2 + \|\bm{\beta}^*\|^2} \frac{b_2^*}{\|\bm{\beta} - \bm{\beta}^*\|} b_2^* \\
    &\le \kappa \|\bm{\beta} - \bm{\beta}^*\| + \kappa (16 \sin^3\theta) \|\bm{\beta}^*\| \frac{\eta^2}{1 + \eta^2},
\end{align*}
where we used $\frac{b_2^*}{\|\bm{\beta} - \bm{\beta}^*\|} \le 1$. 

\paragraph{\it Case III. $b_1 > \frac{\sigma^2}{\sigma_2^2}b_1^*$, $\sigma < b_2^*$:}

This condition leads us to a special analysis, a constant rate of contraction in local region with high SNR.

First note that, $b_1' \ge b_1^*$ and its difference $(b_1' - b_1^*)$ is increasing in $b_1$. Therefore, invoking lemma \ref{lemma:b_1_prime} yields 
\begin{align*}
    b_1' - b_1^* &\le \frac{2}{\pi} (\sigma_2 + b_1^* \tan^{-1} (\frac{b_1^*}{\sigma_2})) - b_1^* \\
    &\le \frac{2}{\pi} (\sigma_2 + b_1^* \tan^{-1} (\frac{b_1^*}{b_2^*})) - b_1^* \\
    &\le \frac{2}{\pi} (\sqrt{2} - \theta \cot \theta) b_2^*,
\end{align*}
where we used $\sigma_2^2 = \sigma^2 + {b_2^*}^2 \le 2{b_2^*}^2$, $\tan^{-1} (\frac{b_1^*}{b_2^*}) = \frac{\pi}{2} - \theta$, and $b_1^* = b_2^* \cot \theta$.

One can easily check that $\theta \cot \theta$ is decreasing in $[0, \frac{\pi}{2}]$. Therefore, we can further bound it:
\begin{align*}
    b_1' - b_1^* \le \frac{2}{\pi} (\sqrt{2} - \frac{\pi}{8} \cot \frac{\pi}{8}) b_2^* \le 0.3 b_2^*.
\end{align*}

On the other side, 
\begin{align*}
	b_2^* - b_2' & = (1-S)b_2^* \le \frac{b_2^*}{\sqrt{1+(b_1^*/\sigma_2)^2}} \\ 
	& \le \frac{b_2^*}{\sqrt{1+\frac{1}{2} (b_1^*/b_2^*)^2 }} = \frac{b_2^*}{\sqrt{1+ \frac{\cot^2 \frac{\pi}{8}}{2}}} \le 0.51 b_2^*.
\end{align*}

Combining the result, we get 
\begin{align*}
    \|\bm{\beta}' - \bm{\beta}^*\| \le 0.6 b_2^* \le 0.6\|\bm{\beta} - \bm{\beta}^*\|,
\end{align*}
as claimed. 
\end{proof}

\paragraph{Proof of Corollary \ref{corollary:distance}}
\corollarydist*

\begin{proof}
We first show that $\kappa$ is only decreasing as iteration goes on. It is enough to show that after one EM iteration, $b_1' \ge \min(\newbeta, b_1^*)$, and $b_1^*$ is increasing as the iteration is going on.

If $\newbeta$ is larger than $b_1^*$, $b_1'$ becomes larger than $b_1^*$ as we can conclude from Lemma \ref{lemma:b_1_prime} and (\ref{eq:b_1_prime_fixed}). If $\newbeta$ were less than $b_1^*$, then the corresponding $\newbeta$ at the next iteration is larger than it, as it is inferred from (\ref{ineq:newbeta_smaller}). The fact that $b_1^* = \|\bm{\beta}^*\| \cos \theta_t$ is increasing is obvious from the fact that angle is always decreasing. 

Now we will fix $\kappa$, the contraction rate at the first iteration. We compare the following quantities:
\begin{align*}
    0.6, \left(\sqrt{1+ \frac{2 {b_1^*}^2}{\sigma^2+\|\bm{\beta}^*\|^2} } \right)^{-3}, 
    \left({\sqrt{1 + \frac{\min(\newbeta, b_1^*)^2}{\sigma_2^2}}} \right)^{-1}.
\end{align*}
each of which can be rewritten as
\begin{align*}
    0.6, \left(\sqrt{1+ \frac{2 \eta^2 \cos^2 \theta_0}{1+\eta^2} } \right)^{-3}, 
    \left(\sqrt{1 + (1+\eta^2 \sin^2 \theta_0) \frac{\|\bm{\beta}_0\|^2}{\sigma^2}} \right)^{-1},
    \left(\sqrt{1 + \frac{\eta^2 \cos^2 \theta_0}{1+\eta^2 \sin^2 \theta_0}} \right)^{-1}.
\end{align*}
Since we start from $\theta_0 < \pi/8$, we can plug $\theta_0 = \pi/8$ above and simplify the candidates as (\ref{eq:kappa_candidates}). We will pick the maximum among these values and fix $\kappa$.

Next, we rewrite the equation now with subscript $t$ on each variable:
\begin{align*}
    \|\bm{\beta}_{t+1} - \bm{\beta}^*\| &\le \kappa \|\bm{\beta}_t - \bm{\beta}^*\| + \kappa (16 \sin^3 \theta_t) \|\bbeta^*\| \frac{\eta^2}{1+\eta^2} \\
    &\le \kappa^2 \|\bm{\beta}_{t-1} - \bm{\beta}^*\| + 2 \kappa^2 (16 \sin^3 \theta_{t-1}) \|\bbeta^*\| \frac{\eta^2}{1+\eta^2} \\
    &... \\
    &\le \kappa^{T} \|\bm{\beta}_0 - \bm{\beta}^*\| + t \kappa^{T} (16 \sin^3 \theta_0) \|\bbeta^*\| \frac{\eta^2}{1+\eta^2} \\
    &\le \kappa^{T} \|\bm{\beta}_0 - \bm{\beta}^*\| + t \kappa^{T} \|\bbeta^*\| \frac{\eta^2}{1+\eta^2},
\end{align*}
where for the last inequality, we used $\theta_0 < \pi/8$.
\end{proof}

\section{Proofs for Finite-Sample Based EM}

\subsection{Proof for Theorem \ref{theorem:finite_cosine_update}}
\finitecosineupdate*
\begin{proof}
We start from the end of the proof for Theorem \ref{thm:finite_cos_update}. We now replace statistical errors in terms of $\epsilon_f$ using the sample complexity $n = \tilde{O} ((1+\eta^{-2})d / \epsilon_f^2)$. Recall the way we compute cosine,
\begin{align*}
    \cos \tilde{\theta}' &= \frac{\langle \tilde{\bm{\beta}}', \bm{\beta}^* \rangle}{\|\tilde{\bm{\beta}}'\| \ \|\bm{\beta}^*\|} \\
    &= \frac{\langle \bm{\beta}', \bm{\beta}^* \rangle}{\|\tilde{\bm{\beta}}'\| \ \|\bm{\beta}^*\|} + \frac{\langle \tilde{\bm{\beta}}' - \bm{\beta}', \bm{\beta}^* \rangle}{\|\tilde{\bm{\beta}}'\| \ \|\bm{\beta}^*\|} \\
    &= \cos \theta' \frac{\|\bm{\beta}'\|}{\|\tilde{\bm{\beta}}'\|} + \frac{\langle \tilde{\bm{\beta}}' - \bm{\beta}', \bm{\beta}^* \rangle}{\|\tilde{\bm{\beta}}'\| \ \|\bm{\beta}^*\|} \\
    &\ge \cos \theta' \left( 1 - \frac{\epsilon_f}{\|\bm{\beta}'\|/\|\bm{\beta}^*\| + \epsilon_f} \right) - \max \left( \frac{\epsilon_f}{\sqrt{d}}, \epsilon_f^2 \right) \frac{\|\bm{\beta}^*\|}{\|\tilde{\bm{\beta}}'\|} \\
    &\ge \cos \theta' (1 - 10 \epsilon_f) - O\left( \max \left( \frac{\epsilon_f}{\sqrt{d}}, \epsilon_f^2 \right) \right) \\
    &\ge \kappa (1 - 10\epsilon_f) \cos \theta - O\left( \max \left( \frac{\epsilon_f}{\sqrt{d}}, \epsilon_f^2 \right) \right),
\end{align*}
where the last two inequalities follows from the Lemma \ref{lemma:finite_b1p_lb} in Appendix \ref{appendix:subsec_aux_lb_norm}, and equation (\ref{ineq:population_cosine_kappa}) in the proof of Theorem \ref{theorem:cosine}. 

Now for sine, we have that
\begin{align*}
    \sin^2 \tilde{\theta}' &= 1 - \cos^2 \tilde{\theta}' \\
    &\le 1 - \cos^2 \theta' + O(\epsilon_f) \\
    &\le \sin^2 \theta' + O(\epsilon_f) \\
    &\le \kappa' \sin^2 \theta + O(\epsilon_f),
\end{align*}
where the last inequality comes from Theorem \ref{theorem:sine}.
\end{proof}

\subsection{Proof of Theorem \ref{thm:easyem_update}}
\easyemupdate*

\begin{proof}
    From bounding $A$ in the proof of Theorem \ref{thm:finite_cos_update}, we can directly see
    \begin{gather*}
        |(\tilde{\bm{\beta}}'' - \bm{\beta}')^{\top} \bm{\beta}^*| \le c_1     \sqrt{\sigma^2+\|\bm{\beta}^*\|^2} \sqrt{\frac{1}{n} \log(1/\delta)},
    \end{gather*}
    for some constant $c_1$. For bounding the norm, standard covering set argument tells that we can take union bound over 1/2-covering set of unit sphere to bound $P(\sup_{v \in \mathbb{S}^d} |(\tilde{\bm{\beta}}'' - \bm{\beta}')^{\top} v| \ge t)$, from which we can conclude 
    $$
    \|\tilde{\bm{\beta}}'' - \bm{\beta}'\| \le c_2 \sqrt{\sigma^2+\|\bm{\beta}^*\|^2} \sqrt{\frac{d}{n} \log(1/\delta)},
    $$ 
    with probability at least $1 - \delta$.
    
    Bound for cosine and sine can be derived by the exactly same procedure used in the proof of Theorem \ref{theorem:finite_cosine_update}.
\end{proof}

Now we are ready to prove lemmas on finite-sample based EM in three convergence phases. We will use the concentration results that with probability $1-\delta/T$ in each EM iteration, $\|\tilde{\bm{\beta}}' - \bm{\beta}'\| \le \epsilon_f$ from \cite{balakrishnan_statistical_2017} as well as Theorem \ref{thm:finite_cos_update} in Appendix \ref{appendix:aux_concentration_one_dir}.

\paragraph{Proof of Lemma \ref{lemma:finite_cosine}}
\finitesamplecosine*
\begin{proof}
From equation (\ref{ineq:cosine_sample}) with sufficiently small $\epsilon_f$, we have 
\begin{align*}
    \cos \tilde{\theta}_T &\ge \kappa \cos \tilde{\theta}_{T-1} - O(\frac{\epsilon_f}{\sqrt{d}}) \\
    &\ge \kappa^2 \cos \tilde{\theta}_{T-2} - (1+\kappa) O(\frac{\epsilon_f}{\sqrt{d}}) \\ 
    &... \\
    &\ge \kappa^{T} \cos \tilde{\theta}_0 - (1+\kappa + \kappa^2 + ... + \kappa^{T-1}) O(\frac{\epsilon_f}{\sqrt{d}}) \\
    &\ge \kappa^{T} \cos \tilde{\theta}_0 - \frac{\kappa^{T} - 1}{\kappa - 1}  O(\frac{\epsilon_f}{\sqrt{d}}),
\end{align*}
where each inequality holds with probability at least $1 - \delta/T$, and all inequalities hold with probability $1-\delta$ by taking a union bound.
\end{proof}

\paragraph{Proof of lemma \ref{lemma:finite_sine}}

\finitesamplesine*
\begin{proof}
    Similarly, 
    \begin{align*}
        \sin^2 \tilde{\theta}_T &\le \kappa^2 \sin^2 \tilde{\theta}_{T-1} + O(\epsilon_f) \\
        &\le \kappa^4 \sin^2 \tilde{\theta}_{T-2} + (1+\kappa^2) O(\epsilon_f) \\
        &... \\
        &\le \kappa^{2T} \sin^2 \tilde{\theta}_0 + (1+\kappa^2 + \kappa^4 + ... + \kappa^{2(T-1)}) O(\epsilon_f) \\
        &\le \kappa^{2T} \sin^2 \tilde{\theta}_0 + \frac{1}{1 - \kappa^2}  O(\epsilon_f),
    \end{align*}
    with probability $1-\delta$.
    
    Finally, 
    \begin{align*}
        \frac{1}{1 - \kappa^2} O(\epsilon_f) = \frac{\min(1,\eta^2)}{1 - \kappa^2} O(\epsilon) = \min(1, \eta^2) \frac{1+1.5\eta^2}{0.5\eta^2} O(\epsilon) 
        = O(\epsilon),
    \end{align*}
    which yields the desired result.
\end{proof}

\section{Proof of Theorem \ref{thm:final_stats_error}}
\label{sec:proof_of_statserror}
\finitestatserror* 

The first part of the theorem is proved in Section \ref{subsec:original_final_stats_error} and the second part of the theorem is proved in Section \ref{subsec:improved_l2}.

\subsection{Statistical Bound depending on the optimal parameter}
\label{subsec:original_final_stats_error}
\begin{lem}[Convergence of Distance in Finite-Sample EM]
\label{lemma:finite_dist}
	Suppose we get $\tilde{\bbeta}_0$ whose angle formed with $\bbeta^*$ is less than $\pi/8$ from previous phase. We run sample-splitting EM with $n/T = \tilde{O}(\max(1,\eta^{-2})d/\epsilon^2)$, getting
    \begin{equation}
    	\|\tilde{\bbeta}_T - \bbeta^*\| \le \kappa^{T} \|\tilde{\bbeta}_0 - \bbeta^*\| + T \kappa^{T} \|\bbeta^*\| \frac{\eta^2}{1+\eta^2} + O(\epsilon) \|\bbeta^*\|,
    \end{equation}
where $\kappa$ is the maximum among (\ref{eq:kappa_candidates}) as in Corollary \ref{corollary:distance}.

After $T = O(\max(1,\eta^{-2}) \log(1/\epsilon_1))$ iterations, we get $\bm{\beta}_*$ within $O(\epsilon)$ error.
\end{lem}

\begin{proof}
We assume $||\bbeta^*|| = 1$ for the sake of simplicity in the proof. We start from Theorem \ref{theorem:conv_distance}. Note that the chosen $\kappa$ satisfies $\sin^3 \tilde{\theta}_T \le \kappa^{T} \sin^3 \tilde{\theta}_{0} + \frac{1}{1-\kappa} O(\epsilon_f)$, which can be shown similarly as Lemma \ref{lemma:finite_sine}.
\begin{align*}
    \|\tilde{\bbeta}_{T} - \bbeta^*\| &\le \kappa \|\tilde{\bbeta}_{T-1} - \bbeta^*\| + O(\epsilon_f) + \kappa (16 \sin^3 \tilde{\theta}_{T-1}) \frac{\eta^2}{1+\eta^2}\\
    &\le \kappa^2 \|\tilde{\bbeta}_{T-2} - \bbeta^*\| + (1+\kappa) O(\epsilon_f) + \frac{16\eta^2}{1+\eta^2}  (\kappa^2 \sin^3 \tilde{\theta}_{T-2} + \kappa \sin^3 \tilde{\theta}_{T-1}) \\
    &... \\
    &\le \kappa^{T} \|\tilde{\bbeta}_0 - \bbeta^*\| + \frac{1}{1-\kappa} O(\epsilon_f) +  \frac{16\eta^2}{1+\eta^2} (\kappa^{T} \sin^3 \tilde{\theta}_0 + \kappa^{T-1} \sin^3 \tilde{\theta}_1 + ... + \kappa \sin^3 \tilde{\theta}_{T-1}) \\
    &\le \kappa^{T} \|\tilde{\bbeta}_0 - \bbeta^*\| + 
    \frac{1}{1-\kappa} O(\epsilon_f) + \frac{16\eta^2}{1+\eta^2} (T \kappa^{T} \sin^3 \tilde{\theta}_0 + \frac{\kappa + \kappa^2 + ... + \kappa^{T}}{1-\kappa} O(\epsilon_f) ) \\
    &\le \kappa^{T} \|\tilde{\bbeta}_0 - \bbeta^*\| + 
    \frac{1}{1-\kappa} O(\epsilon_f) + T \kappa^{T} \frac{\eta^2}{1+\eta^2} + \frac{16\eta^2}{1+\eta^2} \frac{1}{(1-\kappa)^2} O(\epsilon_f) \\
    &= \kappa^{T} \|\tilde{\bbeta}_0 - \bbeta^*\| + T \kappa^{T} \frac{\eta^2}{1+\eta^2} + \frac{1}{1-\kappa} O(\epsilon_f) + \frac{1}{(1-\kappa)^2} \frac{\eta^2}{1+\eta^2} O(\epsilon_f).
\end{align*}
Finally, check that $1-\kappa$ is $O(\min(1, \eta^2))$. Then the statistical error is $O(\epsilon)$, as desired.
\end{proof}

\subsection{Statistical Bound independent of the optimal parameter}
\label{subsec:improved_l2}


The statistical error we have seen in the previous result is proportional to $\|\bm{\beta}^*\|$. This unsatisfactory result, especially in the high SNR regime, has been often ignored in literature as if EM algorithm does not guarantee an exact recovery. However, this is in contrast to the result in \cite{yi2014alternating} where they guaranteed exact recovery in the noiseless setting. In other words, the existing statistical guarantees are not \textit{tight}. In this section, we provide a more refined analysis of the (standard) EM algorithm in the finite sample case. The main difference from previous technique is that instead of coupling $\tilde{\bm{\beta}}'$ and $\bm{\beta}'$ directly, we utilize the sample covariance matrix ($\frac{1}{n}\sum _{i=1}^{n}\bm{x}_i\bm{x}_i^{\top}$) to decompose the error between $\tilde{\bm{\beta}}'$ and $\bm{\beta}^*$ so that the additive statistical error does not depend on $\|\bm{\beta}^*\|$. One implication of our results is that in the high SNR regime, the $l_2$ estimation error of the EM iterate does not scale linearly with $\|\bm{\beta}^*\|$, but only with $\sigma$.\\


\noindent Recall that the $1-d$ EM update $\alpha'$ for GMM with two symmetric components \cite{daskalakis2017ten}, with the current parameter $\alpha$ and the optimal parameter is the following:
\[
\alpha'= \mathbb{E}_{X\sim \mathcal{N}(\alpha^*,\sigma^2)}X\tanh\left(\frac{\alpha X}{\sigma^2} \right).
\]
The consistency property guarantees that:
\begin{align*}
\alpha^*=\mathbb{E}_{X\sim \mathcal{N}(\alpha^*,\sigma^2)}X\tanh\left(\frac{\alpha^* X}{\sigma^2} \right).
\end{align*}
It follows that for each $i=1,\ldots,n$:
\[
\mathbb{E}_{y_{i}\sim \mathcal{N}(\langle \bm{x}_{i},\bm{\beta}^*\rangle,\sigma^2)}y_{i}\bm{x}_{i}\tanh\left(\frac{y_{i}\langle \bm{x}_{i},\bm{\beta}^*\rangle}{\sigma^2 }\right )=\bm{x}_{i}\bm{x}_{i}^{\top}\bm{\beta}^*.
\]
This allows us to decompose the difference of $\tilde{\bm{\beta}}'-\bm{\beta}^*$
in the following way:
\begin{align*}
 & \tilde{\bm{\beta}}'-\bm{\beta}^*\\
= & \left(\frac{1}{n}\sum_{i=1}^{n}\bm{x}_{i}\bm{x}_{i}^{\top}\right)^{-1}\left(\frac{1}{n}\sum_{i=1}^{n}y_{i}\bm{x}_{i}\tanh \left(\frac{y_{i}\langle \bm{x}_{i},\bm{\beta}\rangle}{\sigma^2}\right)-\mathbb{E}_{y}\frac{1}{n}\sum_{i=1}^{n}y_{i}\bm{x}_{i}\tanh \left(\frac{y_{i}\langle \bm{x}_{i},\bm{\beta}^*\rangle}{\sigma^2}\right)\right)\\
= & \left(\frac{1}{n}\sum_{i=1}^{n}\bm{x}_{i}\bm{x}_{i}^{\top}\right)^{-1}\underbrace{\left(\frac{1}{n}\sum_{i=1}^{n}y_{i}\bm{x}_{i}\tanh \left(\frac{y_{i}\langle \bm{x}_{i},\bm{\beta}\rangle}{\sigma^2} \right)-\mathbb{E}_{y}\frac{1}{n}\sum_{i=1}^{n}y_{i}\bm{x}_{i}\tanh \left(\frac{y_{i}\langle \bm{x}_{i},\bm{\beta}\rangle}{\sigma^2}\right)\right)}_{\text{term } 1}\\
 & +\left(\frac{1}{n}\sum_{i=1}^{n}\bm{x}_{i}\bm{x}_{i}^{\top}\right)^{-1}\underbrace{\left(\mathbb{E}_{y}\frac{1}{n}\sum_{i=1}^{n}y_{i}\bm{x}_{i}\tanh\left(\frac{y_{i}\langle \bm{x}_{i},\bm{\beta}\rangle}{\sigma^2}\right)-\mathbb{E}_{y}\frac{1}{n}\sum_{i=1}^{n}y_{i}\bm{x}_{i}\tanh\left(\frac{y_{i}\langle \bm{x}_{i},\bm{\beta}^*\rangle}{\sigma^2}\right)\right)}_{\text{term } 2}.
\end{align*}
For term $1$, we can apply a standard concentration result for function of Gaussian random variables by conditioning on the event that the covariance matrix ($\frac{1}{n}\sum_{i=1}^{n}\bm{x}_i\bm{x}_i^{\top}$) is close to 1 in spectral norm. Specifically, it has been proved that this term is $O(\sqrt{\frac{d}{n}})$ in $\ell_2$ norm, which is independent of $\|\bm{\beta}^*\|$. For term $2$, we observe that for each $i$, the following difference 
\[
\mathbb{E}_{y_i}y_{i}\tanh\left(\frac{y_{i}\langle \bm{x}_{i},\bm{\beta}\rangle}{\sigma^2}\right) - \mathbb{E}_{y_i}y_{i} \tanh\left(\frac{y_{i}\langle \bm{x}_{i},\bm{\beta}^*\rangle}{\sigma^2}\right)
\]
is the difference between a 1-d population EM iterate and the optimal parameter in the GMM problem with the current iterate being $\langle \bm{x}_i,\bm{\beta}\rangle$ and the optimal parameter being $\langle \bm{x}_i, \bm{\beta}^*\rangle$. We are able to adapt the sensitivity analysis technique in \cite{daskalakis2017ten} here to show that this term is a contraction term when SNR is large. 

The main result is summarized in the following theorem:
\begin{restatable}[Improved Convergence of distance in Finite-Sample EM in high SNR]{theorem}{improvedstatsbound}
\label{thm:improved_l2}
There exists constants $C>1$ such that for all $\eta>C$, the following holds: suppose we get $\bm{\beta}_0$ whose angle formed with $\bm{\beta}^*$ is less than $\frac{\pi}{70}$ and $\|\bm{\beta}_0\|\geq \frac{\|\bm{\beta}^*\|}{10}$, we run the sample-splitting finite-sample EM with $n/T=\tilde{O}(d/\epsilon_f^2)$, getting:
\begin{align}
    \|\bm{\beta}_T-\bm{\beta}^*\|\leq (0.95+\epsilon_f)^{\top}\|\bm{\beta }_0-\bm{\beta}^*\|+O(\epsilon_f)\sigma .
\end{align}
After $T=O(\log (1/\epsilon_f))$ iterations, we estimate $\bm{\beta}_*$ with an $\ell_2$ error  bounded by $O(\epsilon_f)$.
\end{restatable}
\subsection{Proof of Theorem \ref{thm:improved_l2}}
\label{appendix:subsec_improved_l2}
\begin{proof}
Consider the scaling: $\bm{\beta}\to \frac{\bm{\beta}}{\sigma}$, $\bm{\beta}^*\to \frac{\bm{\beta}^*}{\sigma}$, and $y\to \frac{y}{\sigma}$. We can without loss of generality assume that $\sigma=1$. In the following proof, we omit the appearance of $\sigma$ for simplicity. As we have shown before, we can decompose the difference $\tilde{\bm{\beta}}'-\bm{\beta}^*$ in the following way:
\begin{align*}
 & \tilde{\bm{\beta}}'-\bm{\beta}^*\\
= & \left(\frac{1}{n}\sum_{i=1}^{n}\bm{x}_i\bm{x}_i^{\top}\right)^{-1}\left(\frac{1}{n}\sum_{i=1}^{n}y_{i}\bm{x}_i\tanh\left(y_{i}\langle \bm{x}_i,\bm{\beta}\rangle\right)-\mathbb{E}_{y}\frac{1}{n}\sum_{i=1}^{n}y_{i}\bm{x}_i\tanh\left(y_{i}\langle \bm{x}_i,\bm{\beta}^*\rangle\right)\right)\\
= & \underbrace{\left(\frac{1}{n}\sum_{i=1}^{n}\bm{x}_i\bm{x}_i^{\top}\right)^{-1}}_{A}\underbrace{\left(\frac{1}{n}\sum_{i=1}^{n}y_{i}\bm{x}_i\tanh\left(y_{i}\langle \bm{x}_i,\bm{\beta}\rangle\right)-\mathbb{E}_{y}\frac{1}{n}\sum_{i=1}^{n}y_{i}\bm{x}_i\tanh\left(y_{i}\langle \bm{x}_i,\bm{\beta}\rangle\right)\right)}_{B}\\
 & +\left(\frac{1}{n}\sum_{i=1}^{n}\bm{x}_i\bm{x}_i^{\top}\right)^{-1}\underbrace{\left(\mathbb{E}_{y}\frac{1}{n}\sum_{i=1}^{n}y_{i}\bm{x}_i\tanh\left(y_{i}\langle \bm{x}_i,\bm{\beta}\rangle\right)-\mathbb{E}_{y}\frac{1}{n}\sum_{i=1}^{n}y_{i}\bm{x}_i\tanh\left(y_{i}\langle \bm{x}_i,\bm{\beta}^*\rangle\right)\right)}_{C}.
\end{align*}
We provide bounds for $A$, $B$ and $C$ in the following:
\begin{itemize}
    \item $\|A\|=1+O\left(\sqrt{\frac{d}{n}}\right)$ (standard concentration result),
    \item Conditioning on the sample covariance matrix has bounded spectral norm, $\|B\|=O\left(\sqrt{\frac {d}{n}}\right)$ (cf. Proposition \ref{prop:term_B}),
    \item If $\eta\geq c$ for some constant $c>1$, $C\leq \left(0.95+O(1/\sqrt{d})\right)\|\bm{\beta}-\bm{\beta}^*\|$ (cf. Proposition \ref{prop:term_c}).
\end{itemize}
Therefore, for one step of EM, with $n/T=\tilde{O}(d/\epsilon_f^2)$ samples, 
\[
\tilde{\bm{\beta}}'- \bm{\beta}^* = (1+\epsilon_f)\epsilon_f + (1+\epsilon_f)(0.95+\epsilon_f) \|\bm{\beta}-\bm{\beta}^*\|\leq 2\epsilon_f +(0.95+2\epsilon_f)\|\bm{\beta}-\bm{\beta}^*\|.
\]
Since all the future iterate remains lower bounded by $\frac{\|\bm{\beta}^*\|}{10}$ (cf. Lemma \ref{lemma:finite_b1p_lb}) in $\ell_2$ norm and the angle increases (cf. Lemma \ref{lemma:finite_sine}). We can use induction to show that:
\[
\|\bm{\beta}_T-\bm{\beta}^*\|\leq (0.95+2\epsilon_f)^{\top}\|\bm{\beta}_0-\bm{\beta}^*\|+\frac{2}{0.2-\epsilon_f}\epsilon
\]
The result follows by picking small enough $\epsilon_f$.
\end{proof}

\begin{proposition}(Controlling $B$)
\label{prop:term_B}
For each fixed $\bm{\beta}$, with probability at least  $1-\exp\left(-cn\right) - 6^d \exp\left(-\frac{nt^2}{72}\right)$,
	\begin{align*}
	\Bigg \|\frac{1}{n}\sum_{i=1}^{n} y_i\bm{x}_i \tanh \left(y_i \langle \bm{x}_i, \bm{\beta}\rangle \right) - \frac{1}{n}\sum_{i=1}^{n}\mathbb{E}_{y_i} \left[y_i\bm{x}_i\tanh \left(y_i\langle \bm{x}_i,\bm{\beta}\rangle\right) \right] \Bigg\| \leq  t
	\end{align*}
for some absolute constant $c>0$.
\end{proposition}

\begin{proof}
We will use the standard epsilon-net argument. Let $\bm{v}\in \mathbb{R}^d$, define 
	\begin{align*}
	f_{\bm{v}}(y): = \frac{1}{n}\sum_{i=1}^{n} y_i\langle \bm{x}_i,\bm{v}\rangle \tanh \left(y_i \langle \bm{x}_i, \bm{\beta}\rangle \right).
	\end{align*} 
	Suppose that $\| \frac{1}{n}\sum_{i=1}^{n}\bm{x}_i\bm{x}_i^{\top}\| \leq 2$, we show in the following that  $f_{\bm{v}}(y)-\mathbb{E}_{y} f_{\bm{v}}(y)$ is $\frac{9}{n}$-subgaussian. 
The tool is a standard concentration result for function of the Gaussian random variables summarized in Lemma \ref{lem:wain}:
\begin{lem}[Lemma 2.1 of \cite{wainwright2015high}]
	\label{lem:wain}
	Let $f: \mathbb{R}^{n} \to \mathbb{R}$ be differentiable, then for every convex $\phi: \mathbb{R}\to \mathbb{R}$, we have 
	\begin{align*}
	\mathbb{E}\left(\phi \left(f(\bm{X})-\mathbb{E} f(\bm{X})\right) \right) \leq \mathbb{E} \left[\phi \left(\frac{\pi}{2}\langle \nabla f(\bm{X}), \bm{Y} \rangle \right) \right],
	\end{align*}
	where $\bm{X},\bm{Y}\sim \mathcal{N}(\bm{0},I_n)$ are standard Gaussians and are independent. 
\end{lem}

\noindent Note that for each $i$, the derivative of $f_{\bm{v}}(y)$ with respect to $y$ can be computed explicitly: 
\begin{align*}
\frac{\partial f_{\bm{v}}(y)}{\partial y_i} = \langle \bm{x}_i,\bm{v}\rangle (\tanh(y_i \langle \bm{x}_i, \bm{\beta}\rangle)+y_i \langle \bm{x}_i, \bm{\beta}\rangle \tanh'(y_i \langle \bm{x}_i, \bm{\beta}\rangle))\; \forall{i}.
\end{align*}
Here we abuse the notation $y$. We actually take the derivative with respect to the noise in $y_i$, which is distributed as a standard Gaussian.
The following numerical inequality for $g(z):=\tanh(z)+z\tanh'(z)$ will be used:
\begin{align*}
|\tanh(z)+z\tanh'(z)|\leq 2.
\end{align*}
For any $\lambda \in \mathbb{R}$, we have:
\begin{align*}
& \mathbb{E}_{y}\left[\exp\left(\lambda \left(f_{\bm{v}}(y)-\mathbb{E}\left[f_{\bm{v}}(y)\right]\right) \right)\right] \\
\leq & \mathbb{E}_{y,\bm{z}} \left[\exp\left(\lambda \left(\frac{\pi}{2}\langle \nabla f_{\bm{v}}'(y), \bm{z}\rangle \right)\right)\right] && (\text{Lemma} \; \ref{lem:wain}) \\
= & \mathbb{E}_{y} \mathbb{E}_{\bm{z}} \left[\exp\left(\lambda \left(\frac{\pi}{2}\frac{1}{n} \sum_{i=1}^{n} z_i \langle \bm{x}_i,\bm{v}\rangle g(y_i \langle \bm{x}_i, \bm{\beta}\rangle) \right)\right) \right] \\
= & \mathbb{E}_{y} \exp \left(\lambda^2 \frac{\pi^2}{8n}\left[\frac{1}{n}\sum_{i=1}^{n} \left(\langle \bm{x}_i,\bm{v}\rangle g\left(y_i \langle \bm{x}_i, \bm{\beta}\rangle \right) \right)^2 \right] \right) && (\text{independence of} \;z_is')\\
\leq & \mathbb{E}_{y} \exp \left(\lambda^2 \frac{\pi^2}{2n}\left[\frac{1}{n}\sum_{i=1}^{n} \langle \bm{x}_i,\bm{v}\rangle^2\right] \right) && (\text{numerical bound on } g)\\
= & \mathbb{E}_{y} \exp\left(\lambda^2 \frac{\pi^2}{2n} \bm{v}^{\top}\left[\frac{1}{n}\sum_{i=1}^{n} \bm{x}_i \bm{x}_i^{\top}\right] \bm{v} \right)\\
\leq & \mathbb{E}_{y} \exp\left(\frac{\pi^2\lambda^2}{n}\right) \leq \exp\left(\frac{18\lambda^2}{n}\right),
\end{align*}
where the last line comes from our assumption that $\|\frac{1}{n}\sum_{i=1}^{n} \bm{x}_i \bm{x}_i^{\top}\|\leq 2$. Using the standard $\frac{1}{2}$-net argument for the norm, we can show that 
\begin{align*}
&\mathbb{P}\left( \bigg\|\frac{1}{n}\sum_{i=1}^{n} \left[y_i \bm{x}_i \tanh \left(y_i \langle \bm{x}_i, \bm{\beta}\rangle \right) - \mathbb{E}_{y_i}y_i \bm{x}_i \tanh \left(y_i \langle \bm{x}_i, \bm{\beta}\rangle \right) \right] \bigg\| >t 
 \;\; \bigg| \;
  \Big\|\frac{1}{n}\sum_{i=1}^{n} \bm{x}_i \bm{x}_i^{\top} \Big\|\leq 2 \right) \\
\leq & 6^d \exp \left(-\frac{nt^2}{72} \right).
\end{align*}
To finish the proof, we note the event $\|\frac{1}{n}\sum_{i=1}^{n} \bm{x}_i \bm{x}_i^{\top} \|\leq 2$ holds with probability at least $1-\exp(-cn)$ for some absolute constant $c>0$.
\end{proof}

\begin{proposition}(Controlling $C$)
\label{prop:term_c}
There exists an absolute constant $c>1$ such that in the regime where $\eta>c$ the following holds: for each fixed $\bm{\beta}$ satisfying $\|\bm{\beta}\|\geq \frac{\|\bm{\beta}^*\|}{10}$, and its angle with $\bm{\beta}^*$, $\theta$ is less than $\frac{\pi}{70}$, we run a finite-sample EM with $n=\tilde{O}\left(\frac{d}{\epsilon_f^2}\right)$, getting:
\begin{align*}
    & \bigg\|\mathbb{E}_{y}\frac{1}{n}\sum_{i=1}^{n}y_{i}\bm{x}_i\tanh(y_{i}\langle \bm{x}_i,\bm{\beta}\rangle)-\mathbb{E}_{y}\frac{1}{n}\sum_{i=1}^{n}y_{i}\bm{x}_i\tanh(y_{i}\langle \bm{x}_i,\bm{\beta}^*\rangle)\bigg\| \\
    \leq & (0.95+\epsilon_f/\sqrt{d}) \|\bm{\beta}-\bm{\beta}^*\|
\end{align*}
\end{proposition}
\begin{proof}
For each $i$, we observe that
\[
\mathbb{E}_{y_i}y_{i}\tanh(y_{i}\langle \bm{x}_i,\bm{\beta}\rangle)-\mathbb{E}_{y_i}y_{i}\tanh(y_{i}\langle \bm{x}_i,\bm{\beta}^*\rangle)
\]
is the difference between the 1-d population EM iterate and the optimal parameter for the GMM problem with the current iterate being $\langle \bm{x}_i,\bm{\beta}\rangle$ and the optimal parameter being $\langle \bm{x}_i,\bm{\beta}^*\rangle$. In \cite{daskalakis2017ten}, they have developed the sensitivity analysis technique to bound the difference with the restriction that the current iterate has the same sign as the optimal parameter. In our case, we note that covariance vector $\bm{x}_i$ can possibly cause $\langle \bm{x}_i,\bm{\beta}\rangle$ and $\langle \bm{x}_i,\bm{\beta}^*\rangle$ to have opposite signs despite the fact that $\bm{\beta}$ has an acute angle with $\bm{\beta}^*$. We get around this issue by performing a more refined sensitivity analysis in both regions: (1)$\langle \bm{x}_i,\bm{\beta}^*\rangle \langle \bm{x}_i,\bm{\beta}\rangle \geq 0$; (2) $\langle \bm{x}_i,\bm{\beta}^*\rangle \langle \bm{x}_i,\bm{\beta}\rangle < 0$.

The key element of the sensitivity analysis is to use the following decomposition:
\begin{align*}
& \mathbb{E}_{y_i\sim \mathcal{N}(\langle \bm{x}_i,\bm{\beta}^*\rangle,1)}y_{i}\tanh(y_{i}\langle \bm{x}_i,\bm{\beta}\rangle)-\mathbb{E}_{y_i\sim \mathcal{N}(\langle \bm{x}_i,\bm{\beta}^*\rangle,1)}y_{i}\tanh(y_{i}\langle \bm{x}_i,\bm{\beta}^*\rangle)\\
= & \mathbb{E}_{y_i\sim \mathcal{N}(\langle \bm{x}_i,\bm{\beta}^*\rangle,1)}y_{i}\tanh(y_{i}\langle \bm{x}_i,\bm{\beta}\rangle)-\mathbb{E}_{y_i\sim \mathcal{N}(\langle \bm{x}_i,\bm{\beta}\rangle,1)}y_{i}\tanh(y_{i}\langle \bm{x}_i,\bm{\beta}\rangle)\\
& + \mathbb{E}_{y_i\sim \mathcal{N}(\langle \bm{x}_i,\bm{\beta}\rangle,1)}y_{i}\tanh(y_{i}\langle \bm{x}_i,\bm{\beta}\rangle)-\mathbb{E}_{y_i\sim \mathcal{N}(\langle \bm{x}_i,\bm{\beta}^*\rangle,1)}y_{i}\tanh(y_{i}\langle \bm{x}_i,\bm{\beta}^*\rangle) \\
= & \mathbb{E}_{y_i\sim \mathcal{N}(\langle \bm{x}_i,\bm{\beta}^*\rangle,1)}y_{i}\tanh(y_{i}\langle \bm{x}_i,\bm{\beta}\rangle)-\mathbb{E}_{y_i\sim \mathcal{N}(\langle \bm{x}_i,\bm{\beta}\rangle,1)}y_{i}\tanh(y_{i}\langle \bm{x}_i,\bm{\beta}\rangle) \\
& + \bm{\beta}-\bm{\beta}^*
\end{align*}
where the last step follows from the consistency property of the EM operator. The mean value theorem tells us that:
\begin{align*}
& \mathbb{E}_{y_i\sim \mathcal{N}(\langle \bm{x}_i,\bm{\beta}^*\rangle,1)}y_{i}\tanh(y_{i}\langle \bm{x}_i,\bm{\beta}\rangle)-\mathbb{E}_{y_i\sim \mathcal{N}(\langle \bm{x}_i,\bm{\beta}\rangle,1)}y_{i}\tanh(y_{i}\langle \bm{x}_i,\bm{\beta}\rangle)\\
= & \int_{t=0}^{1}\mathbb{E}_{y_i \sim \mathcal{N}(\langle \bm{x}_i,Z(t)\rangle,1)}\Delta_i(y)\bm{x}_i\bm{x}_i^{\top}(\bm{\beta}^*-\bm{\beta})dt
\end{align*}
with
\begin{align*}
 Z(t) & =\bm{\beta}+t(\bm{\beta}^*-\bm{\beta}),\\
\Delta_i(y) & =\frac{\partial[y_i\bm{x}_i\tanh(y_i\langle \bm{x}_i,\bm{\beta}\rangle)]}{\partial y_i}=\tanh(y_i\langle \bm{x}_i,\bm{\beta}\rangle)+y_i\langle \bm{x}_i,\bm{\beta}\rangle\tanh'(y_i\langle \bm{x}_i,\bm{\beta}\rangle).  
\end{align*}
Therefore, the original difference is equivalent to:
\begin{align}
& \frac{1}{n}\sum_{i=1}^{n}\mathbb{E}_{y_i\sim \mathcal{N}(\langle \bm{x}_i,\bm{\beta}^*\rangle,1)}y_{i}\tanh(y_{i}\langle \bm{x}_i,\bm{\beta}\rangle)-\frac{1}{n}\sum_{i=1}^{n}\mathbb{E}_{y_i\sim \mathcal{N}(\langle \bm{x}_i,\bm{\beta}^*\rangle,1)}y_{i}\tanh(y_{i}\langle \bm{x}_i,\bm{\beta}^*\rangle) \nonumber\\
= & \frac{1}{n}\sum_{i=1}^{n} \bm{x}_i\bm{x}_i^{\top}\left(1-\int_{0}^{1}\mathbb{E}_{y\sim \mathcal{N}(\langle \bm{x}_i,Z(t)\rangle,1)}\Delta_i(y))dt\right)(\bm{\beta}-\bm{\beta}^*). \label{eq:difference_term}
\end{align}
Since $\bm{\beta}$ and $\bm{\beta}^*$ is fixed, we can assume that the Gaussians $\left\{\bm{x}_i\right\}$s have the orthonormal basis $\left\{\bm{v}_i\right\}_{i=1}^{n}$ satisfying $\bm{v}_1=\hat{\bm{\beta}}_*$ and span$(\bm{v}_1,\bm{v}_2)=$span$(\bm{\beta},\bm{\beta}^*)$. Therefore, to bound the difference term \eqref{eq:difference_term}, it suffices to understand the spectral norm of the $2\times 2$ submatrix:
\begin{equation}
\label{key_matrix}
\left[\frac{1}{n}\sum_{i=1}^{n} \bm{x}_i\bm{x}_i^{\top}\left(1-\int_{0}^{1}\mathbb{E}_{y\sim \mathcal{N}(\langle \bm{x}_i,Z(t)\rangle,1)}\Delta_i(y))dt)(\bm{\beta}-\bm{\beta}^*\right)\right]_{:2,:2}    
\end{equation}

so that we can bound:
\begin{align}
\bigg\|\frac{1}{n}\sum_{i=1}^{n}\mathbb{E}_{y_i\sim \mathcal{N}(\langle \bm{x}_i,\bm{\beta}^*\rangle,1)}y_{i}\tanh(y_{i}\langle \bm{x}_i,\bm{\beta}\rangle)-\frac{1}{n}\sum_{i=1}^{n}\mathbb{E}_{y_i\sim \mathcal{N}(\langle \bm{x}_i,\bm{\beta}^*\rangle,1)}y_{i}\tanh(y_{i}\langle \bm{x}_i,\bm{\beta}^*\rangle)\bigg\| \nonumber \\
\leq \Bigg\|\left[\frac{1}{n}\sum_{i=1}^{n} \bm{x}_i\bm{x}_i^{\top}\left(1-\int_{0}^{1}\mathbb{E}_{y\sim \mathcal{N}(\langle \bm{x}_i,Z(t)\rangle,1)}\Delta_i(y))dt)(\bm{\beta}-\bm{\beta}^*\right)\right]_{:2,:2}\Bigg \|\cdot \|\bm{\beta}-\bm{\beta}^*\| \label{ineq:bound_spectral_norm}.
\end{align}
We will provide an explicit bound for $1-\int_{0}^{1}\mathbb{E}_{y\sim \mathcal{N}(\langle \bm{x}_i,Z(t)\rangle,1)}\Delta_i(y))dt$ in Lemma \ref{lem:2}. For now, we just need to use the fact that they are bounded by a constant. This implies that each entry of the matrix \eqref{key_matrix} is a sub-exponential random variable and it is close to its expectation with statistical error $O(1/\sqrt{n})$. Furthermore, we can deduce the spectral norm of this $2\times 2$ submatrix is close to the spectral norm of the expectated submatrix, with a statistical error $O(1/\sqrt{n})$. The problem is thus further reduced to bound the spectral norm of the following $2\times 2$ matrix:
\begin{align}
\label{eq:expected_mat}
   \mathbb{E}_{X\sim \mathcal{N}(0,I)} \left[1-\int_{0}^{1}\mathbb{E}_{y\sim \mathcal{N}(\langle X,Z(t)\rangle,1)}\Delta(y))dt)\right]_{:2,:2} 
\end{align}

In Lemma \ref{lem:norm_expected_matrix}, it is proved that the spectral norm of \eqref{eq:expected_mat} is bounded by 
\[
 \frac{1}{2}\left(1+\frac{1}{(1+0.5\tau^2)^2}\right)+\frac{1}{\pi}\max\left(1-\frac{1}{(1+0.5\tau^2)^2},\sin(\theta)\right)
 + 2.6\sin(\theta),
\]
where $\tau =\min(\|\bm{\beta}\|,\|\bm{\beta}^*\|)$. Since $\|\bm{\beta}\|\geq \frac{\|\bm{\beta}^*\|}{10}$, it follows that $\tau\geq \frac{\|\bm{\beta}^*\|}{10}$, and 
\[
 1-\frac{1}{(1+0.5\tau^2)^2}
\geq  1-\frac{1}{(1+0.05\|\bm{\beta}^*\|^2)^2}.
\]
As long as $\|\bm{\beta}^*\|$ is sufficiently large, $1-\frac{1}{(1+0.05\|\bm{\beta}^*\|^2)^2}$ will dominate $\sin(\theta)\leq \frac{\pi}{8}$. Therefore, a loose bound for the above spectral norm is:
\[
\frac{1}{2}+\frac{1}{\pi}+ (\frac{1}{2}-\frac{1}{\pi}) \frac{1}{(1+0.05\|\bm{\beta}^*\|^2)^2}+2.6\sin(\theta)
\]
We note that as $\|\bm{\beta}^*\|\to \infty$ and let $\phi=\frac{\pi}{70}$, the above term converges to $\frac{1}{2}+\frac{1}{\pi}+4\sin(\theta)<0.95$. Thus, there exists $c$, such that whenever $\|\bm{\beta}^*\|\geq c$, the above ratio is bounded by $0.95$ for all $\theta < \frac{\pi}{70}$. Now we can conclude that the matrix \eqref{key_matrix} has spectral norm $0.95+O(1/\sqrt{n})$ with high probability and the proof follows from \eqref{ineq:bound_spectral_norm}.
\end{proof}

\begin{lem}
\label{lem:norm_expected_matrix}
Let $\tau = \max(\|\bm{\beta}^*\|,\|\bm{\beta}\|)$, and $\theta$ be the angle between $\bm{\beta}$ and $\bm{\beta}^*$. Suppose that $\theta\leq \frac{\pi}{8}$ and the orthonormal basis $\left\{\bm{v}_i\right\}_{i=1}^{d}$ satisfy $\bm{v}_1 = \hat{\bm{\beta}^*}$ and span$(\bm{v}_1,\bm{v}_2)=$span$(\bm{\beta},\bm{\beta}^*)$, the following inequality holds:

\begin{align*}
 & \Bigg\|\left[\mathbb{E}_{X\sim \mathcal{N}(0,I)}XX^{\top}\left(1-\int_{0}^{1}\mathbb{E}_{y\sim \mathcal{N}(\langle X,Z(t)\rangle,1)}\Delta(t))dt\right)\right]_{:2,:2}\Bigg\| \\
\leq & \frac{1}{2}\left(1+\frac{1}{(1+0.5\tau^{2})^{2}}\right)+\frac{1}{\pi}\max\left(1-\frac{1}{(1+0.5\tau^{2})^{2}},\sin(\theta)\right)+C\sin(\theta).
\end{align*}
for some absolute constant $0<C\leq 2.6$.
\end{lem}

\begin{proof}
We first provide an elementary bound for $1-\int_{0}^{1}\mathbb{E}_{y\sim \mathcal{N}(\langle X,Z(t)\rangle,1)}\Delta(t))dt)$ in Lemma \ref{lem:2}. Using symmetry and Lemma \ref{lem:2}, the following holds for the $2$ by $2$ submatrix of $\mathbb{E}_{X\sim \mathcal{N}(0,I)}XX^{\top}(1-\int_{0}^{1}\mathbb{E}_{y\sim \mathcal{N}(\langle X,Z(t)\rangle,1)}\Delta(t))dt)$:

\begin{align*}
 & \left[\mathbb{E}_{X\sim \mathcal{N}(0,I)}XX^{\top}\left(1-\int_{0}^{1}\mathbb{E}_{y\sim \mathcal{N}(\langle X,Z(t)\rangle,1)}\Delta(t))dt\right)\right]_{:2,:2}\\
\preceq & \left[2\mathbb{E}_{X\sim \mathcal{N}(0,I)}1_{\langle X,\bm{\beta}\rangle>0,\langle X,\bm{\beta}^*\rangle>0}XX^{\top}\exp\left(-\frac{\min(\langle X,\bm{\beta}\rangle,\langle X,\bm{\beta}^*\rangle)^{2}}{2}\right)+2.25\mathbb{E}_{X\sim \mathcal{N}(0,I)}1_{\langle X,\bm{\beta}\rangle\langle X,\bm{\beta}^*\rangle<0}XX^{\top}\right]_{:2,:2}.
\end{align*}
Therefore,
\begin{align*}
& \Bigg\|\left[\mathbb{E}_{X\sim \mathcal{N}(0,I)}XX^{\top}(1-\int_{0}^{1}\mathbb{E}_{y\sim \mathcal{N}(\langle X,Z(t)\rangle,1)}\Delta(t))dt)\right]_{:2,:2}\Bigg\|\\
\leq & 2 \Bigg\|\underbrace{\left[\mathbb{E}_{X\sim \mathcal{N}(0,I)}1_{\langle X,\bm{\beta}\rangle>0,\langle X,\bm{\beta}^*\rangle>0}XX^{\top}\exp\left(-\frac{\min(\langle X,\bm{\beta}\rangle,\langle X,\bm{\beta}^*\rangle)^{2}}{2}\right)\right]_{:2,:2}}_{M_1}\Bigg\| \\
& + 2.25 \Big\|\underbrace{\left[\mathbb{E}_{X\sim \mathcal{N}(0,I)}1_{\langle X,\bm{\beta}\rangle\langle X,\bm{\beta}^*\rangle<0}XX^{\top}\right]_{:2,:2}}_{M_2}\Big\|.   
\end{align*}
We use the polar coordinates $(r,\phi)$ for the first two components: $X_1,X_2$.
\begin{align*}
(X_{1},X_{2}) & =(r\cos\phi,r\sin\phi)\\
\langle X,\bm{\beta}\rangle & =\|\bm{\beta}\|X_{1}\cos(\phi)+\|\bm{\beta}\|X_{2}\sin(\phi)\\
 & =\|\bm{\beta}\|r\cos(\phi-\theta)\\
\langle X,\bm{\beta}^*\rangle & =\|\bm{\beta}^*\|X_{1}=\|\bm{\beta}^*\|r\cos(\phi).
\end{align*}
It is seen that the region $S_{1}:=\{X:\langle X,\bm{\beta}\rangle>0,\langle X,\bm{\beta}^*\rangle>0\}$
corresponds to $S_{1}=\{(r,\phi):r>0,\phi\in(-\frac{\pi}{2}+\theta,\frac{\pi}{2})\}$ using the polar coordinates. Similarly, the region $S_{2}:=\{X:\langle X,\bm{\beta}\rangle\langle X,\bm{\beta}^*\rangle<0\}$
corresponds to $S_{1}=\{(r,\phi):r>0,\phi\in(-\frac{\pi}{2},-\frac{\pi}{2}+\theta)\cup (\frac{\pi}{2},\frac{\pi}{2}+\theta)\}$ using the polar coordinates. This helps us to get an explicit formula for each of the entry in $M_1$ and $M_2$. Before providing bounds for each entry, we use the following relation for $\min(\langle X,\bm{\beta}\rangle,\langle X,\bm{\beta}^*\rangle)$:
\begin{gather}
\cos(\phi-\theta)\geq\cos(\phi) \quad\phi\in \left(\frac{\theta}{2},\frac{\pi}{2}\right),\nonumber \\
\cos(\phi-\theta)\leq\cos(\phi)  \quad\phi\in \left(-\frac{\pi}{2}+\theta,\frac{\theta}{2}\right).
\end{gather}

Therefore,
\begin{gather}
\label{ineq:compare_min}
  \min(\langle X,\bm{\beta}\rangle,\langle X,\bm{\beta}^*\rangle)\in  
  \begin{cases}
 (\tau r\cos(\theta), \tau r\cos(\phi-\theta) &  \quad\phi\in\left(\frac{\theta}{2},\frac{\pi}{2}\right) \\
 (\tau r\cos(\phi-\theta),\tau r\cos(\theta)) & 
 \quad \phi\in \left(-\frac{\pi}{2}+\theta,\frac{\theta}{2}\right).
 \end{cases}
\end{gather}
 In Lemma \ref{lemma:m1} and Lemma \ref{lemma:m2}, we show that:
 \begin{itemize}
    \item $(1,1)$th entry of $M_1$: $ 0<M_1^{1,1}\leq \frac{1}{4}\left(1+\frac{1}{(1+0.5\tau^{2})^{2}}\right)+\frac{\sin(\theta)}{2\pi}$,
    \item $(2,2)$th entry of $M_1$: $ 0<M_1^{2,2} \leq
    \frac{1}{4}\left(1+\frac{1}{(1+0.5\tau^{2})}\right)+\frac{1}{2\pi}\left(1-\frac{1}{(1+0.5\tau^{2})^{2}}\right)$,
    \item $(1,2)$th entry of $M_1$: $|M_1^{1,2}| \leq
    \max\left(\sin^{2}(\phi)\left[\frac{1}{2\pi(1+\cos^{2}(\theta)\tau^{2})}+\frac{1}{2\pi(1+\tau^{2}\sin^{2}(\theta))}\right],\frac{1}{\pi}\sin(\theta)+\frac{1}{2\pi}\frac{\sin^2(\theta)}{1+\tau^2\sin^2(\theta)}\right)$, 
    \item $(1,1)$th entry of $M_2$: $M_2^{1,1}=\frac{\theta}{\pi}-\frac{\sin(2\theta)}{2\pi}$,
    \item $(2,2)$th entry of $M_2$: $M_2^{2,2}=\frac{\theta}{\pi}+\frac{\sin(2\theta)}{2\pi}$,
    \item $(1,2)$th entry of $M_2$:
    $M_2^{1,2} = -\frac{\sin^2(\theta)}{\pi}$.
 \end{itemize}
 Now we can apply Lemma \ref{lem:spectral_22matrix} to bound the spectral norm of $M_1$ and $M_2$:
 \begin{align*}
     \|M_1\|  \leq & \frac{1}{4}\left(1+\frac{1}{(1+0.5\tau^{2})^{2}}\right) + \frac{1}{2\pi}\max\left(\sin(\theta),1-\frac{1}{(1+0.5\tau^{2})^{2}}\right) + |M_1^{1,2}|\\
     \leq &\frac{1}{4}\left(1+\frac{1}{(1+0.5\tau^{2})^{2}}\right) + \frac{1}{2\pi}\max\left(\sin(\theta),1-\frac{1}{(1+0.5\tau^{2})^{2}}\right)+C_1 \sin(\theta)
 \end{align*}
 for some absolute constant $0<C_1<0.4$. Similarly,
 \begin{align*}
     \|M_2\|\leq \frac{\theta}{\pi}+\frac{\sin(2\theta)}{2\pi}+\frac{\sin^2(\theta)}{\pi}\leq C_2\sin(\theta).
 \end{align*}

for some absolute constant $0<C_2<0.8$. In the last step, we use the fact that when $\theta \in (0,\frac{\pi}{8})$, $\theta\leq 1.1 \sin(\theta)$. We thus obtain a compact bound for the spectral norm of the $2\times 2$ matrix, $\bigg \|\left[\mathbb{E}_{X\sim \mathcal{N}(0,I)}XX^{\top}(1-\int_{0}^{1}\mathbb{E}_{y\sim \mathcal{N}(\langle X,Z(t)\rangle,1)}\Delta(t))dt)\right]_{2\times2}\bigg\|$:
 \begin{align}
 & 2\|M_1\|+2.25\|M_2\| \leq \frac{1}{2}\left(1+\frac{1}{(1+0.5\tau^{2})^{2}}\right) + \frac{1}{\pi}\max\left(\sin(\theta),1-\frac{1}{(1+0.5\tau^{2})^{2}}\right)+ C_3\sin(\theta)
 \end{align}
for some absolute constant $0<C_3\leq 2.6$.
\end{proof}

\begin{lem}
\label{lem:2}When $\langle X,\bm{\beta}\rangle>0$,$\langle X,\bm{\beta}^*\rangle>0$,
\[
\int_{t=0}^{1}[1-\mathbb{E}_{y\sim \mathcal{N}(\langle X,Z(t)\rangle,1)}\Delta(t)]dt\leq\exp\left(-\frac{\min(\langle X,\bm{\beta}\rangle,\langle X,\bm{\beta}^*\rangle)^{2})}{2}\right).
\]
When $\langle X,\bm{\beta}\rangle<0,\langle X,\bm{\beta}^*\rangle<0$,
\[
\int_{t=0}^{1}[1-\mathbb{E}_{y\sim \mathcal{N}(\langle X,Z(t)\rangle,1)}\Delta(t)]dt\leq\exp\left(-\frac{\min(-\langle X,\bm{\beta}\rangle,-\langle X,\bm{\beta}^*\rangle)^{2}}{2}\right).
\]
When
$\langle X,\bm{\beta}\rangle$ and $\langle X,\bm{\beta}^*\rangle$ have different
sign, 
\[\int_{t=0}^{1}[1-\mathbb{E}_{y\sim \mathcal{N}(\langle X,Z(t)\rangle,1)}\Delta(t)]dt\leq 2.25.
\]
\end{lem}

\begin{proof}
We first show the bound for $\langle X,\bm{\beta}\rangle>0$, $\langle X,\bm{\beta}^*\rangle>0$,
and the bound for $\langle X,\bm{\beta}\rangle<0$, $\langle X,\bm{\beta}^*\rangle<0$
can be proved in the same way. 

\begin{align}
 & \int_{t=0}^{1}\left[1-\mathbb{E}_{y\sim \mathcal{N}(\langle X,Z(t)\rangle,1)}\Delta(t)\right]dt \nonumber \\
= & \int_{t=0}^{1}\mathbb{E}_{y\sim \mathcal{N}(\langle X,Z(t)\rangle,1)}\left[1-\tanh(y\langle X,\bm{\beta}\rangle)\right] dt - \int_{t=0}^{1} \mathbb{E}_{y\sim \mathcal{N}(\langle X,Z(t)\rangle,1)}y\langle X,\bm{\beta}\rangle\tanh'(y\langle X,\bm{\beta}\rangle) dt \label{ineq:two_parts}\\
\leq & \int_{t=0}^{1}\exp\left(-\frac{Z(t)\min(Z(t),\langle X,\bm{\beta}\rangle)}{2}\right)dt \label{ineq:ten_step}\\
\leq & \int_{t=0}^{1}\exp\left(-\frac{\min(\langle X,\bm{\beta}\rangle,\langle X,\bm{\beta}^*\rangle)^{2}}{2}\right)dt \label{ineq:same_sign}\\
= & \exp\left(-\frac{\min(\langle X,\bm{\beta}\rangle,\langle X,\bm{\beta}^*\rangle)^{2}}{2}\right) \nonumber.
\end{align}

Inequality \eqref{ineq:ten_step} follows since the second summand in \eqref{ineq:two_parts} is non-negative (c.f Lemma \ref{lem:lemma1tensteps})
and inequality \eqref{ineq:same_sign} follows from Lemma \ref{lem:lemma2tensteps},
with the condition $\langle X,\bm{\beta}\rangle Z(t)\geq0$ satisfied.
To establish the bound for $\langle X,\bm{\beta}^*\rangle\langle X,\bm{\beta}\rangle<0$,
we again use the following numerical inequality:

\[
|\tanh(z)+z\tanh'(z)|\leq 1.25.
\]

Therefore,
\[
\int_{t=0}^{1}[1-\mathbb{E}_{y\sim \mathcal{N}(\langle X,Z(t)\rangle,1)}\Delta(t)]dt\leq 2.25
\]
\end{proof}

\begin{lem}[$M_1$]
\label{lemma:m1}
Let $\tau=\min(\|\bm{\beta}\|,\|\bm{\beta}^*\|)$, and let $\theta$ be the angle between $\bm{\beta}$ and $\bm{\beta}^*$, the following bounds hold for each entry of the symmetric $2\times 2$ matrix $M_1$:
\begin{align*}
0< & M_1^{1,1}\leq  \frac{1}{4}
\left(1+\frac{1}{(1+0.5\tau^{2})^{2}}\right)+\frac{\sin(\theta)}{2\pi}, \\
0 < & M_1^{2,2}\leq 
\frac{1}{4}\left(1+\frac{1}{(1+0.5\tau^{2})}\right)+\frac{1}{2\pi}\left(1-\frac{1}{(1+0.5\tau^{2})^{2}}\right),\\
 & M_1^{1,2}\leq 
\sin^{2}(\theta)\left[\frac{1}{2\pi(1+\cos^{2}(\theta)\tau^{2})}+\frac{\tau^{2}}{(1+\tau^{2}\sin^{2}(\theta))(1+\tau^{2})}\right],\\
& -M_1^{1,2}\leq \frac{1}{\pi}\sin(\theta)+\frac{1}{2\pi}\frac{\sin^2(\theta)}{1+\tau^2\sin^2(\theta)}.
\end{align*}
\end{lem}

\begin{proof}
    We need to go through a very careful integration. $M_1^{1,1}$ and $M_1^{2,2}$ are clearly non-negative.
    \begin{align}
 & M_1^{1,1}=\left[[\mathbb{E}_{X\sim \mathcal{N}(0,I)}1_{\langle X,\bm{\beta}\rangle>0,\langle X,\bm{\beta}^*\rangle>0}XX^{\top}\exp\left(-\frac{\min(\langle X,\bm{\beta}\rangle,\langle X,\bm{\beta}^*\rangle)^{2}}{2}\right)\right]_{11} \nonumber\\
 & =\int_{r=0}^{\infty}\int_{\phi=-\frac{\pi}{2}+\theta}^{\frac{\pi}{2}}r^{2}\cos^{2}(\phi)\exp\left(-\frac{\min(\|\bm{\beta}\|r\cos(\phi-\theta),\|\bm{\beta}^*\|r\cos(\phi))^{2}}{2}\right)\frac{1}{2\pi}\exp\left(-\frac{r^{2}}{2}\right)r dr d\phi \nonumber\\
 & \leq\int_{r=0}^{\infty}\int_{\phi=-\frac{\pi}{2}+\theta}^{\frac{\theta}{2}}r^{2}\cos^{2}(\phi)\exp\left(-\frac{r^{2}\cos^{2}(\phi-\theta)\min(\|\bm{\beta}\|,\|\bm{\beta}^*\|)^{2}}{2}\right)\frac{1}{2\pi}\exp\left(-\frac{r^{2}}{2}\right)r dr d\phi \nonumber\\
 & +\int_{r=0}^{\infty}\int_{\phi=\frac{\theta}{2}}^{\frac{\pi}{2}}r^{2}\cos^{2}(\phi)\exp\left(-\frac{r^{2}\cos^{2}(\phi)\min(\|\bm{\beta}\|,\|\bm{\beta}^*\|)^{2}}{2}\right)\frac{1}{2\pi}\exp\left(-\frac{r^{2}}{2}\right)r dr d\phi \label{ineq:m1_11_step1}\\
 & =\int_{\phi=-\frac{\pi}{2}+\theta}^{\frac{\theta}{2}}\frac{1}{2\pi}\cos^{2}(\phi)\frac{2}{(1+\cos^{2}(\phi-\theta)\min(\|\bm{\beta}\|,\|\bm{\beta}^*\|)^{2})^{2}}d\phi \nonumber \\
 & +\int_{\phi=\frac{\theta}{2}}^{\frac{\pi}{2}}\frac{1}{2\pi}\cos^{2}(\phi)\frac{2}{(1+\cos^{2}(\phi)\min(\|\bm{\beta}\|,\|\bm{\beta}^*\|)^{2})^{2}}d\phi \label{ineq:m1_11_step2}\\
 & =\int_{\phi=\frac{\theta}{2}}^{\frac{\pi}{2}}\frac{1}{2\pi}[\cos^{2}(\phi-\theta)+\cos^{2}(\phi)]\frac{2}{(1+\cos^{2}(\phi)\min(\|\bm{\beta}\|,\|\bm{\beta}^*\|)^{2})^{2}}d\phi \label{eq:m1_11_step3},
\end{align}

where step \eqref{ineq:m1_11_step1} follows from the bound on $\min(\langle X,\bm{\beta}\rangle,\langle X,\bm{\beta}^*\rangle)$ in \eqref{ineq:compare_min}, and step \eqref{ineq:m1_11_step2} holds by integrating over $r$. Finally, step \eqref{eq:m1_11_step3} holds by change of variable. In a similar fashion, we can bound $M_1^{2,2}$. 
\begin{align}
 & M_1^{2,2}=\left[\mathbb{E}_{X\sim \mathcal{N}(0,I)}1_{\langle X,\bm{\beta}\rangle>0,\langle X,\bm{\beta}^*\rangle>0}XX^{\top}\exp\left(-\frac{\min(\langle X,\bm{\beta}\rangle,\langle X,\bm{\beta}^*\rangle)^{2}}{2}\right)\right]_{22} \nonumber \\
 & =\int_{r=0}^{\infty}\int_{\phi=-\frac{\pi}{2}+\theta}^{\frac{\pi}{2}}r^{2}\sin^{2}(\phi)\exp\left(-\frac{\min(\|\bm{\beta}\|r\cos(\phi-\theta),\|\bm{\beta}^*\|r\cos(\phi))^{2}}{2}\right)\frac{1}{2\pi}\exp\left(-\frac{r^{2}}{2}\right)rdrd\phi \nonumber \\
 & \leq\int_{r=0}^{\infty}\int_{\phi=-\frac{\pi}{2}+\theta}^{\frac{\theta}{2}}r^{2}\sin^{2}(\phi)\exp\left(-\frac{r^{2}\cos^{2}(\phi)\min(\|\bm{\beta}\|,\|\bm{\beta}^*\|)^{2}}{2}\right)\frac{1}{2\pi}\exp\left(-\frac{r^{2}}{2}\right)rdrd\phi \nonumber\\
 & +\int_{r=0}^{\infty}\int_{\phi=-\frac{\theta}{2}}^{\frac{\pi}{2}}r^{2}\sin^{2}(\phi)\exp\left(-\frac{r^{2}\cos^{2}(\phi-\theta)\min(\|\bm{\beta}\|,\|\bm{\beta}^*\|)^{2}}{2}\right)\frac{1}{2\pi}\exp\left(-\frac{r^{2}}{2}\right)rdrd\phi \nonumber\\
 & =\int_{\phi=\frac{\theta}{2}}^{\frac{\pi}{2}}\frac{1}{2\pi}[\sin^{2}(\phi-\theta)+\sin^{2}(\phi)]\frac{2}{(1+\cos^{2}(\phi)\min(\|\bm{\beta}\|,\|\bm{\beta}^*\|)^{2})^{2}}d\phi \label{eq:m1_22_step}.
\end{align}
Finally, to obtain a bound for $|M_1^{1,2}|$, we upper bound both $M_1^{1,2}$ and $-M_1^{1,2}$:
\begin{align}
 & M_1^{1,2}=\left[\mathbb{E}_{X\sim \mathcal{N}(0,I)}1_{\langle X,\bm{\beta}\rangle>0,\langle X,\bm{\beta}^*\rangle>0}XX^{\top}\exp\left(-\frac{\min(\langle X,\bm{\beta}\rangle,\langle X,\bm{\beta}^*\rangle)^{2}}{2}\right)\right]_{12} \nonumber\\
 & =\int_{r=0}^{\infty}\int_{\phi=-\frac{\pi}{2}+\theta}^{\frac{\pi}{2}}r^{2}\sin(\phi)\cos(\phi)\exp\left(-\frac{\min(\|\bm{\beta}\|r\cos(\phi-\theta),\|\bm{\beta}^*\|r\cos(\phi))^{2}}{2}\right)\frac{1}{2\pi}\exp\left(-\frac{r^{2}}{2}\right)rdrd\phi \nonumber \\
 & \leq\int_{r=0}^{\infty}\int_{\phi=\frac{\theta}{2}}^{\frac{\pi}{2}}r^{2}\sin(\phi)\cos(\phi)\exp\left(-\frac{r^{2}\cos^{2}(\phi)\min(\|\bm{\beta}\|,\|\bm{\beta}^*\|)^{2}}{2}\right)\frac{1}{2\pi}\exp\left(-\frac{r^{2}}{2}\right)rdrd\phi\\
 & +\int_{r=0}^{\infty}\int_{\phi=0}^{\frac{\theta}{2}}r^{2}\sin(\phi)\cos(\phi)\exp\left(-\frac{r^{2}\cos^{2}(\phi-\theta)\min(\|\bm{\beta}\|,\|\bm{\beta}^*\|)^{2}}{2}\right)\frac{1}{2\pi}\exp\left(-\frac{r^{2}}{2}\right)rdrd\phi \nonumber \\
 & +\int_{r=0}^{\infty}\int_{\phi=-\frac{\pi}{2}+\theta}^{0}r^{2}\sin(\phi)\cos(\phi)\exp\left(-\frac{r^{2}\cos^{2}(\phi)\min(\|\bm{\beta}\|,\|\bm{\beta}^*\|)^{2}}{2}\right)\frac{1}{2\pi}\exp\left(-\frac{r^{2}}{2}\right)rdrd\phi \nonumber \\
 & =\int_{\phi\in(-\frac{\pi}{2}+\theta,0)\cup(\frac{\theta}{2},\frac{\pi}{2})}\frac{1}{2\pi}\frac{2}{(1+\cos^{2}(\phi)\min(\|\bm{\beta}\|,\|\bm{\beta}^*\|)^{2})^{2}}\sin(\phi)\cos(\phi)d\phi  \nonumber\\
 & +\int_{\phi=0}^{\frac{\theta}{2}}\frac{1}{2\pi}\frac{2}{(1+\cos^{2}(\phi-\theta)\min(\|\bm{\beta}\|,\|\bm{\beta}^*\|)^{2})^{2}}\sin(\phi)\cos(\phi)d\phi \label{eq:m1_12_step1}.
\end{align}

Note that in the above bound, the sign of $\sin(\phi)\cos(\phi)$ differs between region $(-\frac{\pi}{2},0)$ and $(0,\frac{\pi}{2})$. Similarly, for $-M_1^{1,2}$, we have
\begin{align}
 & - M_1^{1,2}\leq  =\int_{\phi\in(-\frac{\pi}{2}+\theta,0)\cup(\frac{\theta}{2},\frac{\pi}{2})}-\frac{1}{2\pi}\frac{2}{(1+\cos^{2}(\phi-\theta)\min(\|\bm{\beta}\|,\|\bm{\beta}^*\|)^{2})^{2}}\sin(\phi)\cos(\phi)d\phi \nonumber\\
 & -\int_{\phi=0}^{\frac{\theta}{2}}\frac{1}{2\pi}\frac{2}{(1+\cos^{2}(\theta)\min(\|\bm{\beta}\|,\|\bm{\beta}^*\|)^{2})^{2}}\sin(\phi)\cos(\phi)d\phi \label{eq:m1_12_step2}.  
\end{align}
The next step is to provide upper bounds for those integrals, \eqref{eq:m1_11_step3},\eqref{eq:m1_22_step}, \eqref{eq:m1_12_step1} and \eqref{eq:m1_12_step2}. The final bounds for $M_1^{1,1}$, $M_1^{2,2}$ and $M_1^{1,2}$ are established in Lemma \ref{lem:m1_11}, Lemma \ref{lem:m1_22} and Lemma \ref{lem:m1_12} respectively.
\end{proof}

\begin{lem}[$M_1$, $(1,1)^{\text{th}}$ entry]
\label{lem:m1_11}
Let $\tau=\min(\|\bm{\beta}\|,\|\bm{\beta}^*\|)$. Suppose $\theta\leq\frac{\pi}{8}$, the following holds: 
\begin{align*}
  \int_{\phi=\frac{\theta}{2}}^{\frac{\pi}{2}}\cos^{2}(\phi-\theta)\frac{1}{(1+\cos^{2}(\phi)\tau^{2})^{2}}d\phi
\leq & \frac{\sin(\theta)}{2}  +\frac{\pi}{8}\left(1+\frac{1}{(1+0.5\tau^{2})^{2}}\right),\\
  \int_{\phi=\frac{\theta}{2}}^{\frac{\pi}{2}}\cos^{2}(\phi)\frac{1}{(1+\cos^{2}(\phi)\tau^{2})^{2}}d\phi
\leq & \frac{\pi}{8}\left(1+\frac{1}{(1+0.5\tau^{2})^{2}}\right).
\end{align*}

Hence, $M_1^{1,1}\leq\frac{1}{4}\left(1+\frac{1}{(1+0.5\tau^{2})^{2}}\right)+\frac{\sin(\theta)}{2\pi}.$
\end{lem}

\begin{proof}
We divide the region $(\frac{\theta}{2},\frac{\pi}{2})$
into two parts, $(\frac{\theta}{2},\gamma)\cup(\gamma,\frac{\pi}{2})$ for some $\gamma>\frac{\theta}{2}$.
In the first part, we bound $\frac{1}{(1+\cos^{2}(\phi)\tau^{2})^{2}}$
by $\frac{1}{(1+\cos^{2}(\gamma)\tau^{2})^{2}}$, and in the second
part, we bound $\frac{1}{(1+\cos^{2}(\phi)\tau^{2})^{2}}$ by $1.$
It then follows that:

\begin{align*}
 & \int_{\phi=\frac{\theta}{2}}^{\frac{\pi}{2}}\cos^{2}(\phi-\theta)\frac{1}{(1+\cos^{2}(\phi)\tau^{2})^{2}}d\phi\\
\leq & \int_{\phi=\frac{\theta}{2}}^{\gamma}\cos^{2}(\phi-\theta)\frac{1}{(1+\cos^{2}(\gamma)\tau^{2})^{2}}d\phi+\int_{\phi=\gamma}^{\frac{\pi}{2}}\cos^{2}(\phi-\theta)d\phi\\
= & \frac{1}{(1+\cos^{2}(\gamma)\tau^{2})^{2}}\left[\frac{1}{2}(\gamma-\frac{\theta}{2})+\frac{1}{4}(\sin(2\gamma-2\theta)+\sin(\theta))\right]+\left[\frac{1}{2}(\frac{\pi}{2}-\gamma)+\frac{1}{4}(\sin(\pi-2\theta)-\sin(2\gamma-2\theta))\right]\\
\leq & \frac{1}{(1+\cos^{2}(\gamma)\tau^{2})^{2}}\frac{1}{2}\gamma+\frac{1}{2}(\frac{\pi}{2}-\gamma)+\frac{1}{4}\sin(2\theta)+\frac{1}{4}\sin(\theta)\\
\end{align*}

The last step holds since $\sin(2\gamma-2\theta)>0$. By picking $\gamma=\frac{\pi}{4}$,
the above bound becomes:

\[
\frac{\pi}{8}\left(1+\frac{1}{(1+0.5\tau^{2})^{2}})+\frac{1}{4}\sin(2\theta\right)+\frac{1}{4}\sin(\theta).
\]

Similarly,

\begin{align*}
 & \int_{\phi=\frac{\theta}{2}}^{\frac{\pi}{2}}\cos^{2}(\phi)\frac{1}{(1+\cos^{2}(\phi)\tau^{2})^{2}}d\phi\\
\leq & \int_{\phi=\frac{\theta}{2}}^{\gamma}\frac{1}{(1+\cos^{2}(\gamma)\tau^{2})^{2}}\cos^{2}(\phi)d\phi+\int_{\phi=\gamma}^{\frac{\pi}{2}}\cos^{2}(\phi)d\phi\\
= & \frac{1}{(1+\cos^{2}(\gamma)\tau^{2})^{2}}\left[\frac{1}{2}(\gamma-\frac{\phi}{2})+\frac{1}{4}(\sin(2\gamma)-\sin(\phi))\right]+\left[\frac{1}{2}(\frac{\pi}{2}-\gamma)+\frac{1}{4}(\sin(\pi)-\sin(2\gamma))\right]\\
\leq & \frac{\pi}{8}\left(1+\frac{1}{(1+0.5\tau^{2})^{2}}\right)-\frac{1}{4}\sin(\phi),
\end{align*}

where we again pick $\gamma=\frac{\pi}{4}$ in the last step. 

Therefore, by adding the above two bounds together, we show that 
\begin{align*}
M_1^{1,1} & =\frac{1}{\pi}\int_{\phi=\frac{\theta}{2}}^{\frac{\pi}{2}}[\cos^{2}(\phi-\theta)+\cos^{2}(\phi)]\frac{1}{(1+\cos^{2}(\phi)\min(\|\bm{\beta}\|,\|\bm{\beta}^*\|)^{2})^{2}}d\phi\\
 & \leq \frac{1}{4}\left(1+\frac{1}{(1+0.5\tau^{2})^{2}}\right)+\frac{1}{2\pi}\sin(\theta).
\end{align*}    
\end{proof}

\begin{lem}[$M_1$, $(2,2)^{\text{th}}$ entry]
\label{lem:m1_22}
Let $\tau=\min(\|\bm{\beta}\|,\|\bm{\beta}^*\|)$. Suppose $\theta \leq\frac{\pi}{8}$, the following holds: 
\begin{align*}
 & \int_{\phi=\frac{\theta}{2}}^{\frac{\pi}{2}}\sin^{2}(\phi-\theta)\frac{1}{(1+\cos^{2}(\phi)\tau^{2})^{2}}d\phi\\
\leq & \frac{\pi}{8}\left(1+\frac{1}{(1+0.5\tau^{2})^{2}}\right)+\frac{1}{4}\left(1-\frac{1}{(1+0.5\tau^{2})^{2}}\right)\cos(2\theta)-\frac{1}{4}\sin(2\theta)-\frac{1}{4(1+0.5\tau^{2})^{2}}\sin(\theta),\\
 & \int_{\phi=\frac{\theta}{2}}^{\frac{\pi}{2}}\sin^{2}(\phi)\frac{1}{(1+\cos^{2}(\phi)\tau^{2})^{2}}d\phi\\
\leq & \frac{\pi}{8}\left(1+\frac{1}{(1+0.5\tau^{2})^{2}}\right)+\frac{1}{4}\left(1-\frac{1}{(1+0.5\tau^{2})^{2}}\right)+\frac{1}{4}\frac{1}{(1+0.5\tau^{2})^{2}}\sin(\theta).
\end{align*}

Hence, $M_1^{2,2} \leq\frac{1}{4}\left(1+\frac{1}{(1+0.5\tau^{2})}\right)+\frac{1}{2\pi}\left(1-\frac{1}{(1+0.5\tau^{2})^{2}}\right)$.
\end{lem}

\begin{proof}
The method is similar as before where we divide the region $(\frac{\theta}{2},\frac{\pi}{2})$
into two parts, $(\frac{\theta}{2},\gamma)\cup(\gamma,\frac{\pi}{2})$ for some $\gamma>\frac{\theta}{2}$. 

\begin{align*}
 & \int_{\phi=\frac{\theta}{2}}^{\frac{\pi}{2}}\sin^{2}(\phi-\theta)\frac{1}{(1+\cos^{2}(\phi)\tau^{2})^{2}}d\phi\\
\leq & \int_{\phi=\frac{\theta}{2}}^{\gamma}\sin^{2}(\phi-\theta)\frac{1}{(1+\cos^{2}(\gamma)\tau^{2})^{2}}d\phi+\int_{\phi=\gamma}^{\frac{\pi}{2}}\sin^{2}(\phi-\theta)d\phi\\
= & \frac{1}{(1+\cos^{2}(\gamma)\tau^{2})^{2}}\left[\frac{1}{2}(\gamma-\frac{\theta}{2})-\frac{1}{4}(\sin(2\gamma-2\theta)+\sin(\theta))\right]+\frac{1}{2}\left(\frac{\pi}{2}-\gamma\right)-\frac{1}{4}\left[\sin(\pi-2\theta)-\sin(2\gamma-2\theta)\right]\\
= & \frac{\pi}{4}\frac{1}{2}\left(1+\frac{1}{(1+0.5\tau^{2})^{2}}\right)+\frac{1}{4}\left(1-\frac{1}{(1+0.5\tau^{2})^{2}}\right)\cos(2\theta)-\frac{1}{4}\sin(2\theta)-\frac{1}{4(1+0.5\tau^{2})^{2}}\sin(\theta),
\end{align*}

where we pick $\gamma=\frac{\pi}{4}$ in the last step.

Similarly, 

\begin{align*}
 & \int_{\phi=\frac{\theta}{2}}^{\frac{\pi}{2}}\sin^{2}(\phi)\frac{1}{(1+\cos^{2}(\phi)\tau^{2})^{2}}d\phi\\
\leq & \int_{\phi=\frac{\theta}{2}}^{\gamma}\sin^{2}(\theta)\frac{1}{(1+\cos^{2}(\gamma)\tau^{2})^{2}}d\phi+\int_{\phi=\gamma}^{\frac{\pi}{2}}\sin^{2}(\phi)d\phi\\
= & \frac{1}{(1+\cos^{2}(\gamma)\tau^{2})^{2}}\left[\frac{1}{2}(\gamma-\frac{\theta}{2})-\frac{1}{4}(\sin(2\gamma)-\sin(\theta))\right]+\frac{1}{2}(\frac{\pi}{2}-\gamma)-\frac{1}{4}\left[\sin(\pi)-\sin(2\gamma)\right]\\
= & \frac{\pi}{4}\frac{1}{2}\left(1+\frac{1}{(1+0.5\tau^{2})^{2}}\right)+\frac{1}{4}\left(1-\frac{1}{(1+0.5\tau^{2})^{2}}\right)+\frac{1}{4}\frac{1}{(1+0.5\tau^{2})^{2}}\sin(\theta),
\end{align*}

where we again pick $\gamma=\frac{\pi}{4}$ in the last step. Therefore,

\begin{align*}
M_1^{2,2} & =\int_{\phi=\frac{\theta}{2}}^{\frac{\pi}{2}}\frac{1}{2\pi}[\sin^{2}(\phi-\theta)+\sin^{2}(\phi)]\frac{2}{(1+\cos^{2}(\phi)\min(\|\beta\|,\|\beta_{*}\|)^{2})^{2}}d\phi.\\
 & \frac{1}{4}\left(1+\frac{1}{(1+0.5\tau^{2})}\right)+\frac{1}{2\pi}\left(1-\frac{1}{(1+0.5\tau^{2})^{2}}\right).
\end{align*}
\end{proof}

\begin{lem}[$M_1 (1,2)^{\text{th}}$ entry]
\label{lem:m1_12}
Let $\tau=\min(\|\bm{\beta}\|,\|\bm{\beta}^*\|)$. Suppose $\theta\leq\frac{\pi}{8}$, the following holds: 
\begin{align*}
M_{1}^{1,2} \leq & \sin^{2}(\theta)\left[\frac{1}{2\pi(1+\cos^{2}(\theta)\tau^{2})}+\frac{1}{2\pi(1+\tau^{2}\sin^{2}(\theta))}\right],\\
-M_1^{1,2}\leq & \frac{1}{\pi}\sin(\theta)+\frac{1}{2\pi}\frac{\sin^2(\theta)}{1+\tau^2\sin^2(\theta)}.
\end{align*}
\end{lem}

\begin{proof}
In \eqref{eq:m1_12_step1}, we know that 
\begin{align*}
M_1^{1,2} & \leq \int_{\phi\in(-\frac{\pi}{2}+\theta,0)\cup(\frac{\theta}{2},\frac{\pi}{2})}\frac{1}{2\pi}\frac{2}{(1+\cos^{2}(\phi)\tau^{2})^{2}}\sin(\phi)\cos(\phi)d\phi\\
 & +\int_{\phi=0}^{\frac{\theta}{2}}\frac{1}{2\pi}\frac{2}{(1+\cos^{2}(\phi-\theta)\tau^{2})^{2}}\sin(\phi)\cos(\phi)d\phi. \\
 \end{align*}
 It remains to bound the right hand side. 
 \begin{align*}
 & \int_{\phi\in(-\frac{\pi}{2}+\theta,0)\cup(\frac{\theta}{2},\frac{\pi}{2})}\frac{1}{2\pi}\frac{2}{(1+\cos^{2}(\phi)\tau^{2})^{2}}\sin(\phi)\cos(\phi)d\phi
 +\int_{\phi=0}^{\frac{\theta}{2}}\frac{1}{2\pi}\frac{2}{(1+\cos^{2}(\phi-\theta)\tau^{2})^{2}}\sin(\phi)\cos(\phi)d\phi \\
 & \leq\frac{1}{\pi(1+\cos^{2}(\theta)\tau^{2})}\int_{\phi=0}^{\frac{\theta}{2}}\sin(\phi)\cos(\phi)d\phi+\frac{1}{2\pi}\left[\frac{-\cos^{2}(\theta)}{(1+\tau^{2})(1+\tau^{2}\sin^{2}(\theta))}+\frac{\cos^{2}(\theta)}{1+\cos^{2}(\theta)\tau^{2}}\right]\\
 & \leq\frac{1}{\pi(1+\cos^{2}(\theta)\tau^{2})}\int_{\phi=0}^{\frac{\theta}{2}}\sin(\phi)\cos(\phi)d\phi+\frac{1}{2\pi}\left[\frac{-\cos^{2}(\theta)}{(1+\tau^{2})(1+\tau^{2}\sin^{2}(\theta))}+\frac{1}{1+\tau^{2}}\right]\\
 & \leq\frac{\sin^{2}(\theta)}{2\pi(1+\cos^{2}(\theta)\tau^{2})}+\frac{\sin^{2}(\theta)}{2\pi(1+\tau^{2}\sin^{2}(\theta))}\\
 & \leq\sin^{2}(\theta)\left[\frac{1}{2\pi(1+\cos^{2}(\theta)\tau^{2}}+\frac{1}{2\pi(1+\tau^{2}\sin^{2}(\theta))}\right].
\end{align*}
Next, let us look at the bound for $-M_1^{1,2}$ in \eqref{eq:m1_12_step2}. There are two terms, one is 
\[
T_1:=\int_{\phi\in(-\frac{\pi}{2}+\theta,0)\cup(\frac{\theta}{2},\frac{\pi}{2})}-\frac{1}{2\pi}\frac{2}{(1+\cos^{2}(\phi-\theta)\tau^{2})^{2}}\sin(\phi)\cos(\phi)d\phi
\]
and the other is:
\[
 T_2:= -\int_{\phi=0}^{\frac{\theta}{2}}\frac{1}{2\pi}\frac{2}{(1+\cos^{2}(\phi)\tau^{2}}\sin(\phi)\cos(\phi)d\phi
\]
When $\phi\in (0,\pi/8)$, $T_2<0$. For $T_1$, let us use change of variable and write the integral as:
\begin{align*}
& \int_{\phi\in(-\frac{\pi}{2},-\theta)\cup(-\frac{\theta}{2},\frac{\pi}{2}-\theta)}-\frac{1}{2\pi}\frac{2}{(1+\cos^{2}(\phi)\tau^{2}}\sin(\phi+\theta)\cos(\phi+\theta)d\phi\\
= & \underbrace{\int_{\phi\in(-\frac{\pi}{2},-\theta)\cup(-\frac{\theta}{2},\frac{\pi}{2}-\theta)}-\frac{1}{2\pi}\frac{2}{(1+\cos^{2}(\phi)\tau^{2})^{2}}\sin(\phi)\cos(\phi)\cos(2\theta) d\phi}_{\text{Part} 1}\\
& + \underbrace{\int_{\phi\in(-\frac{\pi}{2},-\theta)\cup(-\frac{\theta}{2},\frac{\pi}{2}-\theta)}-\frac{1}{2\pi}\frac{2}{(1+\cos^{2}(\phi)\tau^{2})^{2}}\sin(\theta)\cos(\theta)\cos(2\phi) d\phi}_{\text{Part} 2}.
\end{align*}
Note that the first part can be computed exactly as before,
\begin{align}
  &\int_{\phi\in(-\frac{\pi}{2},-\theta)\cup(-\frac{\theta}{2},\frac{\pi}{2}-\theta)}-\frac{1}{2\pi}\frac{2}{(1+\cos^{2}(\phi)\tau^{2})^{2}}\sin(\phi)\cos(\phi)\cos(2\theta) d\phi \nonumber \\
  \leq & \int_{\phi\in(-\frac{\pi}{2},-\theta)\cup(-\theta,\frac{\pi}{2}-\theta)}-\frac{1}{2\pi}\frac{2}{(1+\cos^{2}(\phi)\tau^{2})^{2}}\sin(\phi)\cos(\phi)\cos(2\theta) d\phi \nonumber \\
  = & \frac{1}{2\pi}\left[\frac{\cos^2(\theta)}{1+\cos^2(\theta)\tau^2}-\frac{\cos^2(\theta)-\sin^2(\theta)}{(1+\tau^2\sin^2(\theta))(1+\tau^2\cos^2(\theta))}\right]\cos(2\theta)\nonumber \\
  \leq & \frac{1}{2\pi}\frac{\sin^2(\theta)}{1+\tau^2\sin^2(\theta)} \label{ineq:part1}.
\end{align}
The second part contains the factor $\sin(\theta)\cos(\theta)$, and it remains to bound:
\[
\int_{\phi\in(-\frac{\pi}{2},-\theta)\cup(-\frac{\theta}{2},\frac{\pi}{2}-\theta)}-\frac{1}{2\pi}\frac{2}{(1+\cos^{2}(\phi)\tau^{2})^{2}}\cos(2\phi) d\phi.
\]
Note that the intergrand is an even function in $\phi$. Moreover, when $|\phi|<\frac{\pi}{4}$, the intergrand is negative, and when $|\phi|\in (\frac{\pi}{4},\frac{\pi}{2})$, the integrand is positive. Thus the integral can be further upper bounded by:
\begin{align*}
& 2\int_{\phi\in(\frac{\pi}{4},\frac{\pi}{2})}-\frac{1}{2\pi}\frac{2}{(1+\cos^{2}(\phi)\tau^{2})^{2}}\cos(2\phi) d\phi\\
\leq & \frac{2}{\pi}\int_{\phi=\frac{\pi}{4}}^{\frac{\pi}{2}} -\cos(2\phi)d\phi\\
= & \frac{1}{\pi}.
\end{align*}
Combining the bound on two parts, we obtain:
\begin{align*}
    -M_1^{1,2}= & T_1+T_2 \\
    \leq & \frac{1}{2\pi}\frac{\sin^2(\theta)}{1+\tau^2\sin^2(\theta)}+\frac{1}{\pi}\sin(\theta).
\end{align*}
\end{proof}

\begin{lem}[$M_2$]
\label{lemma:m2}
The entries of the symmetric $2\times 2$ matrix $M_2$ are the following:
\begin{align*}
M_2^{1,1}= & \frac{\theta}{\pi}-\frac{\sin(2\theta)}{4\pi} \\   
M_2^{2,2}= & \frac{\theta}{\pi}+\frac{\sin(2\theta)}{4\pi} \\
M_2^{1,2}=&  -\frac{\sin^2(\theta)}{\pi}.
\end{align*}
\end{lem}
\begin{proof}
It is a simple calculation. 
\begin{align*}
M_2^{1,1}= & \frac{1}{\pi}\int_{r>0}\int_{\phi\in (-\frac{\pi}{2},-\frac{\pi}{2}+\theta)} r^2 \cos^2(\phi) \exp\left(-\frac{r^2}{2}\right)r dr d\phi \\
= & \frac{2}{\pi}\int_{\phi\in (-\frac{\pi}{2},-\frac{\pi}{2}+\theta)} \cos^2(\phi) d\phi = \frac{\theta}{\pi} - \frac{\sin(2\theta)}{2\pi}.
\end{align*}
Similarly,
\begin{align*}
M_2^{2,2} = & \frac{1}{\pi}\int_{r>0}\int_{\phi\in (-\frac{\pi}{2},-\frac{\pi}{2}+\theta)} r^2 \sin^2(\phi) \exp\left(-\frac{r^2}{2}\right)r dr d\phi  \\
= & \frac{2}{\pi}\int_{\phi\in (-\frac{\pi}{2},-\frac{\pi}{2}+\theta)} \sin^2(\phi) d\phi = \frac{\theta}{\pi} + \frac{\sin(2\theta)}{2\pi}.
\end{align*}
For the cross term,
\begin{align*}
M_2^{1,2} = & \frac{1}{\pi}\int_{r>0}\int_{\phi\in (-\frac{\pi}{2},-\frac{\pi}{2}+\theta)} r^2 \sin(\theta)\cos(\theta) \exp\left(-\frac{r^2}{2}\right)r dr d\phi\\
= & \frac{1}{\pi}\int_{\phi\in (-\frac{\pi}{2},-\frac{\pi}{2}+\theta)} \sin(2\phi) d\phi = \frac{\sin^2(\theta)}{\pi}.
\end{align*}

\end{proof}

\begin{lem}[Bound on Spectral norm of a $2\times 2$ matrix]
\label{lem:spectral_22matrix}
Let $M$ be a symmetric $2\times2$ matrix
\[M=\left[
\begin{matrix}
a & c \\
c & b
\end{matrix} \right].
\]
Suppose $a,b>0$, the spectral norm of $M$ is bounded by $\max(a,b)+|c|$.
\end{lem}

\begin{proof}
The characteristic polynomial for the matrix is:
\[
p(x)=x^2-(a+b)x+ab-c^2.
\]
It has two roots: $x_1 = \frac{a+b+\sqrt{(a-b)^2+4c^2}}{2}$ and $x_2 = \frac{a+b-\sqrt{(a-b)^2+4c^2}}{2}$. When $a,b>0$, the larger root is upper bounded by $\frac{a+b+|a-b|+2|c|}{2}$, which is dominated by $\max(a,b)+|c|$.
\end{proof}

\section{Proofs for Auxiliary Results}
\label{appendix:sec_aux}

\subsection{Upper Bound for Norm}
\begin{restatable}[Bounded population EM iterates] {lem}{lemmabprime}
\label{lemma:bprime_bounded}
	For any $\bm{\beta} \in \mathbb{R}^d$, we have
	\begin{equation}
    	\|\bm{\beta}'\| \le 3\sqrt{\sigma^2 + \|\bm{\beta}^*\|^2}.
    \end{equation}
\end{restatable}

\begin{proof}
    From Lemma \ref{lemma:b_1_prime}, we know that $b_1' \le b_1^* + \frac{2}{\pi} \sqrt{\sigma^2 + {b_2^*}^2}$. On the other side, from lemma \ref{lemma:S_bounds} we have $b_2' \le b_2^*$. 
    Therefore, 
    \begin{align*}
        b_1' &\le b_1^* + \frac{2}{\pi}{\sqrt{\sigma^2 + {b_2^*}^2}} \\
        &\le \|\bm{\beta}^*\| + \frac{2}{\pi}{\sqrt{\sigma^2 + \|\bm{\beta}^*\|^2}} \\ 
        &\le 2{\sqrt{\sigma^2 + \|\bm{\beta}^*\|^2}}, \\
        b_2' &\le \|\bm{\beta}^*\|.
    \end{align*}
    Combining the bound for each, we get $\|\bm{\beta}'\| \le 3\sqrt{\sigma^2 + \|\bm{\beta}^*\|^2}$.
\end{proof}

\subsection{Lower Bound for Norm}
\label{appendix:subsec_aux_lb_norm}

\begin{restatable}{lem}{finitebplowerbound}
\label{lemma:finite_b1p_lb}
If $\|\bm{\beta}\| \ge \|\bm{\beta}^*\| / 10$, then after one finite-sample EM update with $n = O(\max(1, poly(\eta^{-2}))$ $(d/\epsilon^2))$ samples, $\|\tilde{\bm{\beta}}'\| \ge \|\bm{\beta}^*\| / 10$.
\end{restatable}

\begin{proof}
    We divide the cases by varying $\theta$. Note that $n$ is now proportional to $poly(\eta^{-2})$, and we control the number of samples so that statistical error in norm is $\|\tilde{\bm{\beta}}' - \bm{\beta}'\| \le O(\epsilon) \min(1, \eta^2)$. We first show that population EM operator $\|\bm{\beta}'\|$ is larger enough than $\frac{\|\bm{\beta}^*\|}{10}$, therefore $\|\bm{\beta}'\| - \|\tilde{\bm{\beta}}' - \bm{\beta}'\|$ is greater than $\frac{\|\bm{\beta}^*\|}{10}$.
    
    \paragraph{\it $\cos \theta \ge 0.2, \sin \theta \ge 0.2$:} Suppose $\|\bm{\beta}\| \ge \frac{\|\bm{\beta}^*\|}{10}$. If $\cos \theta \ge 0.2$ or $b_1^* \ge \frac{\|\bm{\beta}^*\|}{5}$, then as shown in the proof of Corollary \ref{corollary:distance}, $\|\bm{\beta}'\| \ge \min(\frac{\sigma_2^2}{\sigma^2} b_1, b_1^*) \ge \min((1+ \eta^2 \sin^2 \theta ) \frac{\|\bm{\beta}^*\|}{10}, 0.2 \|\bm{\beta}^*\|)$. We take small enough $\epsilon$, we have $\|\tilde{\bm{\beta}}'\| \ge \|\bm{\beta}'\| - \epsilon \ge \frac{\|\bm{\beta}^*\|}{10}$.
    
    \paragraph{\it $\cos \theta \le 0.2$:}
    Recall that $\|\bm{\beta}'\| \ge b_1' = \mathbb{E} [\tanh( \frac{b_1 \alpha_1}{\sigma^2} (\alpha_1 b_1^* + y) ) (\alpha_1 b_1^* + y) \alpha_1]$, where $\alpha_1 \sim \mathcal{N}(0, 1)$, $y \sim \mathcal{N}(0, \sigma_2^2)$. We first claim that $b_1' \ge \mathbb{E} [\tanh \left( \frac{b_1}{\sigma^2} \alpha_1y \right) \alpha_1y]$, \textit{i.e.}, lower bounded by setting $b_1^* = 0$. In order to show that, we differentiate $b_1'$ with respect to $b_1^*$, which yields
    \begin{align*}
        \mathbb{E} [\alpha_1^2 \tanh(\frac{b_1 \alpha_1}{\sigma^2} (\alpha_1 b_1^* + y))] + \mathbb{E} [\frac{\alpha_1^3 b_1}{\sigma^2} (\alpha_1 b_1^* + y) \tanh'(\frac{b_1 \alpha_1}{\sigma^2} (\alpha_1 b_1^* + y))].
    \end{align*}
    However, 
    \begin{align*}
        \mathbb{E} [\alpha_1^2 \tanh(\frac{b_1 \alpha_1}{\sigma^2} & (\alpha_1 b_1^* + y))] = \\& \frac{1}{\pi \sigma_2} \int_{0}^{\infty} \alpha_1^2 e^{-\alpha_1^2 /2}  \int_{0}^{\infty} \tanh(\frac{b_1 \alpha_1}{\sigma^2} y) (e^{-\frac{(y-\alpha_1b_1^*)^2} {2\sigma_2^2}} - e^{-\frac{(y+\alpha_1b_1^*)^2} {2\sigma_2^2}}) dy d\alpha_1
        \ge 0.
    \end{align*}
    Simiarly, 
    \begin{align*}
        \mathbb{E} [\frac{\alpha_1^3 b_1}{\sigma^2} (\alpha_1 b_1^* + y) \tanh'(\frac{b_1 \alpha_1}{\sigma^2} (\alpha_1 b_1^* + y))] = & \\ 
        \frac{1}{\pi \sigma_2} \int_{0}^{\infty} \frac{\alpha_1^3 b_1}{\sigma^2} e^{-\alpha_1^2 /2} \int_{0}^{\infty} y \tanh'(\frac{b_1 \alpha_1}{\sigma^2} y) & (e^{-\frac{(y-\alpha_1b_1^*)^2} {2\sigma_2^2}} - e^{-\frac{(y+\alpha_1b_1^*)^2} {2\sigma_2^2}}) dy d\alpha_1 \ge 0.
    \end{align*}
    Now it becomes clear that $b_1'$ is increasing in $b_1^*$, thus the claim is verified.
    
    Next, we bound $\mathbb{E} [\tanh( \frac{b_1}{\sigma^2} \alpha_1y) \alpha_1y]$. 
    \begin{align*}
        \mathbb{E} [\tanh( \frac{b_1}{\sigma^2} \alpha_1y) \alpha_1y] &= \frac{2}{\pi \sigma_2} \int_{0}^{\infty} \int_{0}^{\infty} \alpha_1 y \tanh( \frac{b_1}{\sigma^2} \alpha_1y) e^{-\frac{y^2}{2\sigma_2^2}} e^{-\frac{\alpha_1^2}{2}} d\alpha_1 dy \\
        &= \frac{2}{\pi} \sigma_2 \int_{0}^{\infty} \int_{0}^{\infty} \alpha_1 y \tanh( \frac{b_1}{\sigma^2} \sigma_2 \alpha_1y) e^{-\frac{y^2}{2}} e^{-\frac{\alpha_1^2}{2}} d\alpha_1 dy \\
        &\ge \frac{2}{\pi} \sigma_2 \int_{0}^{\infty} \int_{0}^{\infty} \alpha_1 y \tanh( \frac{b_1}{\sigma} \alpha_1y) e^{-\frac{y^2}{2}} e^{-\frac{\alpha_1^2}{2}} d\alpha_1 dy.
    \end{align*}
    
    Now suppose if $\frac{b_1}{\sigma} \ge \frac{1}{2}$. We can get a numerical result for the integration 
    $$
    \int_{0}^{\infty} \int_{0}^{\infty} x y \tanh( \frac{1}{2} xy) e^{-\frac{y^2}{2}} e^{-\frac{x^2}{2}} dx dy,
    $$ 
    which is greater than 0.5. Thus we can conclude $b_1' \ge \frac{1}{\pi} \sigma_2 \ge \frac{1}{\pi} b_2^*$, which is much greater than $\|\bm{\beta}^*\|/10$ when $\sin \theta \ge \sqrt{1 - 0.2^2}$.
    
    If $\frac{b_1}{\sigma}$ is less than $1/2$, then we use the Taylor bound for $\tanh(x) \ge x - \frac{x^3}{3}$ to get
    \begin{align}
        \frac{2}{\pi} \sigma_2 & \int_{0}^{\infty} \int_{0}^{\infty} \alpha_1 y \tanh( \frac{b_1}{\sigma} \alpha_1y) e^{-\frac{y^2}{2}} e^{-\frac{\alpha_1^2}{2}} d\alpha_1 dy \nonumber \\
        &\ge \frac{2}{\pi} \sigma_2 \int_{0}^{\infty} \int_{0}^{\infty} \alpha_1 y ( \frac{b_1}{\sigma} \alpha_1y - \frac{1}{3} (\frac{b_1}{\sigma} \alpha_1y)^3 ) e^{-\frac{y^2}{2}} e^{-\frac{\alpha_1^2}{2}} d\alpha_1 dy \nonumber \\
        &= b_1 \frac{\sigma_2}{\sigma} (1 - 3 \frac{b_1^2}{\sigma^2}) \ge b_1 \sqrt{1 + \frac{24}{25} \eta^2} (1 - 3 \frac{b_1^2}{\sigma^2}).
    \end{align}
    
    If $\eta = \frac{\|\bm{\beta}^*\|}{\sigma} \ge 5$, then since we assumed $\frac{b_1}{\sigma} < 1/2$, we have $b_1 \sqrt{1 + \frac{24}{25} \eta^2} (1 - 3 \frac{b_1^2}{\sigma^2}) \ge \frac{5}{4} b_1$. Otherwise, suppose $b_1 = \|\bm{\beta}^*\|/10$, then we have $b_1' \ge b_1 \sqrt{1 + \frac{24}{25} \eta^2} (1 - \frac{3}{100} \eta^2)$. When $1 \le \eta \le 5$, we have $b_1' \ge \frac{5}{4} b_1$. When $0 \le \eta \le 1$, we have $b_1' \ge b_1(1 + 0.3\eta^2)$. Since by (\ref{eq:b_1_prime_increasing}) we know $b_1'$ is increasing as $b_1$ increases, and $\|\bm{\beta}'\| \ge b_1'$. Therefore, we conclude that sufficiently $\epsilon$ guarantees $\|\tilde{\bm{\beta}}'\| \ge \frac{\|\bm{\beta}^*\|}{10}$.
    
    \paragraph{\it $\sin \theta \le 0.2$:} Assume $b_1 = \frac{\|\bm{\beta}^*\|}{10} < \frac{\sigma^2}{\sigma_2^2} b_1^*$. Otherwise we can do as in the first case. From equation (\ref{ineq:newbeta_smaller}), we have 
    $$
        b_1' \ge b_1 + (1 - \kappa^3) (b_1^* - b_1),
    $$
    where $\kappa = \left({\sqrt{1 + \frac{\min(\newbeta, b_1^*)^2}{\sigma_2^2}}} \right)^{-1} \ge \sqrt{1 + \frac{b_1^2}{\sigma^2}}^{-1}$. Since $b_1^* - b_1 \ge \frac{\|\bm{\beta}^*\|}{2}$ in this case, we have
     $b_1' \ge b_1 + \frac{\eta^2}{100 + \eta^2} \frac{\|\bm{\beta}^*\|}{2}$. Similarly as in other cases, since $b_1'$ is increasing in $b_1 = \|\beta\|$, with sufficiently small $\epsilon$ we have $\|\tilde{\bm{\beta}}'\| \ge \frac{\|\bm{\beta}^*\|}{10}$ whenever $\|\beta\| \ge \frac{\|\bm{\beta}^*\|}{10}$.
\end{proof}

\subsection{Concentration Result in One Direction}
\label{appendix:aux_concentration_one_dir}

\begin{restatable}{theorem}{finitedotproduct}
\label{thm:finite_cos_update}
    Consider one iteration of sample-based EM algorithm. There exist absolute constants $c_1, c_2 > 0$, such that statistical error in a fixed direction $\bm{\beta}^*$ can be bounded with probability at least $1-\delta$, by
    \begin{equation}
        \label{eq:cont_fine_grained}
        |\langle \tilde{\bm{\beta}}' - \bm{\beta}', \bm{\beta}^* \rangle| \le \sqrt{\sigma^2+\|\bm{\beta}^*\|^2} \left( c_1 \sqrt{\frac{1}{n} \log(1/\delta)} + c_2 \frac{d}{n} \log(1/\delta) \right).
    \end{equation}
\end{restatable}

\begin{proof}
The error for which we are interested in giving a bound is
\begin{equation}
	\tilde{\bm{\beta}'} - \bm{\beta}' = \underbrace{(\frac{1}{n} \sum_{i=1}^n \bm{x}_i \bm{x}_i^{\top})^{-1}}_{\hat{\Sigma}^{-1}} \underbrace{(\frac{1}{n} \sum_{i=1}^n y_i \bm{x}_i \tanh(y_i \langle \bm{x}_i, \beta \rangle / \sigma^2))}_{\hat{\mu}} - \underbrace{\mathbb{E} [yX \tanh(y \langle X, \beta \rangle / \sigma^2)]}_{\mu}.
\end{equation}
Now we fix some $v \in R^d$ such that $\|v\|=1$, and give a bound for $|\langle \tilde{\bm{\beta}'} - \beta, v \rangle|$. First observe that,
\begin{align*}
	|\langle \tilde{\bm{\beta}'} - \bm{\beta}', v \rangle| &= |(\hat{\Sigma}^{-1} \hat{\mu} - \mu)^\top v| \\
	    &= |(\hat{\mu} - \mu)^\top v +  \mu^\top (\hat{\Sigma}^{-1} - I) v + (\hat{\mu} - \mu)^\top (\hat{\Sigma}^{-1} - I) v| \\
	    &\le |\underbrace{(\hat{\mu} - \mu)^\top v}_{A}| + \ |\underbrace{\mu^\top (\hat{\Sigma}^{-1} - I) v}_{B}| + |\underbrace{(\hat{\mu} - \mu)^\top (\hat{\Sigma}^{-1} - I) v}_{C}|.
\end{align*}
We will bound $A, B$ and $C$ separately. For simplicity, we will assume the problem is normalized, \textit{i.e.}, $\|\bm{\beta}^*\| = 1$.

\paragraph{\it Bounding A:} The product of two sub-Gaussian random variables is sub-exponential, which can be easily shown with the notion of sub-Gaussian norm and sub-exponential norm \cite{vershynin2010introduction}. 

The random variable $y_i \langle \bm{x}_i, v \rangle \tanh(y_i \langle \bm{x}_i, \beta \rangle / \sigma^2))$ is sub-exponential with parameter $C \sqrt{\sigma^2 + \|\beta^*\|^2}$ for some constant $C$, since $|\tanh(\cdot)| \le 1$, $\langle \bm{x}_i, v \rangle$ is sub-Gaussian with parameter 1, and $y_i$ is sub-Gaussian with parameter at most $\sqrt{\sigma^2 + \|\beta^*\|^2}$.

Applying the concentration inequality for sub-exponential random variable, we get
\begin{align*}
    \mathbb{P} \left( \left| \frac{1}{n} \sum_i y_i \langle \bm{x}_i, v \rangle \tanh \left(y_i \frac{\langle \bm{x}_i, \beta \rangle}{\sigma^2} \right) - \mathbb{E} \left[yX \tanh \left(y \frac{\langle X, \beta \rangle}{\sigma^2} \right) \right] \right| \ge t \right) &\le \exp\left(-\frac{nt^2}{K(\sigma^2 + \|\beta^*\|^2)}\right),
\end{align*}
for some absolute constant $K$. 

Equivalently, with probability at least $1-\delta$, we have $A \le c_1 \sqrt{\sigma^2 + \|\beta^*\|^2} \sqrt{\frac{1}{n} \log (1/\delta)}$ for some universal constant $c_1$.

\newcommand{\udotv}{\langle u, v \rangle}
\paragraph{\it Bounding B:} Standard results from random matrix theory imply that $\|\hat{\Sigma_p} - I\|_{op} \le c_2 \sqrt{\frac{d}{n} \log(1/\delta)}$ with high probability. We will consider events under this condition. 

Since inverse operator is hard to handle, we modify it using Taylor's expansion
\begin{align*}
     \hat{\Sigma}^{-1} &= (I - (I - \hat{\Sigma}))^{-1} \\
     &= I + (I - \hat{\Sigma}) + (I - \hat{\Sigma})^2 + ...,
\end{align*}
from where we can see $\mu^{\top} (\hat{\Sigma}^{-1} - I) v = \mu^{\top} (I - \hat{\Sigma}) v + \|\mu\| \tilde{O}(\frac{d}{n})$. 

For simplicity, let us define $u = \frac{\mu}{\|\mu\|}$ and derive a bound for $u^{\top} (I - \hat{\Sigma}) v$. Now we are left with bounding $u^{\top} (I - \hat{\Sigma}) v = u^{\top}v - \frac{1}{n} \sum_i (\bm{x}_i^{\top} u) (\bm{x}_i^{\top} v)$. Let two random variables $Z_1 = X^{\top} u$, $Z_2 = X^{\top} v$. Since $Z_1, Z_2$ are sub-Gaussian random variables with parameter 1, $Z_1 Z_2$ is sub-exponential with parameter at most 2. Thus, we get the sub-exponential concentration bound $u^{\top} (I - \hat{\Sigma}) v \le \tilde{O}(\sqrt{\frac{1}{n}})$. This yields $B \le c_2 \sqrt{\sigma^2 + \|\bm{\beta}^*\|^2} (\sqrt{ \frac{1}{n} \log(1/\delta)} + \frac{d}{n})$ since $\|\mu\| \le O(\sqrt{\sigma^2 + \|\bm{\beta}^*\|^2})$ due to Lemma \ref{lemma:bprime_bounded}. 

\paragraph{\it Bounding C:} We have $\|\hat{\mu} - \mu\| \le c_5 \sqrt{\sigma^2 + \|\bm{\beta}^*\|^2} \sqrt{\frac{d}{n} \log (1/\delta)}$ from \cite{balakrishnan_statistical_2017} with probability at least $1 - \delta$, as well as $\|\hat{\Sigma}^{-1} - I\|_{op} \le c_2 \sqrt{\frac{d}{n} \log(1/\delta)}$. Therefore, we get
$$
    |(\hat{\mu} - \mu)^{\top} (\hat{\Sigma}^{-1} - I) v| \le \\ \|\hat{\mu} - \mu\| \ \|\hat{\Sigma}^{-1} - I\|_{op} \ \|v\| \le \\ \tilde{O} (\sqrt{\sigma^2 + \|\bm{\beta}^*\|^2} \frac{d}{n}).
$$
This gives a bound for $C$.

Finally, combining the bounds on $A, B$ and $C$ with $v = \bm{\beta}^*$, we get the first part of the theorem.
\end{proof}

\end{appendices}

\end{document}